\documentclass{article}

\usepackage{microtype}
\usepackage{subfigure}
\usepackage{booktabs} 

\usepackage[accepted]{icml2024}

\usepackage{amsmath}
\usepackage{amssymb}
\usepackage{mathtools}
\usepackage{amsthm}
\usepackage{bm}

\usepackage{graphicx}
\usepackage[utf8]{inputenc}
\usepackage{amsfonts}
\usepackage{tikz}
\usepackage{mathrsfs}
\usepackage{algorithm}
\usepackage{algorithmic}
\usepackage[shortlabels]{enumitem}

\setlist{itemsep=0.8pt, topsep=2pt}

\usepackage{comment}
\usepackage{url}

\usepackage{hyperref}
\hypersetup{
  colorlinks=true,
  linkcolor={red!80!black},
  citecolor={blue!80!black},
  urlcolor={magenta!80!black}
} 

\theoremstyle{plain}
\newtheorem{theorem}{Theorem}[section]
\newtheorem{proposition}[theorem]{Proposition}
\newtheorem{lemma}[theorem]{Lemma}

\theoremstyle{definition}
\newtheorem{definition}[theorem]{Definition}
\newtheorem{assumption}{Assumption}
\theoremstyle{remark}
\newtheorem{remark}[theorem]{Remark}

\newcommand{\norm}[1]{\lVert#1\rVert}
\newcommand{\Norm}[1]{\left\lVert#1\right\rVert}
\newcommand{\abs}[1]{\left\lvert#1\right\rvert}

\DeclareMathOperator{\tr}{tr}
\DeclareMathOperator{\id}{id}

\DeclareMathOperator{\Var}{Var}

\DeclareMathOperator{\diag}{diag}

\DeclareMathOperator{\rank}{rank}

\DeclareMathOperator{\MLP}{MLP}
\DeclareMathOperator{\LSA}{LSA}
\DeclareMathOperator{\Tan}{Tan}
\DeclareMathOperator{\GP}{GP}

\newcommand{\E}[1]{\mathbb{E}[#1]}

\newcommand{\EE}[2]
{\mathbb{E}_{#1}[#2]}
\newcommand{\EEbig}[2]{\mathbb{E}_{#1}\left[#2\right]}

\DeclareMathOperator{\RR}{\mathbb{R}}
\DeclareMathOperator{\ZZ}{\mathbb{Z}}

\DeclareMathOperator{\NN}{\mathbb{N}}

\DeclareMathOperator{\XX}{\mathscr{X}}
\DeclareMathOperator{\PP}{\mathscr{P}}
\DeclareMathOperator{\TT}{\mathscr{T}}
\DeclareMathOperator{\BB}{\mathcal{B}}
\DeclareMathOperator{\Brr}{\mathscr{B}}

\DeclareMathOperator{\LL}{\mathcal{L}}
\DeclareMathOperator{\DD}{\mathcal{D}}
\DeclareMathOperator{\WW}{\mathcal{W}}
\DeclareMathOperator{\HH}{\mathcal{H}}

\newcommand\bx{\bm{x}}
\newcommand\ba{\bm{a}}
\newcommand\bw{\bm{w}}
\newcommand\bv{\bm{v}}
\newcommand\bz{\bm{z}}
\newcommand\br{\bm{r}}
\newcommand\bk{\bm{k}}
\newcommand\bW{\mathbf{W}}
\newcommand\bI{\mathbf{I}}
\newcommand\bT{\mathbf{T}}
\newcommand\bR{\mathbf{R}}
\newcommand\bE{\mathbf{E}}
\newcommand\bU{\mathbf{U}}
\newcommand\bV{\mathbf{V}}
\newcommand\bD{\mathbf{D}}
\newcommand\bZ{\mathbf{Z}}
\newcommand\bL{\mathbf{L}}
\newcommand\bH{\mathbf{H}}
\newcommand\bA{\mathbf{A}}
\newcommand\bB{\mathbf{B}}
\newcommand\bK{\mathbf{K}}
\newcommand\mm{m}
\newcommand\qq{\textup{qr}}
\newcommand\bSig{\bm{\Sigma}}
\newcommand\bxi{\bm{\xi}}

\newcommand\deltaL{\frac{\delta\!\LL}{\delta\mu}}
\newcommand\ddeltaL{\frac{\delta^2\!\!\LL}{\delta\mu^2}}
\newcommand\deltaF{\frac{\delta F}{\delta\mu}}
\newcommand\ddeltaF{\frac{\delta^2 \!F}{\delta\mu^2}}
\newcommand*{\rd}{\mathop{}\!\mathrm{d}}

\DeclareMathOperator{\Lip}{\mathrm{Lip}}

\DeclareMathOperator*{\argmin}{arg\,min}

\DeclareMathOperator*{\esssup}{ess\,sup}

\allowdisplaybreaks

\icmltitlerunning{Transformers Learn Nonlinear Features In Context}

\begin{document}

\twocolumn[
\icmltitle{Transformers Learn Nonlinear Features In Context:\\Nonconvex Mean-field Dynamics on the Attention Landscape}


\begin{icmlauthorlist}
\icmlauthor{Juno Kim}{ut,aip}
\icmlauthor{Taiji Suzuki}{ut,aip}
\end{icmlauthorlist}

\icmlaffiliation{ut}{Department of Mathematical Informatics, University of Tokyo, Tokyo, Japan}
\icmlaffiliation{aip}{Center for Advanced Intelligence Project, RIKEN, Tokyo, Japan}
\icmlcorrespondingauthor{Juno Kim}{junokim@g.ecc.u-tokyo.ac.jp}
\vskip 0.3in
]


\printAffiliationsAndNotice{}  

\begin{abstract}
Large language models based on the Transformer architecture have demonstrated impressive capabilities to learn in context. However, existing theoretical studies on how this phenomenon arises are limited to the dynamics of a single layer of attention trained on linear regression tasks. In this paper, we study the optimization of a Transformer consisting of a fully connected layer followed by a linear attention layer. The MLP acts as a common nonlinear representation or feature map, greatly enhancing the power of in-context learning. We prove in the mean-field and two-timescale limit that the infinite-dimensional loss landscape for the distribution of parameters, while highly nonconvex, becomes quite benign. We also analyze the second-order stability of mean-field dynamics and show that Wasserstein gradient flow almost always avoids saddle points. Furthermore, we establish novel methods for obtaining concrete improvement rates both away from and near critical points. This represents the first saddle point analysis of mean-field dynamics in general and the techniques are of independent interest.
\end{abstract}

\section{Introduction}

Attention-based neural architectures such as Transformers have revolutionized modern machine learning, from tasks in natural language and computer vision to multi-modal learning and beyond. Recently, interest has surged in the remarkable ability of such large language models to learn \emph{in context}, leading to a major paradigm shift in how intelligence arises in artificial systems. In-context learning (ICL) refers to the capacity of a pretrained model to solve previously unseen tasks based on demonstrative example prompts without further tuning its parameters.

A vigorous line of research initiated by \citet{Garg22} has sought to understand the mechanism behind ICL from a theoretical perspective, where prompts are real-valued input-output pairs $(x_i,f(x_i))_{i=1}^n$ generated from some function class $\TT$. Studies have shown that Transformers are capable of implementing various statistical learning algorithms such as gradient descent (GD) in context \citep{Oswald23, Akyrek23, Bai23}. In particular, \citet{Guo23} consider the realistic setting of learning \emph{with representations} where MLP layers act as transformations on top of which ICL is performed, and show that such models consistently achieve near-optimal performance.

While promising, these results are based on specific constructions which may not accurately reflect ICL in real models \citep{Shen23}. Other works have analyzed how ICL emerges from the training dynamics of Transformers \citep{Zhang23, Huang23, Ahn23a}. However, they are limited to models consisting of only a single attention layer due to the dynamical complexity and thus can only explain ICL of linear functions. Hence the following central question at the intersection of the two approaches remains unsolved:
\begin{center}
\emph{How does in-context learning with nonlinear representations (features) arise in Transformers with\\ MLP layers, optimized via gradient descent?}
\end{center}
In this paper, we investigate the optimization dynamics of a Transformer model consisting of a two-layer MLP followed by a linear attention layer, pretrained on linear transformations of feature representations. Contrary to existing approaches which attempt to solve for exact dynamics of the attention matrices, we factor out the attention layer via a two-timescale argument and study the geometric properties of the \emph{loss landscape} faced by an overparametrized MLP. Our contributions are highlighted as follows.
\begin{itemize}
\item We show that the MLP layer greatly increases the flexibility of ICL by extending the class of learnable functions to the Barron space and plays an essential role by encoding task-common features during pretraining.
\item Lifting to the mean-field regime, we show that this infinite-dimensional `attention landscape' is benign (strictly saddle) via directional analysis: all critical points are either global minima or saddle points.\footnote{We also include local maxima as saddle points for brevity. A rigorous definition is given in Section \ref{sec:stabwgf}.}
\item We formally prove that mean-field dynamics `almost always' avoids saddles, explaining how the MLP learns globally optimal representations. We develop a novel local stability analysis of Wasserstein gradient flow on the space of measures using tools from Otto calculus, optimal transport and functional analysis.
\item We further derive concrete improvement rates in three regions under slightly modified dynamics: away from saddle points, near global minima and near saddle points. For the last case, we discuss how perturbed dynamics may help ensure global convergence rates.
\end{itemize}

While the benignity of the attention landscape (Section \ref{sec:landscape}) is our central insight into ICL, Sections \ref{sec:mfd} and \ref{sec:iclrate} constitute the first qualitative and quantitative convergence analyses of \emph{nonconvex} mean-field dynamics around saddles and is also of significant independent interest from a technical standpoint.\footnote{\citet{Boufadene24} study a certain energy functional and prove benignity via flow interchange techniques. However, they do not discuss its implications for general gradient flow.} We present many novel results for general functionals and outline another application to three-layer neural networks. Finally, we conduct numerical experiments complementing our theory.

Theoretical preliminaries are provided in Appendix \ref{app:preliminaries}, and proofs of all results in Sections \ref{sec:icfl}-\ref{sec:iclrate} are given throughout Appendices \ref{app:icfl}-\ref{app:iclrate}.

\subsection{Related Works}

\paragraph{In-context learning.} A wide literature has developed around the various aspects of ICL; we only mention those most relevant to our setup. \citet{Akyrek23, Oswald23, Mahankali23} give a construction where a single linear attention layer is equivalent to one step of GD or ridge regression. Transformers are also capable of implementing statistical \citep{Bai23} and reinforcement learning algorithms \citep{Lin23} and model averaging \citep{Zhang23Bayes}. The attention-over-representation viewpoint has been studied by \citet{Guo23} and also \citet{Tsai19, Han23} from a kernel regression perspective. \citet{Zhang23} analyze the optimization of a linear attention-only Transformer and show global convergence; a relationship to preconditioned GD is established in \citet{Ahn23a}. Also, \citet{Huang23} give a stage-wise analysis for the softmax attention-only model. Finally, a joint dynamic framework for MLP and attention has been proposed in \citet{Tian23}. 

\paragraph{Mean-field dynamics (MFD).} Let $h_\theta$ denote a single neuron with parameter $\theta\in\Omega\subseteq\RR^m$ and $\PP(\Omega)$ the space of probability measures over $\Omega$.\footnote{We will also consider the space $\PP_2(\Omega)$ of the space of probability measures on $\Theta$ with bounded second moment that vanish on the boundary of $\Theta$, equipped with the 2-Wasserstein metric.} Consider a width $N$ two-layer neural network with $\frac{1}{N}$ scaling,
\begin{equation*}
\textstyle h_N(\bx) := \frac{1}{N}\sum_{j=1}^N h_{\theta^{(j)}}(\bx).
\end{equation*}
In the infinite-width limit $N\to\infty$, the output can be seen as an expectation $h_\mu(\bx) := \EE{\theta\sim\mu}{h_\theta(\bx)}$ over a distribution $\mu\in\PP(\Omega)$. The corresponding mean-field limit of gradient flow (GF) w.r.t. an objective functional $F: \PP(\Omega)\to\RR$ is known to be equivalent to the Wasserstein gradient flow \citep{Jordan98} and is given by the continuity equation
\begin{equation}\label{eqn:mfd}
\textstyle\partial_t\mu_t = \nabla\cdot\left(\mu_t\nabla\deltaF(\mu_t)\right),\quad t\geq 0,
\end{equation}
see Section \ref{sec:fininfin}. Networks in this regime are capable of dynamic feature learning, compared to the NTK regime where the underlying kernel is essentially frozen.
Works such as \citet{Chizat18, Mei18, Nitanda22} exploit the linearity in $\mu$ and the convexity of the loss to lift to a convex optimization problem on $\PP(\Omega)$ and obtain convergence results. In contrast, the ICL loss is inherently nonconvex due to the additional attention layer.

\paragraph{Landscape analyses.} Certain nonconvex objectives such as matrix completion, sensing and factorization have been proved to be benign via directional analysis \citep{Ge16, Ge17, Li19}. Recently, Gaussian $k$-index models have been shown to possess benign landscapes w.r.t. the projection matrix after factoring out the link function via a similar two-timescale limit \citep{Bietti23}. However, our work focuses on the optimization of the \emph{infinite}-dimensional variable $\mu\in\PP(\Theta)$, and the ICL objective \eqref{eqn:L} has a novel, more complex structure compared to these problems.

\subsection{Concurrent Works}

Since the initial draft of this paper was made public, a couple of other works extending the ICL literature have been released. \citet{Li24} study the optimization of a Transformer consisting of a softmax attention layer followed by a ReLU layer for classification problems, where the MLP learns to distinguish labels. \citet{Chen24} study the training of a multi-head softmax attention model for ICL of multi-task linear regression. \citet{Zhang24b} consider a linear attention layer followed by a linear layer which learns to encode a mean signal vector, but their model does not include any nonlinearities. Our paper remains the first to analyze the full expressive power and feature learning capabilities of the MLP layer.

\section{In-Context Feature Learning}\label{sec:icfl}

\paragraph{Notation.} We denote both the $L^2$-norm of vectors and spectral norm of matrices by $\norm{\cdot}$ and the nuclear norm by $\norm{\cdot}_*$. The unit ball in $\RR^k$ is written as $\mathbb{D}^k=\{\bz\in\RR^k:\norm{\bz}\leq 1\}$. The unit ball in $\RR^{k\times k}$ with respect to spectral norm is written as $\BB_1(k)=\{\bR\in\RR^{k\times k}: \norm{\bR}\leq 1\}$. The orthogonal group in dimension $k$ is denoted by $\mathcal{O}(k)$. The $L^2$-norm of functions is explicitly written as $\norm{\cdot}_{L^2}$.

\subsection{Setup: In-Context Learning}

The basic theoretical framework for studying ICL was first proposed by \citet{Garg22} and has since been widely embraced \citep{Bai23, Zhang23, Ahn23a, Huang23, Lin23, Wu24}. Let $\DD_{\XX}$ be a distribution over the input space $\XX\subseteq\RR^d$, and let $\TT$ be a class of functions $\XX\to\RR$ with a distribution $\DD_{\TT}$ over functions. For each prompt, we generate a new task $f\sim\DD_{\TT}$ and a batch of $n$ example input-output pairs $(\bx_i,y_i)_{i=1}^n$ where $\bx_i\sim\DD_{\XX}$ are i.i.d. and $y_i=f(\bx_i)$. We also independently generate a query token $\bx_{\qq}\sim\DD_{\XX}$. The prompt is gathered into an embedding matrix
\begin{equation*}
\bE=\begin{bmatrix}
    \bE^{\bx}\\ \bE^y
\end{bmatrix}=
\begin{bmatrix}
    \bx_1 &\cdots &\bx_n & \bx_{\qq} \\ y_1 & \cdots &y_n & 0
\end{bmatrix}\in\RR^{(d+1)(n+1)}.
\end{equation*}
In-context learning of a pretrained model $\mathbb{M}$ refers to the ability to form predictions $\widehat{y}_{\qq} = \mathbb{M}(\bE)$ for $y_{\qq}=f(\bx_{\qq})$ without knowledge of the current task $f$ and without updating its parameters.

\subsection{MLP-Attention Transformer}\label{sec:tfarchitecture}

We now formally define our Transformer model, which consists of a feedforward two-layer neural network (MLP) followed by a single linear self-attention (LSA) layer. This serves as a proxy of the original Transformer which consists of alternating feedforward and attention layers. As in \citet{Guo23}, we switch the conventional ordering of the two networks to view attention as a mechanism to exchange feature information encoded into the MLP layer (we may also consider the initial embedding as the first MLP layer).


\paragraph{MLP layer.} A vector-valued neuron with parameter $\theta=(\ba,\bw)^\top\in\Theta\subseteq\RR^k\times\RR^d$ and activation $\sigma:\RR\to\RR$ is defined as $h_\theta(\bx)=\ba\sigma(\bw^\top\bx)$. While the original Transformer takes $k=d$, we allow any $k\leq d$ representing the number of distinct features. The mean-field network corresponding to a measure $\mu\in\PP(\Theta)$ is defined as $\textstyle h_\mu(\bx) = \int_\Theta h_\theta(\bx)\mu(\rd\theta)$. We will also denote $\bSig_{\mu,\nu} = \EE{\bx\sim\DD_{\XX}}{h_\mu(\bx)h_\nu(\bx)^\top}\in\RR^{k\times k}$. As we wish to extract features from the input tokens, the MLP is applied to only the covariates $\bx_i$ and $\bx_{\qq}$ so that the prompt $\bE$ is transformed into
\begin{equation*}
\MLP(\bE)=\begin{bmatrix}
    h_\mu(\bx_1) &\cdots &h_\mu(\bx_n) & h_\mu(\bx_{\qq}) \\ y_1 & \cdots &y_n & 0
\end{bmatrix}.
\end{equation*}

\paragraph{LSA layer.} Linear attention is widely used to accelerate the quadratic complexity of standard attention \citep{Katharopoulos20, Yang23} and can capture various aspects of softmax attention \citep{Ahn23b} while still being theoretically amenable. For query, key and value matrices $\bW^Q,\bW^K,\bW^V\in\RR^{(k+1)(k+1)}$ define
\begin{equation*}
\textstyle \LSA(\bE) = \bW^V\bE\cdot\frac{1}{n}(\bW^K\bE)^\top(\bW^Q\bE).
\end{equation*}
We impose a specific form, shown to achieve the global optimum in \citet{Zhang23, Mahankali23}, where the query and key matrices are consolidated into $\bW\in\RR^{k\times k}$ and $\bW^V$ is reduced to a scalar multiplier $v$:
\begin{equation*}
\bW^V = \begin{bmatrix}*&*\\0_d^\top &v\end{bmatrix},\quad (\bW^K)^\top\bW^Q = \begin{bmatrix}\bW&0_d\\0_d^\top&*\end{bmatrix}.
\end{equation*}
We further absorb $v$ into $\bW$ and fix $v=1$ in order to focus on the more complex dynamics of the MLP layer. Corresponding to the position of $y_{\qq}$, the $(k+1,n+1)$th element of the output $\LSA\circ\MLP(\bE)$ is read out as the model prediction. Multiplying out the matrices yields
\begin{equation*}
\widehat{y}_{\qq} = \frac{1}{n}\sum_{i=1}^n y_i h_\mu(\bx_i)^\top\bW h_\mu(\bx_{\qq}).
\end{equation*}
Hence $\widehat{y}_{\qq}$ can be interpreted as a linear smoother with the kernel $k(\bx,\bx_{\qq}) = \frac{1}{n}h_\mu(\bx)^\top\bW h_\mu(\bx_{\qq})$ encoded by the MLP layer (cf. \citet{Tsai19} for softmax attention).

\paragraph{Regression over features.} In this paper, we study ICL of linear regression tasks over a common nonlinear transformation or feature map $f^\circ\in C(\XX,\RR^k)$, that is $\TT=\{\bv^\top f^\circ|\,\bv\in\RR^k\}$ with covariance $\bSig_{\bv} = \EE{\bv}{\bv\bv^\top}$. By replacing $f^\circ$ by $\bSig_{\bv}^{1/2}f^\circ$, we may assume $\bSig_{\bv} = \bI_k$. 
We also take the $n\to\infty$ (infinite prompt length) limit to disregard sampling error and let $\widehat{y}_{\qq} = \EE{\bx}{f(\bx)h_\mu(\bx)^\top}\bW h_\mu(\bx_{\qq})$ for any task $f\in\TT$; but see Appendix \ref{app:fin} for a discussion on how to incorporate finite samples. Hence our Transformer is pretrained with the following mean squared risk,
\begin{align}\label{eqn:regloss}
&\LL_\textup{TF}(\mu,\bW) := \frac{1}{2}\EEbig{\bx_{\qq},\bv}{(y_{\qq}-\widehat{y}_{\qq})^2}\\
&= \frac{1}{2}\EEbig{\bx_{\qq}}{\Norm{f^\circ(\bx_{\qq}) - \EE{\bx}{f^\circ(\bx)h_\mu(\bx)^\top}\bW h_\mu(\bx_{\qq})}^2}. \nonumber
\end{align}
Our goal is to show that gradient dynamics converges to a global minimum such that $\LL_\textup{TF} = 0$. Then the MLP layer has successfully learned the true representations $f^\circ$, and even for a new or `unseen' task $\bv_\text{new}\in \RR^k$ the Transformer is able to return the correct regression output $y_{\qq}$:
\begin{equation*}
\widehat{y}_{\qq} = \EE{\bx}{\bv_\textup{new}^\top f^\circ(\bx)h_\mu(\bx)^\top}\bW h_\mu(\bx_{\qq}) = \bv_\textup{new}^\top f^\circ(\bx_{\qq}).
\end{equation*}
We call this behavior \emph{in-context feature learning} (ICFL).

\subsection{Expressivity of Representations}

Before delving into the training dynamics of our Transformer, we show by extending classical analyses of two-layer neural networks that adding even a shallow MLP results in greatly increased in-context learning capabilities, justifying our feature-based approach.

\paragraph{Multivariate Barron class.} Barron-type spaces have been well established as the natural function classes for analyzing approximation and generalization of shallow neural networks \citep{Barron94, Weinan20, Weinan22}. Here, we extend the theory to our vector-valued setting. We focus on the ReLU case for ease of presentation, but many results extend to more general activations \citep{Klusowski16, Li20}.

Set $\Theta=\RR^k\times\RR^d$, $\sigma(z) = \max\{0,z\}$ and suppose $M_2 = \EE{\bx\sim\DD_{\XX}}{\norm{\bx}^2} <\infty$. The Barron space $\Brr_p$ of order $p\in [1,\infty]$ is defined as the set of functions $f = h_\mu$,  $\mu\in\PP(\Theta)$ with finite Barron norm
\begin{equation*}
\norm{f}_{\Brr_p} := \inf_{\mu: f=h_\mu} \left(\int\norm{\ba}^p \norm{\bw}^p \mu(\rd\theta) \right)^{1/p}.
\end{equation*}
This turns out to not depend on $p$ (Lemma \ref{thm:barrondef}), so we refer to \emph{the} Barron space and norm as $(\Brr, \norm{\cdot}_{\Brr})$. This space contains a rich variety of functions. The following is an application of the classical Fourier analysis \citep{Barron93}.
\begin{proposition}\label{thm:sobolev}
Suppose $h_\mu$ includes a bias term, i.e. $\XX\subseteq\XX_0\times \{1\}$. If $f=(f_j)_{j=1}^k \in C(\XX_0,\RR^k)$ such that each $f_j$ satisfies $\textstyle\inf_{\widehat{f}_j} \int_{\RR^{d-1}} \norm{\omega}_1^2|\widehat{f}_j(\omega)| <\infty$ for $\widehat{f}_j$ the Fourier transform of an extension of $f_j$ to $\RR^{d-1}$, then $f\in\Brr$. In particular, the Sobolev space $H^s(\XX_0)^k \subset\Brr$ for $s>\frac{d+1}{2}$.
\end{proposition}
Furthermore, $\Brr$ is exactly the class of representations that can be learned in context, demonstrating the expressive power gained by incorporating the MLP layer:
\begin{lemma}\label{thm:expressive}
$\LL_\textup{TF}(\mu,\bW) = 0$ has a solution such that $\esssup_\mu \norm{\ba}\norm{\bw}<\infty$ if and only if $f^\circ\in\Brr$.
\end{lemma}
In contrast, \citet{Mahankali23} show that the optimal LSA-only Transformer implements one step of GD for the \emph{linear} regression problem $(\bx_i,y_i)_{i=1}^n$ even when $y_i|\bx_i$ is nonlinear; thus we establish a clear gap in learning ability.

\paragraph{Generalization to unseen tasks.} If the Transformer has successfully learned $f^\circ$, it will achieve perfect accuracy on any new linear task $\bv_\textup{new}^\top f^\circ$ as discussed. On the other hand, if the test task is an arbitrary function $g\in C(\XX)$, we cannot hope to do better than the projection to the linear span of learned features $f_1^\circ, \cdots, f_k^\circ$ since \eqref{eqn:regloss} is a regression loss. We show this lower bound is optimal:

\begin{proposition}\label{thm:spanproj}
Suppose $\LL_\textup{TF}(\mu,\bW) \leq \epsilon$ for $f^\circ\in\Brr$ and $\norm{h_\mu}_{\Brr}, \norm{\bW} \lesssim 1$. Then for any new task $g\in C(\XX)$ with $\norm{g}_{L^2(\DD_{\XX})} \lesssim 1$, the ICL test error satisfies
\begin{align*}
&\EEbig{\bx_{\qq}}{\Norm{g(\bx_{\qq}) - \EE{\bx}{g(\bx)h_\mu(\bx)^\top}\bW h_\mu(\bx_{\qq})}^2}\\
&\lesssim \epsilon + \inf_{\bv\in\RR^k} \norm{g-\bv^\top f^\circ}_{L^2(\DD_{\XX})}^2.
\end{align*}
\end{proposition}
This extends the LSA-only case where the optimal output was shown to be the near-optimal linear model in \citet{Zhang23}. This also raises an important question: if the task $g$ depends nonlinearly on $h_{\mu^\circ}$, is it still beneficial to have learned the relevant features $\mu^\circ$? Clearly this depends on both $g$ and the initialization $\mu_0$; however, we present experiments supporting this intuition in Section \ref{sec:exps}.

\subsection{From Finite to Infinite Width}\label{sec:fininfin}

Continuing the above discussion, elements of the Barron space are effectively approximated by finite-width networks, which can be seen as an adaptive kernel method. The proof of the following is essentially due to \citet{Weinan22b}.
\begin{proposition}\label{thm:finapprox}
For any integer $N$ and $f^\circ\in\Brr$, there exists a width $N$ network $\widehat{h}_N$ given by the discrete measure $\widehat{\mu}_N = \frac{1}{N} \sum_{j=1}^N \delta_{\theta^{(j)}}$ with path norm $\norm{\widehat{h}_N}_\mathcal{P}:= \frac{1}{N}\sum_{j=1}^N \norm{\ba^{(j)}}\norm{\bw^{(j)}} \leq 3\norm{f^\circ}_{\Brr}$ and
\begin{equation*}
\inf_{\bW}\LL_\textup{TF}(\widehat{\mu}_N, \bW) \leq \frac{1}{2}\norm{\widehat{h}_N-f^\circ}_{L^2(\DD_{\XX})}^2\leq \frac{M_2\norm{f^\circ}_{\Brr}^2}{N}.
\end{equation*}
\end{proposition}
Using the low Rademacher complexity of Barron spaces, we can also simultaneously bound the generalization gap for a finite number of tasks $T$ as $\widetilde{O}(T^{-1/2})$ which is nearly minimax optimal \citep[Theorem 4.1]{Weinan19}. 

Moreover from a dynamical perspective, a propagation of chaos argument \citep{Sznitman91} shows that gradient descent indeed converges to \eqref{eqn:mfd} in the infinite-width limit. Let $F:\PP_2(\Omega) \to \RR$ be any $C^1$ functional such that $\norm{\nabla\deltaF}\leq L_1$, $\nabla\deltaF(\mu,\theta)$ is $L_2$-Lipschitz w.r.t. $\theta$ and $L_3$-Lipschitz w.r.t. $\mu$ in the $\WW_1$ metric. Denote the initial measure as $\mu_0\in \PP_2(\Omega)$, let $\theta_0^{(1)},\cdots, \theta_0^{(N)}$ be i.i.d. samples from $\mu_0$ and consider the empirical GF trajectories
\begin{equation*}
\frac{\rd}{\rd t}\theta_t^{(j)} = -\nabla\deltaF(\widehat{\mu}_{t,N},\theta_t^{(j)}), \quad \widehat{\mu}_{t,N} = \frac{1}{N}\sum_{j=1}^N \delta_{\theta_t^{(j)}}.
\end{equation*}
\begin{proposition}\label{thm:chaos}
For any $T\geq 0$, the $N$-particle empirical measure $\widehat{\mu}_{t,N}$ converges to $\mu_t$ as $\E{\WW_1(\widehat{\mu}_{t,N},\mu_t)}\to 0$ uniformly for all $t\in [0,T]$ as $N\to\infty$.
\end{proposition}
See Remark \ref{rem:chaos} for the case of the ICFL objective. Hence it is natural to analyze optimization in the mean-field or extremely overparametrized regime.

\section{Benign Attention Landscape}\label{sec:landscape}

In this Section, we characterize the infinite-dimensional landscape of the ICFL objective. We will see that while highly nonconvex, $\LL_\textup{TF}$ possesses various desirable properties that make global optimization feasible via first-order methods. We first state two mild assumptions.

\begin{assumption}
The nonlinearity is $C^2$ and bounded as $|\sigma|\leq R_1$, $|\sigma'|\leq R_2$, $|\sigma''|\leq R_3$. The parameter space is $\Theta=\mathbb{D}^k\times \RR^d$. The input distribution $\DD_{\XX}$ has finite 4th moment, $\EE{\bx\sim\DD_{\XX}}{\norm{\bx}^j} = M_j<\infty$ for $j=2,4$.
\end{assumption}

The smoothness of $\sigma$ (which rules out ReLU activation) and restriction of the second layer to $\mathbb{D}^k$ are technicalities to easily ensure regularity bounds. Alternatively, we may take the parameter space to be $\Theta=\RR^k\times\RR^d$ without no-flux constraints and rely on the second moment bound of $\ba$; see Lemma \ref{thm:secondmoment} and the preceding comments. Only the assumption $|\sigma|\leq R_1$ is needed in this Section, which implies that $h_\theta,h_\mu$ lie within the ball of radius $R_1$ in $\RR^k$. 

Next, from the solvability condition of Lemma \ref{thm:expressive} we are naturally led to take $f^\circ=h_{\mu^\circ}$ for some `true' distribution $\mu^\circ$, which allows for a rich class of feature representations. We only require nondegeneracy of $f^\circ$:

\begin{assumption}\label{ass:features}
$f^\circ = h_{\mu^\circ}$ for some $\mu^\circ\in\PP_2(\Theta)$ satisfying $\underline{r}\bI_k\preceq\bSig_{\mu^\circ,\mu^\circ}\preceq \overline{r}\bI_k$ for $\overline{r}\geq \underline{r}>0$.\footnote{Before task rescaling, this is equivalent to assuming $\bSig_{\bv}$ is also invertible and $f^\circ$ is in a rescaled Barron class due to the pushforward argument in Section \ref{sec:nospur}.}
\end{assumption}
In particular, we do not assume any Gaussianity as in \citet{Oswald23, Akyrek23, Zhang23} nor orthonormality as in \citet{Huang23}. Note $\underline{r} \leq \frac{R_1^2}{k}$ automatically since $R_1^2\geq\tr\bSig_{\mu^\circ,\mu^\circ}\geq k\underline{r}$, and also $\bSig_{\mu,\nu}\preceq R_1^2\bI_k$ for all $\mu,\nu$. We will subsequently take $\underline{r}, \overline{r}=\Theta(\frac{1}{k})$ to extract the dependency on $k$.

One implicit assumption is that the number of true features $\dim h_{\mu^\circ}$ is known and equal to $k$. When the dimensions do not match, attention can perform the regulatory function of selecting important features \citep{Yasuda23}; we leave a dynamical characterization to future work. Experiments on a misspecified model are also conducted in Section \ref{sec:exps}.

\subsection{Fast Convergence of Attention}

In order to isolate the more nuanced dynamics of $\mu$, we first notice that minimizing $\LL_\textup{TF}$ over $\bW$ is a least-squares regression problem. In particular, $\LL_\textup{TF}$ is convex with respect to $\bW$ (strongly convex unless $\bSig_{\mu^\circ,\mu}$ or $\bSig_{\mu,\mu}$ are singular) and thus is optimized potentially much more quickly.

A possibility is that the MLP $\bx\mapsto h_\mu(\bx)$ degenerates to completely lie within a low-dimensional linear subspace of $\RR^k$. As the regression \eqref{eqn:regloss} is ill-conditioned in this case, we set $\PP_2^0(\Theta):=\{\mu\in\PP_2(\Theta):\rank\bSig_{\mu,\mu}<k\}$ and restrict our attention to $\PP_2^+(\Theta)= \PP_2(\Theta)\setminus\PP_2^0(\Theta)$.\footnote{We show the singular set $\PP_2^0(\Theta)$ is sparse in a strong sense in Proposition \ref{thm:dense}, justifying subsequent calculations. \citet{Jiang22} suggest that adaptive optimization methods can outperform SGD by biasing trajectories away from ill-conditioned regions.} Similarly to the asymptotic convergence of $\bW$ to $\Var(\bx)^{-1}$ for the LSA-only model \citep{Zhang23}, we then have:

\begin{lemma}\label{thm:wconv}
For any fixed $\mu\in\PP_2^+(\Theta)$ and any initialization $\bW_0\in\RR^{k\times k}$, the flow $\frac{\rd}{\rd t}\bW_t = -\nabla_{\bW}\LL_\textup{TF}(\mu,\bW_t)$ converges linearly to some $\bW_\mu\in\argmin_{\bW}\LL_\textup{TF}(\mu,\bW)$ which satisfies $\bSig_{\mu^\circ,\mu}\bW_\mu = \bSig_{\mu^\circ,\mu}\bSig_{\mu,\mu}^{-1}$.
\end{lemma}
Therefore it is reasonable to suppose that $\bW$ is updated sufficiently quickly and has already converged to $\bW_\mu$ for each $\mu$ -- formally by modeling as two-timescale dynamics \citep{Berglund06} -- leading us to study the objective
\begin{equation}\label{eqn:L}
\LL(\mu):= \inf_{\bW}\LL_\textup{TF}(\mu,\bW) = \frac{1}{2}\EEbig{\bx}{\norm{\zeta_{\mu^\circ,\mu}(\bx)}^2}
\end{equation}
where we denote $\zeta_{\mu^\circ,\mu}(\bx):=h_{\mu^\circ}(\bx) - \bSig_{\mu^\circ,\mu}\bSig_{\mu,\mu}^{-1} h_\mu(\bx)$. Note the constant bound $0\leq \LL(\mu)\leq \LL_\textup{TF}(\mu,0_{k\times k}) \leq \frac{k\overline{r}}{2}$.

\subsection{No Spurious Local Minima}\label{sec:nospur}

For an orthogonal matrix $\bR\in \mathcal{O}(k)$, define $\bR\sharp\mu$ as the pushforward of $\mu$ along the rotation map $\bR: (\ba,\bw)\mapsto (\bR\ba,\bw)$ so that $h_{\bR\sharp\mu}(\bx)= \int_{\Theta} \bR h_\theta(\bx)\rd\mu(\theta) = \bR h_\mu(\bx)$. Since the convex hull of $\mathcal{O}(k)\subset\RR^{k\times k}$ is equal to $\BB_1(k)$, this can be extended to any $\bR\in\BB_1(k)$ by decomposing $\bR=\sum_{j=1}^m \alpha_j\bR_j$ and defining $\bR\sharp\mu = \sum_{j=1}^m \alpha_j\bR_j\sharp\mu$. See Lemma \ref{thm:hull} for details of the construction. Achieving zero loss implies that we have learned the true representation $h_{\mu^\circ}$ up to a linear transformation: 

\begin{lemma}\label{thm:globalminima}
The pushforwards $\bR\sharp\mu^\circ$ for any invertible $\bR\in \BB_1(k)$ are global minima of $\LL$. Conversely, any global minimum $\mu$ of $\LL$ satisfies $h_\mu(\bx) = \bR h_{\mu^\circ}(\bx),\,\forall\bx\in\XX$ for some invertible matrix $\bR$.
\end{lemma}

The following theorem is the main result of this Section. It states that for any $\mu$ that is not a global minimum, it is either (1) possible to move in a direction where $\LL$ is strictly decreasing, or (2) $\LL$ possesses an unstable direction. In particular, all local minima must also be global minima.

\begin{theorem}[no spurious local minima]\label{thm:landscape}
For any $\mu\in\PP_2^+(\Theta)$ that is not a global minimum the following hold:
\begin{enumerate}[\normalfont(i)]
    \item There exists $\bR\in\BB_1(k)$ depending on $\mu$ such that along the linear homotopy $\bar{\mu}_s=(1-s)\mu+s\bR\sharp\mu^\circ$ we have $\frac{\rd}{\rd s}\big|_{t=0}\LL(\bar{\mu}_s)\leq 0$.\label{item:firstorder}
    \item If $\frac{\rd}{\rd s}\big|_{s=0}\LL(\bar{\mu}_s)=0$ for all $\bR\in\BB_1(k)$ above, then $\LL(\mu)\geq \frac{\underline{r}}{2}$ and $\frac{\rd^2}{\rd s^2}|_{s=0}\LL(\bar{\mu}_s)\leq -\frac{4}{kR_1^2}\LL(\mu)^2$ for some $\bR\in\BB_1(k)$.\label{item:secondorder}
\end{enumerate}
\end{theorem}

The proof, deferred to Appendix \ref{app:landscape}, exploits the linearity of the MLP output in $\mu$ to analyze linearized perturbations.

As a corollary of \ref{item:secondorder}, critical points cannot exist in the band $0<\LL<\frac{\underline{r}}{2}$. The threshold, controlled by $\underline{r}$, corresponds to the minimum loss when the features $h_\mu$ are uninformative in the sense that the regression coefficient $\bSig_{\mu^\circ,\mu}\bSig_{\mu,\mu}^{-1}$ against the true features is singular. This observation can be improved to the following quantitative guarantee:

\begin{proposition}[accelerated convergence phase]\label{thm:accel}
Let $\delta\in[0,\frac{\underline{r}^2}{4R_1^2}]$. For any $\mu\in\PP_2^+(\Theta)$ such that
\begin{equation*}
\frac{\underline{r} - \sqrt{\underline{r}^2-4R_1^2\delta}}{4} \leq \LL(\mu)\leq \frac{\underline{r} + \sqrt{\underline{r}^2 -4R_1^2\delta}}{4},
\end{equation*}
there exists $\bR\in\BB_1(k)$ such that along $\bar{\mu}_s=(1-s)\mu+s\bR\sharp\mu^\circ$ we have $\frac{\rd}{\rd s}\big|_{s=0}\LL(\bar{\mu}_s)\leq -\delta$.
\end{proposition}
See Appendix \ref{app:accel} for the proof. In other words, once in the band $(0,\frac{\underline{r}}{2})$ we are guaranteed a non-vanishing gradient which moreover becomes steeper closer to the center of the band, proportional to $\LL(\mu)(\frac{\underline{r}}{2}-\LL(\mu))$. We prove that for MFD this results in an acceleration-deceleration phase when converging to global minima in Theorem \ref{thm:accelrate}. 

\section{Mean-field Dynamics Avoids Saddle Points}\label{sec:mfd}

\subsection{Local Geometry of Wasserstein Space}

Strict saddle properties such as Theorem \ref{thm:landscape} have powerful implications for nonconvex optimization. In finite dimensions, a central result states that GD almost always avoids saddle points and converges to global optima \citep{Lee19}; see Appendix \ref{app:recap} for a recap. We develop the analogous general result for Wasserstein gradient flows (WGF) \eqref{eqn:mfd} by combining tools from functional analysis, optimal transport and metric geometry. 

Let $F: \PP_2(\Omega)\to\RR$ a general $C^2$ functional with domain $\Omega\subseteq\RR^m$. We use the elegant formalism of Otto calculus \citep{Otto01} to analyze local behavior of distributional flows. The reader is referred to Appendix \ref{app:measure} as well as \citet{Ambrosio05, Villani09} for expository details. There is a one-to-one equivalence between absolutely continuous curves $(\mu_t)$ in $\PP_2(\Omega)$ and time-dependent gradient vector fields $(\bv_t)$ on $\Omega$ solving $\partial_t\mu_t=\nabla\cdot(\bv_t\mu_t)$. This motivates the formal definition of the tangent space to $\PP_2(\Omega)$ at $\mu$ as
\begin{equation}\label{eqn:tangentspace}
\Tan_\mu\PP_2(\Omega):=\overline{\{\bv=\nabla\psi: \psi\in C_c^\infty(\Omega)\}}^{L^2(\Omega,\mu)}
\end{equation}
with the inherited inner product. We can view nearby measures as slight pushforwards of $\mu$ along the optimal transport map $\id_\Omega+\epsilon\bv$, $\epsilon>0$ analogously to the exponential map.


\subsection{Stability of Wasserstein Gradient Flow}\label{sec:stabwgf}

With the above framework in mind, we derive a local transport characterization of MFD by lifting to the tangent space.

\begin{lemma}\label{thm:localevo}
The WGF $(\mu_t)$ in a neighborhood of a critical point $\mu^\dagger$ of $F$ can be written as $\mu_t=(\id_\Omega+\epsilon\bv_t)\sharp\mu^\dagger$ where the velocity field $\bv_t$ changes as
\begin{equation}\label{eqn:localevo}
\partial_t\bv_t(\theta) = -\int \bH_{\mu^\dagger}(\theta,\theta') \bv_t(\theta') \mu^\dagger(\rd\theta') +o(1).
\end{equation}
Here, $\bH_\mu:(\Omega\times\Omega,\mu\otimes\mu)\to\RR^{m\times m}$ denotes the matrix-valued kernel $\textstyle\bH_\mu(\theta,\theta') := \nabla_\theta\nabla_{\theta'}\ddeltaF(\mu,\theta,\theta')$.
\end{lemma}
The tangent field $\nabla\deltaF(\mu_t)$ satisfies a similar dynamics (Lemma \ref{thm:evo}); we posit $\bH_\mu$ is the fundamental quantity governing second-order behavior of WGF. This facilitates stability analysis via the spectral theory of linear operators,
\begin{lemma}\label{thm:generalvalid}
Suppose the kernel $\bH_\mu$ is Hilbert-Schmidt for $\mu\in\PP_2(\Omega)$, that is $\iint\norm{\bH_\mu}^2 \rd\mu\otimes\rd\mu<\infty$. Then the corresponding integral operator on $L^2(\Omega,\mu;\RR^m)$,
\begin{equation}\label{eqn:hh}
\HH_\mu: f\mapsto \HH_\mu\! f(\theta) = \int \bH_\mu(\theta,\theta') f(\theta')\mu(\rd\theta')
\end{equation}
is compact self-adjoint, hence there exists an orthonormal basis $\{\psi_j\}_{j\in\ZZ}$\, for $L^2(\Omega,\mu; \RR^m)$ of eigenfunctions of $\HH_\mu$.
\end{lemma}

We are thus motivated to define the set of \emph{strict saddle} points as $\mathscr{G}^\dagger:=\{\mu\in\PP_2(\Omega): \nabla\deltaF(\mu) = 0,\,\lambda_\textup{min}(\HH_\mu) <0\}$. Near such points, we now apply the center-stable manifold theorem for Banach spaces (Theorem \ref{thm:centerstable}). This tells us that $\Tan_{\mu^\dagger}\PP_2(\Omega)$ can be decomposed into a direct sum of $\HH_{\mu^\dagger}$-invariant subspaces $\mathscr{E}^s\oplus \mathscr{E}^u$ such that all flows \eqref{eqn:localevo} converging to $\mu^\dagger$ must be eventually contained in the graph of a $C^1$ map $h:\mathscr{E}^s\to \mathscr{E}^u$ defined near the origin. Denoting the reversed WGF for time $t$ as $\omega_t^-$ whenever it is defined -- which forms a bi-Lipschitz inverse for the forward flow \citep[Theorem 11.1.4]{Ambrosio05} -- we conclude:


\begin{theorem}\label{thm:measure}
For any $C^2$ functional $F:\PP_2(\Omega)\to\RR$ with Hilbert-Schmidt kernel $\bH_\mu$, the set $\mathscr{G}_0^\dagger=\{\mu_0\in\PP_2(\Omega):\lim_{t\to\infty}\mu_t\in\mathscr{G}^\dagger\}$ of initializations which converge to strictly saddle points is contained in the countable union $\cup_{\ell\in\NN}\cup_{j\in\NN} \omega_\ell^-(\mathscr{V}_j)$ of images of submanifolds $\mathscr{V}_j$.
\end{theorem}

\begin{remark}
We point out that Otto calculus is only formal in the sense that existence and regularity issues are ignored, so it is difficult to rigorously turn the above into a meaningful measure-theoretic statement. This is compounded by the fact that there is no well-behaved canonical measure on $\PP_2(\Omega)$. A possible justification is to restrict to the subspace of measures with smooth positive Lebesgue density whose geometry is well-behaved \citep{Lott08, Villani09}, but this is outside of the scope of our paper.
\end{remark}

For the ICFL objective \eqref{eqn:L}, Proposition \ref{thm:operator} together with Theorem \ref{thm:landscape}\ref{item:secondorder} will show that all critical points that are not global optima are strictly saddle in $\mathscr{G}^\dagger$. Hence Theorem \ref{thm:measure} applies to $\LL$ with the domain of interest replaced by $\PP_2^+(\Theta)$, and thus `almost all' convergent flows in $\PP_2^+(\Theta)$ must converge to global minima.

\paragraph{Application to three-layer networks.} The problem \eqref{eqn:L} can also be motivated by the training dynamics of a certain three-layer neural network in a teacher-student setting. We construct the first two layers identically to our MLP layer $h_\mu$ and add a linear third layer given by the transformation $\bT\in\RR^{k\times k}$. Then the $L^2$ loss with respect to a teacher network $\bx\mapsto \bT^\star h_{\mu^\star}(\bx)$ can be written as
\begin{equation*}
\LL_\textup{NN}(\mu,\bT) = \EE{\bx}{\norm{\bT^\star h_{\mu^\star}(\bx) - \bT h_\mu(\bx)}^2}.
\end{equation*}
By setting $\mu^\circ = \bT^\star\sharp \mu^\star$ and taking the two-timescale limit where the last layer updates infinitely quickly, we see that $\bT$ must converge to $\bSig_{\mu^\circ,\mu}\bSig_{\mu,\mu}^{-1}$ and we end up with the regression objective \eqref{eqn:L}, hence Sections \ref{sec:landscape}-\ref{sec:iclrate} also directly apply to this problem. We remark that the two-timescale regime has been leveraged to show convergence of SGD for \emph{two}-layer networks in \citet{Marion23}.

\section{Convergence Rates for ICFL}\label{sec:iclrate}

Theorem \ref{thm:measure} is encouraging but only qualitative. In this Section, we develop brand-new approaches to obtain quantitative improvement results for mean-field dynamics, in particular for the ICFL objective, (1) away from critical points; (2) near global minima; and (3) near saddle points.

\subsection{MFD with Birth-Death}

Consider the WGF \eqref{eqn:mfd} for $F=\LL$, where we have
\begin{equation*}
\deltaL(\mu,\theta) = -\EEbig{\bx}{\zeta_{\mu^\circ,\mu}(\bx)^\top \bSig_{\mu^\circ,\mu}\bSig_{\mu,\mu}^{-1} h_\theta(\bx)}.
\end{equation*}
To preserve the benign landscape, we do not add entropic regularization typically required in mean-field analyses. However, a different modification will be beneficial in obtaining concrete rates. For a fixed distribution $\pi\in\PP_2(\Theta)$, if at any time $t\geq 0$ the density ratio $\inf_\Theta\frac{\rd\mu_t}{\rd\pi}$ is no larger then a small threshold $\gamma$, we perform the discrete update $\mu_t \gets (1-\gamma)\mu_t +\gamma\pi$. Forcing $\mu_t$ slightly towards $\pi$ ensures sufficient mass to decrease the objective at all times. This can be implemented by a birth-death process where a fraction $\gamma$ of all neurons are randomly deleted and replaced with samples from $\pi$ whenever $\LL$ does not sufficiently decrease; see Algorithm \ref{algo} in the Appendix. 
For $\pi$, we require:

\begin{assumption}\label{ass:pi}
$\pi$ is spherically symmetric in the $\ba$ component, that is $\pi(\ba,\bw)=\pi(\ba',\bw)$ if $\norm{\ba}=\norm{\ba'}$. Also, $\mu^\circ$ has finite density w.r.t. $\pi$ as $\norm{\frac{\rd\mu^\circ}{\rd\pi}}_\infty \leq R_4$.
\end{assumption}

\begin{remark}\label{rmk:forcing}
The continuous-time version of birth-death can be written as a PDE with discontinuous forcing,
\begin{equation*}
\textstyle\partial_t\mu_t = \nabla\cdot\left(\mu_t\nabla\deltaL(\mu_t)\right) + \mathbf{1}_{\left\{\inf_\Theta\frac{\rd\mu_t}{\rd\pi}\leq \gamma\right\}}\gamma(\pi-\mu_t),
\end{equation*}
which also ensures an $\Omega(\gamma)$ lower bound. A similar perturbation has been studied before for convex MFD in \citet{Wei19}. Birth-death mechanisms can also accelerate convergence of mean-field networks \citep{Rotskoff19}.
\end{remark}

\subsection{First-order Improvement}

We first give a result which translates nonzero gradients along a direction of improvement into a first-order rate of decrease for the gradient flow. 
Unlike convex mean-field \emph{Langevin} dynamics which relies on a log-Sobolev inequality to control dissipation \citep{Nitanda22}, our idea is to exploit the mobility of the second layer mass. The argument works for any objective built on top of the MLP layer $h_\mu$; see Proposition \ref{thm:generalfirst} in the Appendix for the general result.

\begin{proposition}\label{thm:mfdfirst}
Suppose MFD with birth-death on $\LL$ at time $t$ satisfies Theorem \ref{thm:landscape}\ref{item:firstorder} with $\frac{\rd}{\rd s}\big|_{s=0}\LL(\bar{\mu}_s) \leq -\delta$. Then $\frac{\rd}{\rd t}\LL(\mu_t) \leq -R_4^{-1}\gamma\delta^2$.
\end{proposition}

Since a steep gradient is guaranteed in the accelerated convergence phase by Proposition \ref{thm:accel}, we further establish the following convergence rate.

\begin{theorem}[accelerated convergence rate]\label{thm:accelrate}
Once $\LL(\mu_t)\leq0.49\underline{r}$ is satisfied, MFD with birth-death will converge in loss with $\LL(\mu_{t+\tau})\leq \epsilon$ in at most $\tau=O(\frac{k^2}{\gamma\epsilon})$ time.
\end{theorem}

The rate is quadratic in the \emph{feature} dimension $k$ and independent of $d = \dim\XX$. Hereafter, big $O$ notation hides at most polynomial dependency on constants $R_j,M_j$, while dependency on $k,d,\gamma$ is made explicit and $\underline{r}, \overline{r}=\Theta(\frac{1}{k})$.

\subsection{Second-order Improvement}

We now arrive at the main difficulty of our analysis: the behavior of mean-field dynamics near critical points. In the finite-dimensional case, local stability is determined by the Hessian matrix. We show that the mean-field analogue is


\begin{lemma}\label{thm:evo}
For a smooth functional $F:\PP_2(\Theta)\to\RR$, the velocity field $\nabla\deltaF$ of \eqref{eqn:mfd} satisfies the evolution equation
\begin{equation}\label{eqn:evo}
\partial_t\left[\nabla\deltaF(\mu_t)\right] = -\HH_{\mu_t} \left[\nabla\deltaF(\mu_t)\right].
\end{equation}
\end{lemma}

This is the non-perturbative or tangent curve version of Lemma \ref{thm:localevo}. For the specific objective $\LL$, we show that $\HH_\mu$ is Hilbert-Schmidt and derive regularity properties in Lemma \ref{thm:kernelvalid} and \ref{thm:kernelreg}. Next, the following lemma translates second-order instability into a spectral bound for $\HH_\mu$.

\begin{proposition}\label{thm:banana}
Suppose MFD with birth-death at time $t$ satisfies Theorem \ref{thm:landscape}\ref{item:secondorder} with $\frac{\rd^2}{\rd s^2}\big|_{s=0}\LL(\bar{\mu}_s) \leq -\Lambda$. Then the smallest eigenvalue $\lambda_0$ of $\HH_{\mu_t}$ satisfies $\lambda_0\leq -R_4^{-1}\gamma\Lambda$.
\end{proposition}

Therefore we expect that even if the dynamics is close to a saddle point and Proposition \ref{thm:mfdfirst} is not useful, as long as the $L^2$-component along the eigenfunction $\psi_0$ corresponding to $\lambda_0$ is not exactly zero, it will blow up exponentially in time until $\mu_t$ escapes and makes progress. In detail,

\begin{theorem}\label{thm:naiveescape}
Suppose MFD with birth-death on $\LL$ satisfies Theorem \ref{thm:landscape}\ref{item:firstorder} with $\frac{\rd}{\rd s}\big|_{s=0}\LL(\bar{\mu}_s) \geq - O(k^{-1}\LL(\mu_t)^2)$ at time $t$. Further suppose $\psi_0$ satisfies
\begin{equation*} \bigg|\int\psi_0^\top\nabla\deltaL(\mu_t)\rd\mu_t\bigg| \geq \alpha
\end{equation*}
for some $\alpha> 0$. Then for the time interval
\begin{equation*}
\tau = \widetilde{O}\left(\frac{k}{\gamma\LL(\mu_t)^2}\log\frac{1}{\alpha}\right),
\end{equation*}
MFD in the region $\{\mu\in\PP_2(\Theta):\lambda_\textup{min}(\bSig_{\mu,\mu})=\Omega(\frac{1}{k})\}$ decreases $\LL$ as
\begin{equation*}
\LL(\mu_{t+\tau})\leq \LL(\mu_t) - \widetilde{\Omega}\left(\frac{\gamma^2\alpha\LL(\mu_t)^4}{k^5d}\right).
\end{equation*}
\end{theorem}

Simply put, we make $\widetilde{\Omega}(\alpha)$ progress in $\widetilde{O}(\log\frac{1}{\alpha})$ time. The proof idea is to find a $\WW_2$-ball where if $\mu_t$ does not escape in time $\tau$, the exponential blowup guarantees improvement of $\LL$; if $\mu_t$ does escape, $\LL$ must have decreased enough to warrant such displacement (via the Benamou-Brenier formula). Again, we present general versions of Proposition \ref{thm:banana} and Theorem \ref{thm:naiveescape} as Proposition \ref{thm:operator} and Theorem \ref{thm:generalescape}.

\paragraph{Dimensional dependency.} The rate is polynomial in the number of features $k$ but only linear in $d$, mitigating the curse of dimensionality. Initially $\LL$ decreases by $\widetilde{\Omega}(k^{-5}d^{-1})$ in time $\widetilde{O}(k)$ when $\LL=\Theta(1)$. As training progresses, the rate worsens to $\widetilde{\Omega}(k^{-9}d^{-1})$ in time $\widetilde{O}(k^3)$ when $\LL=\Theta(\frac{1}{k})$ due to the smaller curvature of $\LL$, until we enter the accelerated convergence phase and Theorem \ref{thm:accelrate} takes over.

\begin{remark}
Since $\LL$ becomes ill-conditioned if $h_\mu(\bx)$ is nearly constrained on a subspace, we have assumed that $\lambda_\textup{min}(\bSig_{\mu,\mu})$ is locally bounded below (on the same order as the upper bound $\frac{R_1^2}{k}$) to obtain regularity estimates. We expect this to not be a problem in practice since $\bW$ will not diverge without timescale separation. In our experiments, $\lambda_\textup{min}(\bSig_{\mu,\mu})$ never changed by over 25\% during each training phase.
\end{remark}

\subsection{Escaping from Saddle Points Efficiently}

Theorem \ref{thm:naiveescape} on its own cannot ensure convergence rates. The flow might be initialized at or pass near multiple saddle points with very small $\alpha$ values, taking longer to escape. This is an unavoidable problem of nonconvex gradient descent even in finite dimensions \citep{Du17}. In contrast, it has been shown that simply adding uniform noise allows GD to escape saddle points efficiently \citep{Ge15, Jin17}. Here, we suggest without proof an adaptation to the Wasserstein gradient flow.

The main problem is how to apply `random' perturbations in the infinite-dimensional space $\PP_2(\Omega)$. Motivated by the characterization of the tangent space \eqref{eqn:tangentspace}, we propose the following scheme which constructs perturbations in the velocity space using vector-valued Gaussian processes. See Definition \ref{def:gp} and Algorithm \ref{algo} for details.
\begin{enumerate}[nosep]
\item Generate a random velocity field $\bxi:\Omega\to\RR^m$ from a stationary Gaussian process $\GP(0,\bK)$ with bounded kernel $\bK: \Omega\times\Omega\to \RR^{m\times m}$.
\item Run the pushforward dynamics $\partial_t\mu_t = \nabla\cdot(\bxi\mu_t)$ from $\mu_0=\mu^\dagger$ for fixed time $\Delta t$.
\end{enumerate} 
This can bypass the dimensional dependency in \citet{Ge15} and ensure a nonzero $\psi_0$-component for $\nabla\deltaF(\mu_{\Delta t})$, which is approximately normally distributed with variance $O(\Delta t)$ (Lemma \ref{thm:gpnormal}). Unfortunately this naive approach is not enough to ensure large $\alpha$, at least in polynomial time, since the eigenfunction $\psi_0$ and base measure also change along the perturbation. \citet{Jin17} bypass this issue in finite dimensions via a geometric argument; we conjecture that our method also guarantees polynomial escape time. If this is true, we may combine Proposition \ref{thm:mfdfirst} with $\delta=O(k^{-1}\LL(\mu_t)^2)$, yielding $O(\frac{k^6}{\gamma^3 t^3})$ convergence away from saddle points, and Theorem \ref{thm:naiveescape} to conclude that \emph{perturbed} WGF enjoys polynomial convergence to global minima.


\begin{figure*}[t]
\centering
\includegraphics[width=0.95\textwidth]{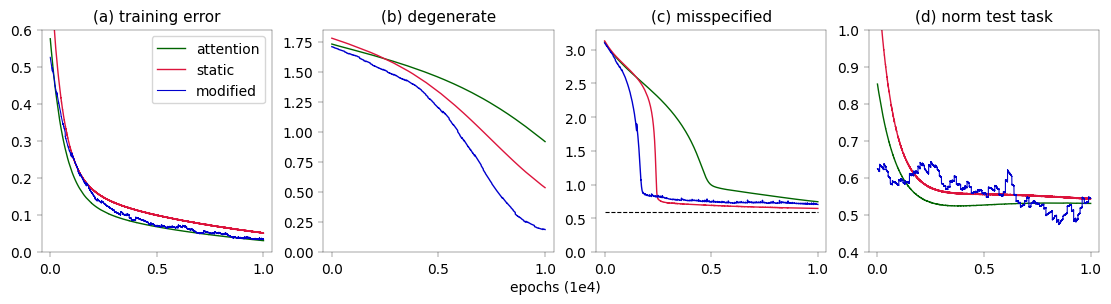}
\caption{(a) Training error of the attention, static and modified Transformers. (b) Learning degenerate features with $\rank\bSig_{\mu^\circ,\mu^\circ}<k$. (c) Training a misspecified model containing two extra features. (d) Test error for the nonlinear norm task $\norm{h_{\mu^\circ}(\bx)}$.}
\label{fig}
\end{figure*}

\section{Numerical Experiments}\label{sec:exps}

Complementing our theoretical analyses, we now explore some empirical aspects of in-context feature learning of a toy Transformer. We compare three models: the \emph{attention} Transformer is constructed as in Section \ref{sec:tfarchitecture} and jointly optimizes the loss $\LL_\textup{TF}(\mu,\bW)$, while the \emph{static} and \emph{modified} Transformers directly minimize $\LL(\mu)$ without passing through the LSA layer. All models are pretrained using SGD on 10K prompts each containing 1K token pairs. For the MLP we set $d=20$, $k=5$ with 500 sigmoid neurons and $\DD_{\XX}\sim \mathcal{N}(0,\bI_d)$. The modified model additionally implements the birth-death and perturbation dynamics of Section \ref{sec:iclrate} if $\LL$ has not decreased by 1\% every 100 epochs.

Figure \ref{fig}(a) shows that the attention and static Transformers exhibit similar dynamics and successfully converge to global optima, justifying the two-timescale approach. This is repeated for the more difficult problem of learning degenerate features ($\rank\bSig_{\mu^\circ,\mu^\circ}<k$) in Figure \ref{fig}(b), where the modified model shows significantly better performance, lending support to our perturbative scheme. Next, Figure \ref{fig}(c) plots the training curve for a misspecified model where the true features $h_{\mu^\circ}$ are 7-dimensional. While zero loss is not achievable due to the increased task complexity, all models still converge to a well-behaved minimum. Finally, we compute the test loss w.r.t. the nonlinear feature-based task $\bx\mapsto\norm{h_{\mu^\circ}(\bx)}$ in Figure \ref{fig}(d). Accuracy still sharply improves when the relevant features are learned, confirming that ICFL can generalize beyond linear regression even in one-layer Transformers and further demonstrating the importance of feature learning during pretraining.

\section{Conclusion}

In this paper, we explored the training dynamics of a Transformer with one MLP and one attention layer, enabling in-context feature learning of regression tasks on a rich class of representations. We showed that the loss landscape becomes benign in the two-timescale and mean-field limit and developed instability and improvement guarantees for the Wasserstein gradient flow. To our knowledge, this represents both the first work to theoretically study how features are learned in context, and the first analysis of nonconvex mean-field dynamics for strict saddle objectives. We hope our insights may be extended to more complex in-context learning behavior in deeper MLP-attention models.

\section*{Impact Statement}

This paper aims to deepen our perception of how in-context learning ability arises in Transformer architectures, which is intimately connected to ethical issues such as AI privacy, fairness and accountability. We hope that our study will lead to a better understanding of the reasoning capabilities of large language models and facilitate the development of more transparent and socially beneficial AI systems.

\section*{Acknowledgments}

JK was partially supported by JST CREST (JPMJCR2015). TS was partially supported by JSPS KAKENHI (20H00576) and JST CREST (JPMJCR2115).

\bibliography{icml2024.bib}
\bibliographystyle{icml2024.bst}

\clearpage
\renewcommand{\contentsname}{Table of Contents}
{
\hypersetup{linkcolor=black}
\tableofcontents
}

\addtocontents{toc}{\protect\newpage}
\newpage
\appendix
\onecolumn

\section{Preliminaries}\label{app:preliminaries}

We begin by providing some necessary background for mean-field dynamics. Let $\Omega\subseteq\RR^m$ be a Euclidean domain with smooth boundary $\partial\Omega$. For $p\geq 1$, let $\PP_p(\Omega)$ be the $p$-Wasserstein space of probability measures on $\Omega$ vanishing on $\partial\Omega$ with finite $p$th moment. We will mostly be concerned with the space $\PP_2(\Omega)$.

\begin{definition}[functional derivative]\label{def:funcderiv}
The functional derivative $\frac{\delta F}{\delta\mu}$ of a functional $F:\PP_p(\Omega)\to\RR$ is defined (if one exists) as a functional $\frac{\delta F}{\delta\mu}:\PP_p(\Omega)\times\Omega \to\RR$ satisfying for all $\nu\in\PP_p(\Omega)$,
\begin{equation*}
\frac{\rd}{\rd\epsilon}\bigg\vert_{\epsilon=0} F(\mu+\epsilon(\nu-\mu)) = \int_\Omega \frac{\delta F}{\delta\mu}(\mu,\theta) (\nu-\mu)(\rd\theta).
\end{equation*}
Note that the functional derivative is defined up to additive constants. We say a functional $F$ is $C^1$ if $\nabla\deltaF(\mu,\theta)$ is well-defined and continuous, and $C^2$ if $\nabla_\theta\nabla_{\theta'}\ddeltaF(\mu,\theta,\theta')$ is well-defined and continuous. Furthermore, the functional $F$ is convex if for all $\nu\in\PP_p(\Omega)$ it holds that
\begin{equation*}
F(\nu)\geq F(\mu)+ \int_\Omega \frac{\delta F}{\delta\mu}(\mu,\theta) (\nu-\mu)(\rd\theta).
\end{equation*}
\end{definition}

\begin{definition}[$p$-Wasserstein metric]
The $p$-Wasserstein distance between $\mu,\nu\in \PP_p(\Omega)$ is defined as
\begin{equation*}
\WW_p(\mu,\nu) = \left(\inf_{\gamma\in\Pi(\mu,\nu)} \int_{\RR^m} \norm{x-y}^p \rd \gamma(x,y)\right)^\frac{1}{p}
\end{equation*}
where $\Pi(\mu,\nu)$ denotes the set of joint distributions on $\Omega\times \Omega$ whose first and second factors have marginal laws $\mu$ and $\nu$, respectively.
\end{definition}
We consider $\PP_p(\Omega)$ as a metric space with respect to $\WW_p$, which metrizes weak convergence on $\PP_p(\Omega)$ \citep[Theorem 6.9]{Villani09}. By H\"{o}lder's inequality it always holds that $\PP_2(\Omega)\subset\PP_1(\Omega)$ and $\WW_1(\mu,\nu)\leq\WW_2(\mu,\nu)$. The $\WW_1$ metric is also characterized via Kantorovich-Rubinstein duality as
\begin{equation*}
\WW_1(\mu,\nu) = \sup_{\norm{f}_{\Lip}\leq 1} \int_\Omega f\rd\mu-\int_\Omega f\rd\nu,
\end{equation*}
where the supremum runs over all 1-Lipschitz functions $f:\Omega\to\RR$, which makes it well-suited for perturbation analyses.

We develop more advanced theory concerning the local metric geometry and characterization of flows on $\PP_2(\Omega)$ in Appendix \ref{app:measure}. As a consequence, one can show the following variational formulation of the $\WW_2$ metric:
\begin{proposition}[Benamou-Brenier formula]\label{thm:benamou}
For $\mu,\nu\in\PP_2(\Omega)$ it holds that
\begin{equation*}
\WW_2(\mu,\nu)^2 = \inf\left\{\int_0^1 \norm{\bv_t}_{L^2(\Omega,\mu_t;\RR^m)}^2 \rd t: \partial_t\mu_t +\nabla\cdot(\bv_t\mu_t) = 0, \; \mu_0=\mu,\,\mu_1=\nu\right\},
\end{equation*}
where the infimum runs over all unit time flows $(\mu_t)_{t\in [0,1]}$ from $\mu$ to $\nu$.
\end{proposition}
\begin{proof}
See e.g. \citet{Ambrosio05}, Chapter 8 or \citet{Santambrogio15}, Section 6.1.
\end{proof}

The formula can be used to bound the movement of Wasserstein flows in relation to the magnitude of the gradient field. For convenience, we will use the following time-rescaled version which is easily checked:
\begin{equation*}
\WW_2(\mu,\nu)^2 = \inf\left\{\tau\int_0^\tau \norm{\bv_t}_{L^2(\Omega,\mu_t;\RR^m)}^2 \rd t: \partial_t\mu_t +\nabla\cdot(\bv_t\mu_t) = 0, \; \mu_0=\mu,\,\mu_\tau=\nu,\, \tau>0 \right\}.
\end{equation*}
When the velocity field $\bv_t = -\nabla\frac{\delta F}{\delta\mu}(\mu_t)$ is given as the functional derivative of a given functional $F$, the dynamics can be interpreted as the continuous-time limit of a discrete gradient descent process on $F$ w.r.t. the $\WW_2$ metric via the celebrated JKO scheme \citep{Jordan98}. Specifically, the implicit Euler scheme
\begin{equation*}
\mu_\eta^{(k+1)} = \argmin_{\mu\in\PP_2(\Omega)} \frac{1}{2}\WW_2(\mu,\mu^{(k)})^2+\eta F(\mu), \quad \mu^{(0)}=\mu_0
\end{equation*}
converges weakly in the limit $\eta\to 0$ to the solution of the continuity or Fokker-Planck equation $\partial_t\mu_t=\nabla\cdot\left(\mu_t\nabla\deltaF(\mu_t)\right)$ in the sense that $\mu_\eta^{\lfloor t/\eta\rfloor} \rightharpoonup \mu_t$ for all time $t\geq 0$. Hence we refer to this process as the Wasserstein gradient flow on $\PP_2(\Omega)$ with respect to $F$.

\paragraph{Implementation.} We provide a simple summary of the proposed modified mean-field dynamics in Algorithm \ref{algo}. Here $\theta_k^{(1)}, \cdots, \theta_k^{(N)}$ denote the values of the $N$ particles at step $k$ with empirical distribution $\widehat{\mu}_k = \frac{1}{N}\sum_{j=1}^N\delta_{\theta_k^{(j)}}$, $\epsilon$ is the convergence error and $\delta_b,\delta_p$ are improvement thresholds for applying the birth-death and perturbation procedures, respectively. We also set learning rate $\eta$, perturbation step size $\eta_p$ and a waiting time $\tau$ for escaping saddle points. More generally, $\delta_b,\delta_p$ could be decreased and $\tau$ could be increased depending on the current objective value as suggested in Theorem \ref{thm:naiveescape}. In addition, the density ratio $\frac{\rd\mu_t}{\rd\pi}$ could be estimated at certain steps to directly check for the birth-death condition; see \citet{Sugiyama18} for an overview of applicable methods, especially in high dimensions. 

\begin{algorithm}[t]
\caption{Mean-field dynamics with birth-death and perturbation}
\label{algo}
\begin{algorithmic}
\REQUIRE i.i.d. samples $\theta_0^{(1)},\cdots,\theta_0^{(N)}\sim\mu_0$
\WHILE{$\LL(\widehat{\mu}_k) >\epsilon$}
\STATE Update all particles as $\theta_{k+1}^{(j)} = \theta_k^{(j)} - \eta\nabla\deltaL(\widehat{\mu}_k, \theta_k^{(j)})$, $j\in [N]$

\IF{$\LL(\widehat{\mu}_k) - \LL(\widehat{\mu}_{k+1}) \leq \delta_b$} 
\STATE Randomly replace $\lfloor \gamma N\rfloor$ neurons with i.i.d. samples from $\pi$
\ENDIF

\IF{$\LL(\widehat{\mu}_k) - \LL(\widehat{\mu}_{k+1}) \leq \delta_p$ \textbf{and} $k - k_p > \tau$}
\STATE $k_p\gets k$
\STATE Generate a Gaussian process $\bxi\sim \GP(0,\bK)$
\STATE Update all particles as $\theta_{k+1}^{(j)} = \theta_k^{(j)} - \eta_p \nabla\bxi(\theta_k^{(j)})$, $j\in [N]$
\ENDIF
\STATE $k\gets k+1$
\ENDWHILE
\end{algorithmic}
\end{algorithm}

\section{Proofs for Section \ref{sec:icfl}}\label{app:icfl}

\subsection{Barron Class Analysis of Representations}

\begin{lemma}\label{thm:barrondef}
For any $p>1$ it holds that $\norm{f}_{\Brr_1} = \norm{f}_{\Brr_p} = \norm{f}_{\Brr_\infty}$, where
\begin{equation*}
\norm{f}_{\Brr_\infty}:= \inf_{\mu:f=h_\mu}\esssup_{(\ba,\bw)\sim\mu} \norm{\ba}\norm{\bw}.
\end{equation*}
\end{lemma}

\begin{proof}
Note $\norm{f}_{\Brr_1} \leq \norm{f}_{\Brr_p} \leq \norm{f}_{\Brr_\infty}$ trivially by H\"{o}lder's inequality. For $f\in\Brr_1$, choose a measure $\mu$ such that $f=h_\mu$ and $\int \norm{\ba}\norm{\bw}\mu(\rd\theta) \leq \norm{f}_{\Brr_1}+\epsilon$ and define the nonnegative measure $\underline{\mu}$ on $\mathbb{S}^{k-1} \times \mathbb{S}^{d-1}$ as
\begin{equation*}
\underline{\mu}(A,B) = \int_{\underline{\ba}\in A, \underline{\bw}\in B} \norm{\ba}\norm{\bw} \mu(\rd\ba,\rd\bw), \quad \underline{\ba} = \frac{\ba}{\norm{\ba}}, \quad \underline{\bw} = \frac{\bw}{\norm{\bw}}
\end{equation*}
for Borel sets $A\subseteq\mathbb{S}^{k-1}$, $B\subseteq\mathbb{S}^{d-1}$. Then we can rewrite $f$ via the `projected' measure $\underline{\mu}$ as
\begin{align*}
f &= \int h_\theta(\bx)\mu(\rd\theta) = \int \norm{\ba}\norm{\bw}\cdot \underline{\ba} \sigma(\underline{\bw}^\top\bx) \mu(\rd\ba,\rd\bw) = \int \underline{\ba} \sigma(\underline{\bw}^\top\bx) \underline{\mu}(\rd\ba,\rd\bw).
\end{align*}
Factoring out the total mass of $\underline{\mu}$ to form a probability distribution on $\PP(\mathbb{S}^{k-1} \times \mathbb{S}^{d-1})$, we obtain a representation of $f$ such that the $\infty$-Barron norm becomes bounded as
\begin{equation*}
\norm{f}_{\Brr_\infty} \leq \underline{\mu}(\mathbb{S}^{k-1}, \mathbb{S}^{d-1}) \esssup_{(\ba,\bw)\sim \underline{\mu}} \norm{\ba}\norm{\bw} \leq \norm{f}_{\Brr_1}+\epsilon.
\end{equation*}
Taking $\epsilon\to 0$ shows the reverse inequality.
\end{proof}

\paragraph{Proof of Proposition \ref{thm:sobolev}.} If $f_j\in C(\XX,\RR)$ satisfies $\textstyle\inf_{\widehat{f}_j} \int_{\RR^{d-1}} \norm{\omega}_1^2|\widehat{f}_j(\omega)| <\infty$ for some transform $\widehat{f}_j$, it admits a representation
\begin{equation*}
f_j(\bx) = \int a_j\sigma(\bw^\top\bx) \mu_j(\rd a_j,\rd\bw)
\end{equation*}
for a probability distribution $\mu_j$ on $\RR\times\RR^d$ \citep{Barron93, Weinan22b}. Consider the scaled inclusion map
\begin{equation*}
\iota_j: \RR\times\RR^d\hookrightarrow \RR^k\times\RR^d, \quad \iota_j(a_j,\bw) = (ka_je_j,\bw),
\end{equation*}
where $e_j$ is the unit vector with all zeros except for a single 1 at the $j$th coordinate. Then for the averaged pushforward measure $\bar{\mu} = \frac{1}{k} \sum_{j=1}^k \iota_j\sharp\mu_j \in \PP(\Theta)$ it holds that
\begin{align*}
h_{\bar{\mu}}(\bx) =\frac{1}{k} \sum_{j=1}^k \int \ba_j\sigma(\bw^\top \bx) \iota_j\sharp\mu_j (\rd\ba_j,\rd\bw) = \sum_{j=1}^k \int a_je_j \sigma(\bw^\top \bx) \mu_j(\rd\ba_j,\rd\bw) = \sum_{j=1}^k e_jf_j(\bx) = f(\bx),
\end{align*}
and therefore $f\in\Brr$. \qed

\paragraph{Proof of Lemma \ref{thm:expressive}.} Let $\mathbf{A}^\dagger$ denote the pseudoinverse of a matrix $\mathbf{A}$. If $f^\circ = h_{\mu^\circ}$ for some distribution $\mu^\circ\in\PP(\Theta)$ with $\norm{f^\circ}_{\Brr}<\infty$, then setting $\bW^\circ = \bSig_{\mu^\circ,\mu^\circ}^\dagger = \EE{\bx}{f^\circ(\bx)f^\circ(\bx)^\top}^\dagger$,
\begin{align*}
\LL_\textup{TF} (\mu^\circ,\bW^\circ) &= \frac{1}{2}\EEbig{\bx_{\qq}}{\Norm{f^\circ(\bx_{\qq}) - \bSig_{\mu^\circ,\mu^\circ} \bW^\circ f^\circ(\bx_{\qq})}^2} \\
& = \frac{1}{2} \tr \bSig_{\mu^\circ,\mu^\circ} -\tr\left(\bSig_{\mu^\circ,\mu^\circ} \bW^\circ \bSig_{\mu^\circ,\mu^\circ} \right) + \frac{1}{2}\tr\left(\bSig_{\mu^\circ,\mu^\circ} \bW^\circ \bSig_{\mu^\circ,\mu^\circ} \bW^\circ \bSig_{\mu^\circ,\mu^\circ}\right)\\
&=0.
\end{align*}
Conversely, $\LL_\textup{TF}(\mu,\bW) = 0$ implies that $f^\circ(\bx_{\qq}) = \EE{\bx}{f^\circ(\bx)h_\mu(\bx)^\top}\bW h_\mu(\bx_{\qq})$ or $f^\circ = \bA h_\mu$ for some $\bA\in\RR^{k\times k}$. Then the pushforward measure $\bA\sharp\mu$ of $\mu$ along the map $(\ba,\bw) \mapsto (\bA\ba,\bw)$ satisfies
\begin{equation*}
h_{\bA\sharp\mu}(\bx) = \int \bA\ba\sigma(\bw^\top\bx) \mu(\rd\theta) = \bA h_\mu(\bx) = f^\circ(\bx), \quad \esssup_{\bA\sharp\mu} \norm{\ba}\norm{\bw} \leq \norm{\bA} \esssup_\mu \norm{\ba}\norm{\bw} <\infty,
\end{equation*}
thus $f^\circ = h_{\bA\sharp\mu} \in \Brr$. \qed

\paragraph{Proof of Proposition \ref{thm:spanproj}.} Since the minimization problem is standard linear regression, we can explicitly set
\begin{equation*}
\bv = \EE{\bx}{f^\circ(\bx)f^\circ(\bx)^\top}^\dagger \EE{\bx}{g(\bx)f^\circ(\bx)}, \quad \norm{\bv} \leq \norm{\EE{\bx}{f^\circ(\bx)f^\circ(\bx)^\top}^\dagger} \cdot \norm{f^\circ}_{L^2(\DD_{\XX})} \norm{g}_{L^2(\DD_{\XX})}.
\end{equation*}
Writing $g_0 = g - \bv^\top f^\circ$, we can bound
\begin{align*}
&\frac{1}{2}\EEbig{\bx_{\qq}}{\Norm{g(\bx_{\qq}) - \EE{\bx}{g(\bx)h_\mu(\bx)^\top}\bW h_\mu(\bx_{\qq})}^2}\\
&\leq \norm{\bv}^2 \EEbig{\bx_{\qq}}{\Norm{f^\circ(\bx_{\qq}) - \EE{\bx}{f^\circ(\bx)h_\mu(\bx)^\top}\bW h_\mu(\bx_{\qq})}^2} + \EEbig{\bx_{\qq}}{\Norm{g_0(\bx_{\qq}) - \EE{\bx}{g_0(\bx)h_\mu(\bx)^\top}\bW h_\mu(\bx_{\qq})}^2}\\
&\leq 2\norm{\bv}^2\epsilon +2\norm{g_0}_{L^2(\DD_{\XX})}^2 + 2\norm{g_0}_{L^2(\DD_{\XX})}^2 \norm{h_\mu}_{L^2(\DD_{\XX})}^4 \norm{\bW}^2.
\end{align*}
The statement follows by noting that
\begin{equation*}
\norm{h_\mu}_{L^2(\DD_{\XX})} \leq \left(\int \EE{\bx}{\norm{h_\theta(\bx)}^2} \mu(\rd\theta)\right)^{1/2} \leq \left(\int \norm{\ba}^2\norm{\bw}^2 \EE{\bx}{\norm{\bx}^2} \mu(\rd\theta)\right)^{1/2} = M_2^{1/2} \norm{h_\mu}_{\Brr}
\end{equation*}
from the limiting argument in Lemma \ref{thm:barrondef}. \qed

\subsection{Finite-width Approximation and Optimization}

\paragraph{Proof of Proposition \ref{thm:finapprox}.} For any network $\mu$, we may take $\bW = \EE{\bx}{f^\circ(\bx)h_\mu(\bx)^\top}^\dagger$ so that
\begin{align*}
\inf_{\bW}\LL_\textup{TF}(\mu, \bW) &\leq \frac{1}{2} \EEbig{\bx_{\qq}}{\Norm{f^\circ(\bx_{\qq}) - \EE{\bx}{f^\circ(\bx)h_\mu(\bx)^\top}\bW h_\mu(\bx_{\qq})}^2}\\
&= \frac{1}{2} \EE{\bx}{\norm{f^\circ(\bx)}^2} - \tr\left(\EE{\bx}{f^\circ(\bx)h_\mu(\bx)^\top} \bW \EE{\bx}{h_\mu(\bx)f^\circ(\bx)^\top}\right)\\&\qquad +\frac{1}{2} \tr\left(\bW^\top \EE{\bx}{h_\mu(\bx)f^\circ(\bx)^\top}\EE{\bx}{f^\circ(\bx)h_\mu(\bx)^\top} \bW \bSig_{\mu,\mu}\right)\\
&\leq \frac{1}{2} \EE{\bx}{\norm{f^\circ(\bx)}^2} - \tr\left(\EE{\bx}{f^\circ(\bx)h_\mu(\bx)^\top}\right) + \frac{1}{2} \tr\bSig_{\mu,\mu}\\
&= \frac{1}{2} \EE{\bx}{\norm{f^\circ(\bx) - h_\mu(\bx)}^2}.
\end{align*}

Now let $\mu^\circ\in\PP(\Theta)$ be a distribution such that $f^\circ = h_{\mu^\circ}$ and $\int\norm{\ba}^2 \norm{\bw}^2 \mu^\circ(\rd\theta) \leq (1+\epsilon)\norm{f^\circ}_{\Brr}^2$. Let $\theta^{(1)}, \cdots,\theta^{(N)}$ be an i.i.d. sample from $\mu^\circ$. Then from $\EE{\theta\sim\mu^\circ}{h_\theta(\bx)} = h_{\mu^\circ}(\bx)$, it holds on average that
\begin{align*}
&\mathbb{E}_{\widehat{\mu}_N}\EE{\bx}{\norm{\widehat{h}_N(\bx)-f^\circ(\bx)}^2}\\
&= \mathbb{E}_{\bx}\EEbig{\widehat{\mu}_N}{\Bigg\lVert\frac{1}{N}\sum_{j=1}^N h_{\theta^{(j)}}(\bx)-f^\circ(\bx)\Bigg\rVert^2}\\
&= \frac{1}{N^2}\sum_{j=1}^N \mathbb{E}_{\bx}\EEbig{\widehat{\mu}_N}{\norm{h_{\theta^{(j)}}(\bx)-h_{\mu^\circ}(\bx)}^2} + \frac{1}{N^2}\sum_{j\neq \ell} \mathbb{E}_{\bx}\EEbig{\widehat{\mu}_N}{(h_{\theta^{(j)}}(\bx)-h_{\mu^\circ}(\bx))^\top (h_{\theta^{(\ell)}}(\bx)-h_{\mu^\circ}(\bx))}\\
&\leq \frac{1}{N^2}\sum_{j=1}^N \mathbb{E}_{\bx}\EE{\theta^{(j)}\sim\mu^\circ}{\norm{h_{\theta^{(j)}}(\bx)}^2}\\
&\leq \frac{1}{N}\int \norm{\ba}^2 \EE{\bx}{(\bw^\top\bx)^2} \mu^\circ(\rd\theta) \\
&\leq \frac{(1+\epsilon)M_2}{N} \norm{f^\circ}_{\Brr}^2.
\end{align*}
Moreover, the path norm is bounded on average as $\EE{\widehat{\mu}_N}{\norm{\widehat{h}_N}_\mathcal{P}} \leq (1+\epsilon)\norm{f^\circ}_{\Brr}$. Then by Markov's inequality, the event $\norm{\widehat{h}_N-f^\circ}_{L^2(\DD_{\XX})}^2 > \frac{2M_2\norm{f^\circ}_{\Brr}^2}{N}$ has probability at most $\frac{1+\epsilon}{2}$, and the event $\norm{\widehat{h}_N}_\mathcal{P} > 3\norm{f^\circ}_{\Brr}$ has probability at most $\frac{1+\epsilon}{3}$. Hence the stated bounds hold with positive probability as $\epsilon\to 0$, thus for some size $N$ network $\widehat{\mu}_N$. \qed

For the propagation of chaos result, we require the following bounds.
\begin{lemma}\label{thm:expmoment}
The second moment $\mm_2(\mu_t) = \int\norm{\theta}^2 \mu_t(\rd\theta)$ satisfies $\mm_2(\mu_t)\leq e^{2L_1^2}\mm_2(\mu_0)$.
\end{lemma}

\begin{proof}
The assertion follows immediately from
\begin{equation*}
\frac{\rd}{\rd t}\mm_2(\mu_t) = \int\norm{\theta}^2 \partial_t\mu_t(\rd\theta) = -2\int\theta^\top\nabla\deltaF(\mu_t,\theta) \mu_t(\rd\theta) \leq 2L_1^2\mm_2(\mu_t).
\end{equation*}
\end{proof}

\begin{lemma}\label{thm:fournier}
Let $\mu\in\PP_2(\Omega)$ and $\theta^{(1)},\cdots,\theta^{(N)}$ be an i.i.d. sample from $\mu$ with corresponding empirical distribution $\widehat{\mu}_N = \frac{1}{N}\sum_{j=1}^N \delta_{\theta^{(j)}}$. Then for dimension $m\geq 3$ it holds that $\E{\WW_1(\mu,\widehat{\mu}_N)} \leq C_m\cdot \mm_2(\mu)^{1/2} N^{-1/m}$. The rate is replaced by $N^{-1/2}\log N$ if $m=2$ and $N^{-1/2}$ if $m=1$.
\end{lemma}

\begin{proof}
See e.g. \citet{Fournier15} for the case $m\geq 2$ and \citet{Bobkov19} for $m=1$.
\end{proof}

\paragraph{Proof of Proposition \ref{thm:chaos}.} Consider the coupled process
\begin{equation*}
\frac{\rd}{\rd t}\tilde{\theta}_t^{(j)} = -\nabla\deltaF(\mu_t,\tilde{\theta}_t^{(j)}), \quad \tilde{\theta}_0^{(j)} = \theta_0^{(j)}, \quad j\in [N]
\end{equation*}
and write the corresponding empirical distribution as $\tilde{\mu}_{t,N} = \frac{1}{N}\sum_{j=1}^N \delta_{\tilde{\theta}_t^{(j)}}$. For any finite time horizon $T\geq 0$, it holds that
\begin{align*}
\frac{1}{N}\sum_{j=1}^N \norm{\theta_T^{(j)} - \tilde{\theta}_T^{(j)}} &= \frac{1}{N}\sum_{j=1}^N \Norm{\int_0^T \nabla\deltaF(\widehat{\mu}_{t,N}, \theta_t^{(j)}) - \nabla\deltaF(\mu_t, \tilde{\theta}_t^{(j)}) \rd t} \\
&\leq \int_0^T \frac{L_2}{N}\sum_{j=1}^N \norm{\theta_t^{(j)} - \tilde{\theta}_t^{(j)}} + L_3 \WW_1(\widehat{\mu}_{t,N}, \mu_t) \rd t.
\end{align*}
Then applying Gronwall's inequality and taking the expectation over random initialization, we have for all $t\in [0,T]$
\begin{equation*}
\E{\WW_1(\widehat{\mu}_{t,N}, \tilde{\mu}_{t,N})} \leq \mathbb{E} \Bigg[\frac{1}{N}\sum_{j=1}^N \norm{\theta_t^{(j)} - \tilde{\theta}_t^{(j)}} \Bigg] \leq L_3e^{L_2T} \int_0^t \E{\WW_1(\widehat{\mu}_{s,N}, \mu_s)} \rd s.
\end{equation*}
Since each trajectory $\tilde{\theta}_t^{(j)}$ of the coupled process is an independent sample from the true distribution $\mu_t$, by Lemma \ref{thm:expmoment} and \ref{thm:fournier} it moreover holds that
\begin{align*}
\E{\WW_1(\widehat{\mu}_{t,N},\mu_t)} &\leq \E{\WW_1(\widehat{\mu}_{t,N}, \tilde{\mu}_{t,N})} + \E{\WW_1(\tilde{\mu}_{t,N}, \mu_t)}\\
&\leq L_3e^{L_2T} \int_0^t \E{\WW_1(\widehat{\mu}_{s,N}, \mu_s)} \rd s + C_m e^{L_1^2} \mm_2(\mu_0)^{1/2} N^{-1/m}
\end{align*}
with the appropriate modification when $m=1,2$. Hence another application of Gronwall's inequality yields
\begin{equation*}
\E{\WW_1(\widehat{\mu}_{t,N},\mu_t)} \leq C_m \mm_2(\mu_0)^{1/2} N^{-1/m} \exp(L_1^2 + L_3Te^{L_2T})\to 0
\end{equation*}
as $N\to \infty$. The convergence is uniform for any finite horizon $T$. \qed

\begin{remark}\label{rem:chaos}
When $F=\LL$, we rely on the Lipschitz constants obtained in Lemma \ref{thm:funcgrad} to obtain the same statement, with the caveat that the flow must not reach the singular set $\PP_2^0(\Theta)$ in order to ensure existence and regularity of the flow; this will be a recurring issue. The result is clearly still valid for mean-field dynamics incorporating birth-death by the ordinary law of large numbers, assuming the update happens at the same instant for $\widehat{\mu}_t$ and $\mu_t$. See also \citet{Rotskoff19} for a more involved study of birth-death dynamics.
\end{remark}

\begin{remark}
The above bounds are not optimized; compare for example \citet{Berthier23}. Explicit \emph{uniform-in-time} propagation of chaos bounds have recently been proved for convex mean-field Langevin dynamics \citep{Chen22, Suzuki23} and convex-concave descent-ascent dynamics \citep{Kim24}. It remains an open problem to prove such results for general nonconvex mean-field dynamics, with or without the entropic regularization framework.
\end{remark}

\section{Proofs for Section \ref{sec:landscape}}

\subsection{Auxiliary Results}\label{app:aux}

We will use the following elementary results from linear algebra without proof.

\begin{lemma}\label{thm:blocknorm}
The spectral norm of a block matrix $\bA = \begin{bmatrix}
    \bA_{1,1}&\bA_{1,2}\\\bA_{2,1}&\bA_{2,2}
\end{bmatrix}$ is bounded as $\norm{\bA}\leq \sum_{i,j=1}^2 \norm{\bA_{i,j}}$.
\end{lemma}

\begin{lemma}
The spectral and nuclear norms are dual: $\norm{\bA}_* = \max_{\norm{\bB}\leq 1} \langle \bA, \bB\rangle$ and $\norm{\bA} = \max_{\norm{\bB}_*\leq 1} \langle \bA, \bB\rangle$ for any $\bA\in\RR^{m\times m}$, $m\geq 1$. In particular, $\tr(\bA^\top\bB) \leq \norm{\bA}\norm{\bB}_*$ for any $\bA,\bB\in\RR^{m\times m}$.
\end{lemma}

\begin{lemma}
For a positive semi-definite matrix $\bA\in\RR^{k\times k}$ it holds that $\frac{1}{k}(\tr\bA)^2 \leq \tr\bA^2\leq(\tr\bA)^2$.
\end{lemma}

The neural network output is continuous and well-behaved in the following sense:

\begin{lemma}\label{thm:lip}
The map $\theta\mapsto h_\theta(\bx)$ on $\Theta$ is $(R_1^2+R_2^2\norm{\bx}^2)^{1/2}$-Lipschitz for each $\bx\in\XX$. Also, the map $\mu\mapsto h_\mu(\bx)$ on $\PP_2(\Theta)$ is $(kR_1^2+kR_2^2\norm{\bx}^2)^{1/2}$-Lipschitz w.r.t. 1-Wasserstein distance for each $\bx\in\XX$.
\end{lemma}

\begin{proof}
For $\theta_1=(\ba_1,\bw_1), \theta_2=(\ba_2,\bw_2)$ we have
\begin{align*}
\norm{h_{\theta_1}(\bx) - h_{\theta_2}(\bx)} &= \norm{\ba_1\sigma(\bw_1^\top\bx) - \ba_2\sigma(\bw_2^\top\bx)}\\
&\leq \norm{\ba_1-\ba_2}\cdot|\sigma(\bw_1^\top\bx)| + \norm{\ba_2}\cdot|\sigma(\bw_1^\top\bx)-\sigma(\bw_2^\top\bx)|\\
&\leq R_1\norm{\ba_1-\ba_2} + R_2\norm{\bw_1-\bw_2}\cdot\norm{\bx}\\
&\leq (R_1^2+R_2^2\norm{\bx}^2)^{1/2} \norm{\theta_1-\theta_2}.
\end{align*}
The difference of each coordinate $|h_{\theta_1}(\bx)_j - h_{\theta_2}(\bx)_j|$ satisfies the same bound for $1\leq j\leq k$, implying that
\begin{equation*}
|h_\mu(\bx)_j-h_\nu(\bx)_j| = \abs{\int_\Theta h_\theta(\bx)_j\mu(\rd\theta) - \int_\Theta h_\theta(\bx)_j\nu(\rd\theta)} \leq (R_1^2+R_2^2\norm{\bx}^2)^{1/2} \WW_1(\mu,\nu)
\end{equation*}
and hence $\norm{h_\mu(\bx)-h_\nu(\bx)}\leq (kR_1^2+kR_2^2\norm{\bx}^2)^{1/2} \WW_1(\mu,\nu)$.
\end{proof}

\paragraph{Proof of Lemma \ref{thm:wconv}.} The gradient flow equation for $\bW$ is given as
\begin{align*}
\frac{\rd}{\rd t}\bW_t &= -\frac{1}{2} \nabla_{\bW} |_{\bW_t}\tr\left(-2\bSig_{\mu^\circ,\mu}\bW\bSig_{\mu,\mu^\circ} + \bSig_{\mu^\circ,\mu}\bW\bSig_{\mu,\mu}\bW^\top\bSig_{\mu,\mu^\circ}\right)\\
&= -\bSig_{\mu,\mu^\circ}\bSig_{\mu^\circ,\mu}(\bW_t\bSig_{\mu,\mu}-\bI_k).
\end{align*}
Denote the singular value decomposition of $\bSig_{\mu,\mu^\circ}$ as $\bU_1\bD_1\bV_1^\top$ and the spectral decomposition of $\bSig_{\mu,\mu}$ as $\bU_2\bD_2\bU_2^\top$ where $\bU_1,\bU_2,\bV_1\in\mathcal{O}(k)$ and $\bD_j=\diag(d_{j,1},\cdots,d_{j,k})$. Since we assume $\bSig_{\mu,\mu}=\EE{\bx}{h_\mu(\bx)h_\mu(\bx)^\top}$ is positive definite, we also have $b_{2,i}>0$ for all $i$. Further defining the auxiliary matrix $\bZ_t=\bU_1^\top\bW_t\bU_2$, the dynamics for $\bZ_t$ is expressed as
\begin{align*}
\frac{\rd}{\rd t}\bZ_t &= -\bU_1^\top (\bU_1\bD_1\bV_1^\top)(\bV_1\bD_1\bU_1^\top)(\bU_1\bZ_t\bD_2\bU_2^\top-\bI_k)\bU_2\\
&= -\bD_1^2\bZ_t\bD_2+\bD_1^2\bU_1^\top\bU_2.
\end{align*}
Writing $\bU_1^\top\bU_2 = (u_{i,j})_{1\leq i,j\leq k}$, for each entry $z_{i,j}(t):=(\bZ_t)_{i,j}$ we obtain that $z_{i,j}'(t) = -d_{1,i}^2 (d_{2,j}z_{i,j}(t)-u_{i,j})$ and therefore
\begin{equation*}
\lim_{t\to\infty} z_{i,j}(t) = \begin{rcases}
\begin{dcases}
    d_{2,j}^{-1}u_{i,j} & d_{1,i}\neq 0\\
    z_{i,j}(0) & d_{1,i}=0
\end{dcases}
\end{rcases} = \mathbf{1}_{\{d_{1,i}\neq 0\}} d_{2,j}^{-1}u_{i,j} + \mathbf{1}_{\{d_{1,i}=0\}} z_{i,j}(0).
\end{equation*}
This can be recast in matrix form as $\lim_{t\to\infty}\!\bZ_t = \bD_1^\dagger\bD_1 \bU_1^\top\bU_2\bD_2^{-1} +(\bI_k-\bD_1^\dagger\bD_1)\bZ_0$, and the convergence rate is exponential. We conclude for the limit $\bW_\mu:=\lim_{t\to\infty}\bW_t$ that
\begin{equation*}
\bSig_{\mu^\circ,\mu}\bW_\mu= (\bV_1\bD_1\bU_1^\top) \bU_1 \big(\bD_1^\dagger\bD_1 \bU_1^\top\bU_2\bD_2^{-1} +(\bI_k-\bD_1^\dagger\bD_1)\bZ_0\big)\bU_2^\top = (\bV_1\bD_1\bU_1^\top)(\bU_2\bD_2^{-1}\bU_2^\top) = \bSig_{\mu^\circ,\mu}\bSig_{\mu,\mu}^{-1}.
\end{equation*}\qed

\begin{proposition}\label{thm:dense}
For any $\mu\in\PP_2(\Theta)$, $\nu\in\PP_2^+(\Theta)$, there are at most $k$ values $t\in[0,1]$ such that $(1-t)\mu+t\nu\in\PP_2^0(\Theta)$. Consequently, $\PP_2^+(\Theta)$ is dense in $\PP_2(\Theta)$.
\end{proposition}
Note in particular that $\bR\sharp\mu^\circ\in \PP_2^+(\Theta)$ for any invertible $\bR\in\BB_1(k)$ as $\bSig_{\bR\sharp\mu^\circ,\bR\sharp\mu^\circ} \succeq\underline{r}\bR\bR^\top$. This justifies the computations which appear in the statement and proof of Theorem \ref{thm:landscape}.

\begin{proof}
Suppose there exist $k+1$ distinct $t_j\in [0,1]$, $j=0,1,\cdots,k$ such that $(1-t_j)\mu+t_j\nu \in\PP_2^0(\Theta)$; note that $t_j\neq 1$ since $\nu\in\PP_2^+(\Theta)$. Then there exist nonzero vectors $\bz_j$ such that $(1-t_j)\bz_j^\top h_\mu(\bx) +t_j\bz_j^\top h_\nu(\bx) \equiv 0$, and which must be linearly dependent. Without loss of generality, let $\{\bz_j\}_{j=0}^\ell$ be a minimally dependent subset of $\{\bz_j\}_{j=0}^k$ so that $\sum_{j=0}^\ell b_j\bz_j=0$ for constants $b_j$ not all zero. Suppose $b_0\neq 0$. Then the equality
\begin{equation*}
0\equiv\sum_{j=0}^\ell b_j\bz_j^\top h_\mu(\bx) +\frac{t_jb_j}{1-t_j}\bz_j^\top h_\nu(\bx) = \bigg(\sum_{j=0}^\ell \frac{t_jb_j}{1-t_j}\bz_j^\top\bigg) h_\nu(\bx)
\end{equation*}
implies that
\begin{equation*}
\sum_{j=0}^{\ell-1}\bigg(\frac{t_j}{1-t_j}-\frac{t_\ell}{1-t_\ell}\bigg)b_j\bz_j = \sum_{j=0}^\ell \frac{t_jb_j}{1-t_j}\bz_j - \frac{t_\ell}{1-t_\ell} \sum_{j=0}^\ell b_j\bz_j = 0,
\end{equation*}
which contradicts the minimality of $\{\bz_j\}_{j=0}^\ell$ since the coefficient of $\bz_0$ is nonzero. This proves the first claim. Denseness of $\PP_2^+(\Theta)$ immediately follows: for any $\mu\in\PP_2(\Theta)$, all but finitely many mixture distributions $(1-t)\mu+t\mu^\circ$ lie in $\PP_2^+(\Theta)$, so there exists a subsequence weakly converging to $\mu$ in $\PP_2^+(\Theta)$.
\end{proof}

\begin{lemma}\label{thm:hull}
Any element $\bR\in\BB_1(k)$ can be expressed as a convex combination of finitely many elements $\bR_1,\cdots,\bR_m$ of $\mathcal{O}(k)$.  In particular, the pushforward can be defined for any $\bR\in\BB_1(k)$.
\end{lemma}

\begin{proof}
Denote the singular value decomposition of $\bR$ as $\bU\bD\bV^\top$ and denote by $\bD_1,\cdots,\bD_{2^k}$ all diagonal matrices with every diagonal element equal to $\pm 1$. Since every diagonal element of $\bD$ has absolute value at most $1$, $\bD$ is contained in the convex hull of $\bD_1,\cdots,\bD_{2^k}$ and hence $\bR$ can be written a convex combination of $\bU\bD_1\bV^\top,\cdots,\bU\bD_{2^k}\bV^\top\in\mathcal{O}(k)$.

Furthermore, writing $\bR=\sum_{j=1}^m \alpha_j\bR_j$ for $\alpha_j\in (0,1)$, $\sum_{j=1}^m \alpha_j = 1$ we may define for all $\mu\in\PP_2(\Theta)$ the pushforward measure $\bR\sharp\mu:= \sum_{j=1}^m \alpha_j\bR_j\sharp\mu$ so that
\begin{equation*}
    h_{\bR\sharp\mu}(\bx)=\int_{\Theta} h_\theta(\bx)\rd\bR\sharp\mu(\theta) = \sum_{j=1}^m\alpha_j \int_{\Theta} h_\theta(\bx)\rd\bR_j\sharp\mu(\theta) = \sum_{j=1}^m\alpha_j \int_{\Theta} \bR_j h_\theta(\bx)\rd\mu(\theta) = \bR h_\mu(\bx).
\end{equation*}
We remark that simply defining $\bR\sharp\mu$ as the pushforward along the map $\bR:(\ba,\bw)\mapsto (\bR\ba,\bw)$ would not preserve the bounded density condition (Assumption \ref{ass:pi}) for pushforwards of $\mu^\circ$.
\end{proof}

\paragraph{Proof of Lemma \ref{thm:globalminima}.} It is straightforward to check that
\begin{equation*}
\LL(\bR\sharp\mu^\circ) = \frac{1}{2} \EEbig{\bx}{\norm{h_{\mu^\circ}(\bx) - (\bSig_{\mu^\circ,\mu^\circ}\bR^\top)(\bR\bSig_{\mu^\circ,\mu^\circ}\bR^\top)^{-1}\bR h_{\mu^\circ}(\bx)}^2}=0.
\end{equation*}
Conversely, $\LL(\mu)=0$ implies that $h_{\mu^\circ}(\bx)= \bSig_{\mu^\circ,\mu}\bSig_{\mu,\mu}^{-1} h_\mu(\bx)$ a.e. Since $\bx\mapsto h_\mu(\bx)$ is always continuous, equality holds for all $\bx\in\XX$. Finally, $\bSig_{\mu^\circ,\mu}\bSig_{\mu,\mu}^{-1}$ cannot be singular since the image of $h_{\mu^\circ}$ is not constrained on a lower-dimensional subspace by Assumption \ref{ass:features}. \qed

\subsection{Proof of Theorem \ref{thm:landscape}}\label{app:landscape}

We study the first- and second-order properties of the optimization landscape for the functional $\LL$. Let us denote
\begin{equation*}
\bL_\mu=\frac{1}{2}\EEbig{\bx}{\zeta_{\mu^\circ,\mu}(\bx)\zeta_{\mu^\circ,\mu}(\bx)^\top} = \frac{1}{2}\bSig_{\mu^\circ,\mu^\circ} - \frac{1}{2}\bSig_{\mu^\circ,\mu} \bSig_{\mu,\mu}^{-1} \bSig_{\mu,\mu^\circ}
\end{equation*}
so that $\bL_\mu$ is positive semi-definite and $\tr\bL_\mu=\LL(\mu)$. Let $\bR\in\BB_1(k)$ and $\bar{\mu}_s=(1-s)\mu +s\bR\sharp\mu^\circ$ for $s\in[0,1]$. By linearity of the mean-field mapping $\mu\mapsto h_\mu$,
\begin{align*}
&\frac{\rd}{\rd s}h_{\bar{\mu}_s}(\bx) = \bR h_{\mu^\circ}(\bx)-h_\mu(\bx),\quad \frac{\rd}{\rd s}\bSig_{\mu^\circ,\bar{\mu}_s} = \bSig_{\mu^\circ,\mu^\circ} \bR^\top- \bSig_{\mu^\circ,\mu},\\
&\frac{\rd}{\rd s}\bSig_{\bar{\mu}_s,\bar{\mu}_s} = 2s \bSig_{\mu^\circ,\mu^\circ} +(1-2s)(\bR\bSig_{\mu^\circ,\mu} + \bSig_{\mu,\mu^\circ}\bR^\top) -2(1-s)\bSig_{\mu,\mu}.
\end{align*}
Then the time derivative of $\LL(\bar{\mu}_s)$ for $s\in [0,1]$ is obtained as
\begin{align*}
    \frac{\rd}{\rd s}\LL(\bar{\mu}_s) &=-\EEbig{\bx}{\zeta_{\mu^\circ,\bar{\mu}_s}(\bx)^\top\frac{\rd}{\rd s}\left(\bSig_{\mu^\circ,\bar{\mu}_s}\bSig_{\bar{\mu}_s, \bar{\mu}_s}^{-1} h_{\bar{\mu}_s}(\bx)\right)}\\
    &=-\EEbig{\bx}{\zeta_{\mu^\circ,\bar{\mu}_s}(\bx)^\top\bSig_{\mu^\circ,\bar{\mu}_s}\bSig_{\bar{\mu}_s, \bar{\mu}_s}^{-1}(\bR h_{\mu^\circ}(\bx)-h_\mu(\bx))},
\end{align*}
where we have used that
\begin{equation*}
\EEbig{\bx}{h_{\bar{\mu}_s}(\bx)\zeta_{\mu^\circ, \bar{\mu}_s}(\bx)^\top} = \EEbig{\bx}{h_{\bar{\mu}_s}(\bx)(h_{\mu^\circ}(\bx)^\top-h_{\bar{\mu}_s}(\bx)^\top \bSig_{\bar{\mu}_s,\bar{\mu}_s}^{-1}\bSig_{\bar{\mu}_s,\mu^\circ})} = 0.
\end{equation*}
In particular, the derivative at $s=0$ is equal to 
\begin{align*}
\frac{\rd}{\rd s}\bigg\vert_{s=0}\LL(\bar{\mu}_s) &= -\EEbig{\bx}{\zeta_{\mu^\circ,\mu}(\bx)^\top\bSig_{\mu^\circ,\mu}\bSig_{\mu,\mu}^{-1}(\bR h_{\mu^\circ}(\bx)-h_\mu(\bx))}\\
&= -\EEbig{\bx}{\zeta_{\mu^\circ,\mu}(\bx)^\top\bSig_{\mu^\circ,\mu}\bSig_{\mu,\mu}^{-1} \bR\zeta_{\mu^\circ,\mu}(\bx)}\\
&= -2\tr\left(\bR\bL_\mu \bSig_{\mu^\circ,\mu}\bSig_{\mu,\mu}^{-1}\right).
\end{align*}
We may choose the pushforward $\bR$ so that this quantity is minimized over $\bR\in\BB_1(k)$. Via duality of the spectral and nuclear norms, this yields
\begin{equation}\label{eqn:firstderiv}
\frac{\rd}{\rd s}\bigg\vert_{s=0}\LL(\bar{\mu}_s) =\min_{\norm{\bR}\leq 1} -2\tr\left(\bR\bL_\mu \bSig_{\mu^\circ,\mu}\bSig_{\mu,\mu}^{-1}\right) = -2\,\norm{\bL_\mu \bSig_{\mu^\circ,\mu}\bSig_{\mu,\mu}^{-1}}_* \leq 0,
\end{equation}
proving the first claim.

Now if the above first order analysis does not yield a direction of improvement (strict decrease) for $\LL$, it must be the case that $\bL_\mu \bSig_{\mu^\circ,\mu}\bSig_{\mu,\mu}^{-1}=0$. If $\mu$ is not a global minimum then $\bL_\mu\neq 0$ and hence $\rank \bSig_{\mu^\circ,\mu}\bSig_{\mu,\mu}^{-1} <k$, so that the linear regression predictions $\bSig_{\mu^\circ,\mu}\bSig_{\mu,\mu}^{-1}h_\mu(\bx)$ are contained in a lower-dimensional subspace $\{\bz\}^\perp$ for some $\bz\in\mathbb{S}^{k-1}$. This further implies that
\begin{equation*}
\LL(\mu) \geq \frac{1}{2}\EEbig{\bx}{(\bz^\top h_{\mu^\circ}(\bx))^2} = \frac{1}{2}\bz^\top\bSig_{\mu^\circ,\mu^\circ}\bz \geq \frac{1}{2}\underline{r},
\end{equation*}
confirming the critical point lower bound.

We proceed to analyze the second-order stability of critical points. The second derivative along any pushforward $\bR\in\BB_1(k)$ is computed as
\begin{align*}
\frac{\rd^2}{\rd s^2}\bigg\vert_{s=0}\LL(\bar{\mu}_s) &= - \frac{\rd}{\rd s}\bigg\vert_{s=0}\EEbig{\bx}{\zeta_{\mu^\circ,\bar{\mu}_s}(\bx)^\top\bSig_{\mu^\circ,\bar{\mu}_s}\bSig_{\bar{\mu}_s, \bar{\mu}_s}^{-1} (\bR h_{\mu^\circ}(\bx)-h_\mu(\bx))}\\
&= \EEbig{\bx}{\frac{\rd}{\rd s}\bigg\vert_{s=0} \left(\bSig_{\mu^\circ,\bar{\mu}_s}\bSig_{\bar{\mu}_s, \bar{\mu}_s}^{-1}h_{\bar{\mu}_s}(\bx)\right)^\top \bSig_{\mu^\circ,\bar{\mu}_s}\bSig_{\bar{\mu}_s, \bar{\mu}_s}^{-1} (\bR h_{\mu^\circ}(\bx)-h_\mu(\bx))}\\
&\qquad - \EEbig{\bx}{\zeta_{\mu^\circ,\bar{\mu}_s}(\bx)^\top \frac{\rd}{\rd s}\bigg\vert_{s=0} \bSig_{\mu^\circ,\bar{\mu}_s}\bSig_{\bar{\mu}_s, \bar{\mu}_s}^{-1} (\bR h_{\mu^\circ}(\bx)-h_\mu(\bx))}.
\end{align*}
The first term can be expanded as
\begin{align*}
&\mathbb{E}_{\bx}\Big[(\bR h_{\mu^\circ}(\bx)-h_\mu(\bx))^\top\bSig_{\mu,\mu}^{-1}\bSig_{\mu,\mu^\circ}\bSig_{\mu^\circ,\mu} \bSig_{\mu,\mu}^{-1}(\bR h_{\mu^\circ}(\bx)-h_\mu(\bx))\\
&\qquad - h_\mu(\bx)^\top\bSig_{\mu,\mu}^{-1} (\bR\bSig_{\mu^\circ,\mu} + \bSig_{\mu,\mu^\circ}\bR^\top -2\bSig_{\mu,\mu}) \bSig_{\mu,\mu}^{-1}\bSig_{\mu,\mu^\circ}\bSig_{\mu^\circ,\mu} \bSig_{\mu,\mu}^{-1}(\bR h_{\mu^\circ}(\bx)-h_\mu(\bx))\\
&\qquad + h_\mu(\bx)^\top\bSig_{\mu,\mu}^{-1}(\bR\bSig_{\mu^\circ,\mu^\circ}-\bSig_{\mu,\mu^\circ}) \bSig_{\mu^\circ,\mu} \bSig_{\mu,\mu}^{-1}(\bR h_{\mu^\circ}(\bx)-h_\mu(\bx))\Big]\\
&= \mathbb{E}_{\bx}\Big[(\bR h_{\mu^\circ}(\bx)+h_\mu(\bx))^\top\bSig_{\mu,\mu}^{-1}\bSig_{\mu,\mu^\circ}\bSig_{\mu^\circ,\mu} \bSig_{\mu,\mu}^{-1}(\bR h_{\mu^\circ}(\bx)-h_\mu(\bx))\\
&\qquad - h_\mu(\bx)^\top\bSig_{\mu,\mu}^{-1} (\bR\bSig_{\mu^\circ,\mu}\bSig_{\mu,\mu}^{-1}\bSig_{\mu,\mu^\circ} + \bSig_{\mu,\mu^\circ}\bR^\top\bSig_{\mu,\mu}^{-1}\bSig_{\mu,\mu^\circ}) \bSig_{\mu^\circ,\mu} \bSig_{\mu,\mu}^{-1}(\bR h_{\mu^\circ}(\bx)-h_\mu(\bx))\\
&\qquad + h_\mu(\bx)^\top\bSig_{\mu,\mu}^{-1}(\bR\bSig_{\mu^\circ,\mu^\circ}-\bSig_{\mu,\mu^\circ}) \bSig_{\mu^\circ,\mu} \bSig_{\mu,\mu}^{-1}(\bR h_{\mu^\circ}(\bx)-h_\mu(\bx))\Big]\\
&= \tr\left(\bSig_{\mu,\mu}^{-1}\bSig_{\mu,\mu^\circ}\bSig_{\mu^\circ,\mu} \bSig_{\mu,\mu}^{-1}\bR \bSig_{\mu^\circ,\mu^\circ}\bR^\top\right) - \tr\left(\bSig_{\mu^\circ,\mu} \bSig_{\mu,\mu}^{-1}\bSig_{\mu,\mu^\circ}\right)\\
&\qquad + \tr\left((\bR\bSig_{\mu^\circ,\mu^\circ} - \bR\bSig_{\mu^\circ,\mu}\bSig_{\mu,\mu}^{-1}\bSig_{\mu,\mu^\circ} -\bSig_{\mu,\mu^\circ} - \bSig_{\mu,\mu^\circ}\bR^\top\bSig_{\mu,\mu}^{-1}\bSig_{\mu,\mu^\circ}) \bSig_{\mu^\circ,\mu} \bSig_{\mu,\mu}^{-1}(\bR\bSig_{\mu^\circ,\mu}\bSig_{\mu,\mu}^{-1}-\bI_k)\right)\\
& = \tr\left(\bSig_{\mu,\mu}^{-1}\bSig_{\mu,\mu^\circ}\bSig_{\mu^\circ,\mu} \bSig_{\mu,\mu}^{-1}\bR \bSig_{\mu^\circ,\mu^\circ}\bR^\top\right) -\tr\bSig_{\mu^\circ,\mu^\circ} +2\tr\bL_\mu\\
&\qquad +2\tr\left(\bR\bL_\mu\bSig_{\mu^\circ,\mu} \bSig_{\mu,\mu}^{-1}(\bR\bSig_{\mu^\circ,\mu}\bSig_{\mu,\mu}^{-1}-\bI_k)\right) -\tr\left((\bSig_{\mu^\circ,\mu^\circ} -2\bL_\mu)(\bR^\top\bSig_{\mu,\mu}^{-1}\bSig_{\mu,\mu^\circ}+\bI_k) (\bSig_{\mu^\circ,\mu} \bSig_{\mu,\mu}^{-1}\bR-\bI_k)\right)\\
&= 2\tr\left(\bL_\mu(\bSig_{\mu^\circ,\mu} \bSig_{\mu,\mu}^{-1}\bR+\bR^\top \bSig_{\mu,\mu}^{-1}\bSig_{\mu,\mu^\circ}-\bI_k)\bSig_{\mu^\circ,\mu} \bSig_{\mu,\mu}^{-1}\bR\right),
\end{align*}
where we have taken advantage of the symmetry of $\bL_\mu$ to cancel out various terms. The second term can be expanded as
\begin{align*}
&\mathbb{E}_{\bx}\Big[- \zeta_{\mu^\circ,\mu}(\bx)^\top (\bSig_{\mu^\circ,\mu^\circ}\bR^\top-\bSig_{\mu^\circ,\mu})\bSig_{\mu,\mu}^{-1}(\bR h_{\mu^\circ}(\bx)-h_\mu(\bx))\\
&\qquad+ \zeta_{\mu^\circ,\mu}(\bx)^\top\bSig_{\mu^\circ,\mu}\bSig_{\mu,\mu}^{-1} (\bR\bSig_{\mu^\circ,\mu} + \bSig_{\mu,\mu^\circ}\bR^\top -2\bSig_{\mu,\mu})\bSig_{\mu,\mu}^{-1} (\bR h_{\mu^\circ}(\bx)-h_\mu(\bx))\Big]\\
&= 2\tr\left(\bL_\mu(-\bSig_{\mu^\circ,\mu^\circ} \bR^\top-\bSig_{\mu^\circ,\mu} + \bSig_{\mu^\circ,\mu}\bSig_{\mu,\mu}^{-1}\bR\bSig_{\mu^\circ,\mu} + \bSig_{\mu^\circ,\mu}\bSig_{\mu,\mu}^{-1}\bSig_{\mu,\mu^\circ}\bR^\top) \bSig_{\mu,\mu}^{-1}\bR\right)\\
&= -4\tr\left(\bL_\mu^2\bR^\top\bSig_{\mu,\mu}^{-1}\bR\right) +2\tr\left(\bL_\mu(\bSig_{\mu^\circ,\mu} \bSig_{\mu,\mu}^{-1}\bR-\bI_k)\bSig_{\mu^\circ,\mu} \bSig_{\mu,\mu}^{-1}\bR\right).
\end{align*}
Combining the above, we obtain
\begin{equation}\label{eqn:secondderiv}
\frac{\rd^2}{\rd s^2}\bigg\vert_{s=0}\LL(\bar{\mu}_s) = -4\tr\left(\bL_\mu^2\bR^\top\bSig_{\mu,\mu}^{-1}\bR\right) + 2\tr\left(\bL_\mu(2\bSig_{\mu^\circ,\mu} \bSig_{\mu,\mu}^{-1}\bR+\bR^\top \bSig_{\mu,\mu}^{-1}\bSig_{\mu,\mu^\circ}-2\bI_k)\bSig_{\mu^\circ,\mu} \bSig_{\mu,\mu}^{-1}\bR\right).
\end{equation}
When $\bL_\mu \bSig_{\mu^\circ,\mu}\bSig_{\mu,\mu}^{-1}=0$, we may
take $\bR\in\mathcal{O}(k)$ such that $\bSig_{\mu^\circ,\mu} \bSig_{\mu,\mu}^{-1}\bR$ is symmetric, i.e. $\bR=\bV\bU^\top$ where $\bU\bD\bV^\top$ is the singular value decomposition of $\bSig_{\mu^\circ,\mu}\bSig_{\mu,\mu}^{-1}$. Then the second trace term vanishes since $\bSig_{\mu^\circ,\mu} \bSig_{\mu,\mu}^{-1}\bR\bL_\mu = (\bL_\mu^\top\bSig_{\mu^\circ,\mu} \bSig_{\mu,\mu}^{-1}\bR)^\top = (\bL_\mu \bSig_{\mu^\circ,\mu} \bSig_{\mu,\mu}^{-1}\bR)^\top = 0$ and we have that
\begin{equation*}
\frac{\rd^2}{\rd s^2}\bigg\vert_{s=0}\LL(\bar{\mu}_s) = -4\tr\left(\bL_\mu^2\bR^\top\bSig_{\mu,\mu}^{-1}\bR\right) \leq -\frac{4}{R_1^2}\tr\bL_\mu^2 \leq -\frac{4}{kR_1^2}\LL(\mu)^2,
\end{equation*}
which moreover implies the constant bound $\frac{\rd^2}{\rd s^2}\big|_{s=0}\LL(\bar{\mu}_s)\leq -\frac{\underline{r}^2}{kR_1^2}$. This concludes the second claim. \qed

\subsection{Proof of Proposition \ref{thm:accel}}\label{app:accel}

Observe that the term $\bL_\mu \bSig_{\mu^\circ,\mu}\bSig_{\mu,\mu}^{-1}$ lower bounding the first order decrease of $\LL$ in the proof of Theorem \ref{thm:landscape} also appears in the expansion
\begin{equation*}
\bL_\mu^2 = \frac{1}{2}\bL_\mu \bSig_{\mu^\circ,\mu^\circ} - \frac{1}{2}\bL_\mu \bSig_{\mu^\circ,\mu}\bSig_{\mu,\mu}^{-1}\bSig_{\mu,\mu^\circ}.
\end{equation*}
Supposing $\norm{\bL_\mu \bSig_{\mu^\circ,\mu}\bSig_{\mu,\mu}^{-1}}_*< \frac{\delta}{2}$ then allows us to construct the following inequality,
\begin{align*}
\LL(\mu)^2 &= (\tr\bL_\mu)^2 \geq \tr\bL_\mu^2\\
&= \frac{1}{2}\tr\left(\bL_\mu\bSig_{\mu^\circ,\mu^\circ}\right) - \frac{1}{2} \tr\left(\bL_\mu \bSig_{\mu^\circ,\mu}\bSig_{\mu,\mu}^{-1}\bSig_{\mu,\mu^\circ}\right)\\
&\geq \frac{\underline{r}}{2}\LL(\mu) - \frac{1}{2}\norm{\bL_\mu \bSig_{\mu^\circ,\mu}\bSig_{\mu,\mu}^{-1}}_*\norm{\bSig_{\mu,\mu^\circ}}\\
&> \frac{\underline{r}}{2}\LL(\mu) - \frac{R_1^2\delta}{4},
\end{align*}
which implies either $4\LL(\mu) < \underline{r} -\sqrt{\underline{r}^2-4R_1^2\delta}$ or $4\LL(\mu) > \underline{r} + \sqrt{\underline{r}^2-4R_1^2\delta}$. The bounds are non-vacuous only when $\delta\leq\frac{\underline{r}^2}{4R_1^2}$ and are strictly tighter for larger $\delta$. Taking the contrapositive yields the desired statement. \qed

\subsection{Finite Prompt and Task Length}\label{app:fin}

In this subsection, we give a brief indication as to how to incorporate finite prompt length into our setting. That is, instead of \eqref{eqn:regloss} we consider the $n$-sample in-context prediction loss
\begin{equation*}
\LL_\textup{TF}^n(\mu,\bW)= \frac{1}{2}\EEbig{\bx_1,\cdots,\bx_n,\bx_{\qq}}{\Norm{f^\circ(\bx_{\qq}) - \frac{1}{n}\sum_{i=1}^n f^\circ(\bx_i) h_\mu(\bx_i)^\top\bW h_\mu(\bx_{\qq})}^2}.
\end{equation*}
By treating the sampling process as stochastic noise, we can bound the perturbation magnitude using concentration as follows (ignoring constants).
\begin{align*}
&\abs{\LL_\textup{TF}^n(\mu,\bW) - \LL_\textup{TF}(\mu,\bW)} \\
&\leq \frac{1}{2}\EEbig{\bx_1,\cdots,\bx_n,\bx_{\qq}}{\Bigg\vert\Norm{f^\circ(\bx_{\qq}) - \frac{1}{n}\sum_{i=1}^n f^\circ(\bx_i) h_\mu(\bx_i)^\top\bW h_\mu(\bx_{\qq})}^2 - \Norm{f^\circ(\bx_{\qq}) -\EE{\bx}{f^\circ(\bx) h_\mu(\bx)^\top}\bW h_\mu(\bx_{\qq})}^2 \Bigg\vert}\\
&\lesssim \EEbig{\bx_1,\cdots,\bx_n,\bx_{\qq}}{\Norm{\left(\frac{1}{n}\sum_{i=1}^n f^\circ(\bx_i) h_\mu(\bx_i)^\top -\EE{\bx}{f^\circ(\bx) h_\mu(\bx)^\top}\right)\bW h_\mu(\bx_{\qq})}}\\
&\lesssim \sqrt{\frac{\log k}{n}} \quad\text{for large enough } n,
\end{align*}
by an application of the matrix Bernstein inequality \citep[Section 1.6.3]{Tropp15}. Moreover, the joint objective is still convex in $\bW$ so that $\LL^n (\mu):= \inf_{\bW}\LL_\textup{TF}^n (\mu,\bW)$
also satisfies
\begin{equation}\label{eqn:perturbbound}
\abs{\LL^n (\mu) - \LL(\mu)} \lesssim \sqrt{\frac{\log k}{n}}
\end{equation}
uniformly over all $\mu$ of interest, assuming $\bSig_{\mu,\mu}^{-1}$ is uniformly bounded as in Theorem \ref{thm:naiveescape}. With some more work, the additional stochastic error due to a finite number $T$ of \emph{tasks} may also be bounded with high probability as $O(T^{-1/2})$.

The main issue is that even with this guarantee, the landscape $\LL^n$ may no longer be benign so that our results on mean-field dynamics in the subsequent sections do not directly apply. For an illustration, consider the $\epsilon$-perturbed quadratic function $-x^2 - \epsilon \exp(-(\frac{x}{\epsilon})^2)$ on $\RR$ with a bump at the origin; the point $x=0$ becomes a local minimum for all $\epsilon>0$. Nevertheless, it is easily shown that such attraction basins must still be small with radius at most $O(\sqrt{\epsilon})$. Similarly, combining the strict curvature bound Theorem \ref{thm:landscape}\ref{item:secondorder} and \eqref{eqn:perturbbound} with regularity estimates for $\LL^n$ as in Lemma \ref{thm:funcgrad} yields an $n^{-1/4}$ upper bound for the $\WW_2$ radius of potential local attraction basins. Hence it is plausible that the dynamics still mostly manages to avoid being trapped in local minima.

Finally, we note that complexity bounds for finite task and prompt lengths have been established for the single LSA layer model in \citet{Wu24}. We could also consider noisy data $y_i=f(\bx_i)+\varepsilon_i$, $\varepsilon_i\sim\mathcal{N}(0,\sigma_\varepsilon^2)$, which only leads to $\LL_\textup{TF}$ being shifted by a constant $\frac{1}{2}\sigma_\varepsilon^2$.

\section{Proofs for Section \ref{sec:mfd}}\label{app:measure}

\subsection{Recap: Finite-dimensional Dynamics}\label{app:recap}

To help gain intuition, we draw parallels with the ordinary GF for a $C^2$ nonconvex function $f:\RR^m\to\RR$,
\begin{equation*}
    \rd\bz_t = -\nabla_{\bz} f(\bz_t)\rd t.
\end{equation*}
A strict saddle point $\bz^\dagger$ is defined as a critical point such that $\lambda_\textup{min}(\textup{Hess}_f (\bz^\dagger)) <0$, where $\textup{Hess}_f$ is the local curvature or Hessian matrix of $f$. \citet{Lee19} show that the set of initial values $\bz_0$ for which $\lim_{t\to\infty}\bz_t$ converges to a strict saddle point has measure zero.\footnote{More precisely, this is shown for iterates of discrete gradient descent, but the proof is easily adapted to the continuous-time flow.} If every saddle point of $f$ is strict and all local minima are also global minima, $\bz_t$ converges to global minima for almost all initializations. The result follows easily from the center-stable manifold theorem \citep[Theorem III.7]{Shub13}, which states that all stable local orbits must be contained in a local embedded disk tangent to the stable eigenspace of $\textup{Hess}_f$ at $\bz^\dagger$.

\subsection{Local Geometry of Wasserstein Space}

We present some background theory on the metric geometry of Wasserstein spaces. The following result characterizes absolutely continuous curves in $\PP_2(\Omega)$.

\begin{theorem}[\citet{Ambrosio05}, Theorem 8.3.1 and Proposition 8.4.5]\label{thm:accurves}
Let $I\subset\RR$ be an open interval and $\mu_t:I\to\PP_2(\Omega)$ an absolutely continuous curve with metric derivative $\abs{\mu'}\in L^1(I)$. Then among all Borel vector fields $\bv_t\in L^2(\Omega,\mu_t)$ satisfying the continuity equation $\partial_t\mu_t+\nabla\cdot(\bv_t\mu_t) = 0$, there exists an $L^1(I)$-a.e. unique minimal norm velocity field $(\bv_t)$ such that
\begin{equation*}
\norm{\bv_t}_{L^2(\Omega,\mu_t)}\leq \abs{\mu'}(t).
\end{equation*}
The field $(\bv_t)$ is also uniquely characterized by the condition that $\bv_t$ is $L^1(I)$-a.e. contained in the $L^2(\Omega,\mu_t)$-closure of the subspace $\{\nabla\psi: \psi\in C_c^\infty(\Omega)\}$. Conversely, a narrowly continuous curve given by the continuity equation for some square-integrable Borel velocity field $\bv_t$ with $\norm{\bv_t}_{L^2(\Omega,\mu_t)}\in L^1(I)$ satisfies $\abs{\mu'}(t)\leq \norm{\bv_t}_{L^2(\Omega,\mu_t)}$ a.e.
\end{theorem}
This motivates the formal definition of the tangent space \eqref{eqn:tangentspace}. The space can also be retrieved by the following variational principle: a vector field $\bv\in L^2(\Omega,\mu)$ belongs to $\Tan_\mu\PP_2(\Omega)$ if and only if $\norm{\bv+\bw}_{L^2(\Omega,\mu)} \geq \norm{\bv}_{L^2(\Omega,\mu)}$ for all divergence-free fields $\bw\in L^2(\Omega,\mu)$ such that $\nabla\cdot(\bw\mu)=0$. Moreover, for every $\bv\in L^2(\Omega,\mu)$ there exists a unique representative $\Pi\bv\in \Tan_\mu\PP_2(\Omega)$ equivalent to $\bv$ modulo divergence-free fields. Geometrically, this allows us to describe infinitesimal transport along curves $\mu_t$ via their tangent vectors.

\begin{proposition}[\citet{Ambrosio05}, Theorem 8.3.1 and Proposition 8.4.6]\label{thm:expapprox}
Let $\mu_t:I\to\PP_2(\Omega)$ be an absolutely continuous curve with velocity field $\bv_t\in\Tan_{\mu_t}\PP_2(\Omega)$ determined as in Theorem \ref{thm:accurves}. Then for a.e. $t\in I$ we have
\begin{equation*}
\WW_2(\mu_{t+\epsilon},(\id_\Omega +\epsilon\bv_t)\sharp \mu_t) = o(\epsilon).
\end{equation*}
\end{proposition}
In light of Proposition \ref{thm:expapprox}, the tangent space can alternatively be defined using optimal transport plans. Denote by $\Gamma_o(\mu,\nu) \subset \PP_2(\Omega\times\Omega)$ the set of optimal transport plans from $\mu$ to $\nu$ with cost function the 2-norm and let
\begin{equation}\label{eqn:tanot}
\Tan_\mu\PP_2(\Omega)=\overline{\{\lambda(\br-\id_\Omega): (\id_\Omega\times\br) \sharp\mu \in \Gamma_o(\mu,\br\sharp\mu), \;\lambda>0\}}^{L^2(\Omega,\mu)};
\end{equation}
this construction is equivalent to \eqref{eqn:tangentspace} \citep[Theorem 8.5.1]{Ambrosio05}.


\subsection{Stability of Wasserstein Gradient Flow}
We now proceed with the proofs.

\paragraph{Proof of Lemma \ref{thm:localevo}.} Let $\mu^\dagger$ be a critical point of $F$, that is $\deltaF(\mu^\dagger) =0$. From the description \eqref{eqn:tanot} for the tangent space at $\mu^\dagger$, we write a local WGF $(\mu_t)$ as $\mu_t=(\id_\Omega+\epsilon\bv_t)\sharp\mu^\dagger$ for a velocity field $\bv_t\in\Tan_{\mu^\dagger}\PP_2(\Omega)$. The evolution of $\bv_t$ is derived as follows: for any smooth integrable function $g:\Omega\to\RR$, the identity $\int g\rd\mu_t = \int g\circ(\id_\Omega+\epsilon\bv_t) \rd\mu^\dagger$ implies that
\begin{align*}
\int \nabla g\cdot \nabla\deltaF(\mu_t)\rd\mu_t = -\int g\rd(\partial_t\mu_t) = -\epsilon\int \nabla g\circ(\id_\Omega+\epsilon\bv_t) \cdot\partial_t\bv_t \rd\mu^\dagger = -\epsilon\int \nabla g \cdot\partial_t\bv_t\circ(\id_\Omega+\epsilon\bv_t)^{-1} \rd\mu_t,
\end{align*}
and hence $\partial_t\bv_t = -\epsilon^{-1}\nabla\deltaF(\mu_t)\circ (\id_\Omega+\epsilon\bv_t)$. On the other hand, by Proposition \ref{thm:expapprox} we can locally approximate the pushforward displacement by the absolutely continuous curve defined by $\partial_s\tilde{\mu}_s+\nabla\cdot(\bv_t\mu_s) = 0$ initialized at $\tilde{\mu}_0=\mu^\dagger$:
\begin{align*}
\nabla\deltaF(\tilde{\mu}_\epsilon,\theta)-\nabla\deltaF(\mu^\dagger,\theta) &= \nabla\deltaF(\tilde{\mu}_s,\theta)\bigg\vert_{s=0}^\epsilon\\
&=\nabla_\theta\int_0^\epsilon \int \ddeltaF(\tilde{\mu}_s,\theta,\theta') \partial_s\tilde{\mu}_s(\rd\theta') \rd s\\
&=\nabla_\theta \int_0^\epsilon \int \nabla_{\theta'} \ddeltaF(\tilde{\mu}_s,\theta,\theta') \bv_t(\theta') \tilde{\mu}_s(\rd\theta') \rd s\\
&=\int_0^\epsilon \int \bH_{\mu^\dagger}(\theta,\theta') \bv_t(\theta') \mu^\dagger(\rd\theta') + O(\WW_2(\tilde{\mu}_s,\mu^\dagger)) \rd s\\
&= \epsilon\HH_{\mu^\dagger}\bv_t + o(\epsilon)
\end{align*}
so that
\begin{align*}
\partial_t\bv_t &= -\frac{1}{\epsilon}\bigg(\underbrace{\nabla\deltaF(\mu_t)\circ(\id_\Omega+\epsilon\bv_t)-\nabla\deltaF(\mu_t)}_{=o(\epsilon)}+ \underbrace{\nabla\deltaF(\mu_t)-\nabla\deltaF(\tilde{\mu}_\epsilon)}_{=o(\epsilon)} + \nabla\deltaF(\tilde{\mu}_\epsilon)-\nabla\deltaF(\mu^\dagger) + \underbrace{\nabla\deltaF(\mu^\dagger)}_{=0} \bigg)\\
&=-\HH_{\mu^\dagger}\bv_t + o(1).
\end{align*}
Here, we see that the $o(1)$ perturbation term is more precisely of order $O(\WW_2(\mu_t,\mu^\dagger))$ and vanishes when the $L^2$-norm of the velocity field $\bv_t$ goes to zero. \qed

\paragraph{Proof of Lemma \ref{thm:generalvalid}.} It will suffice to show $\bH_\mu$ is symmetric in the sense that $\bH_\mu(\theta,\theta')^\top = \bH_\mu(\theta',\theta)$ for all $\theta,\theta'\in\Omega$. We appeal directly to Definition \ref{def:funcderiv}: for any $\mu,\nu_1,\nu_2$,
\begin{align*}
\frac{\rd^2}{\rd\epsilon_1 \rd\epsilon_2}\bigg\vert_{\epsilon_1=\epsilon_2=0} F(\mu+\epsilon_1(\nu_1-\mu)+\epsilon_2(\nu_2-\mu)) &= \frac{\rd}{\rd\epsilon_2}\bigg\vert_{\epsilon_2=0} \int \deltaF(\mu+\epsilon_2(\nu_2-\mu), \theta) (\nu_1-\mu)(\rd\theta)\\
&= \iint \ddeltaF(\mu,\theta,\theta') (\nu_1-\mu)(\rd\theta) (\nu_2-\mu)(\rd\theta'),
\end{align*}
and comparing with the same computation with the indices swapped yields that $\ddeltaF$ is symmetric in $\theta,\theta'$. Therefore the Hessian matrix satisfies $\nabla_\theta\nabla_{\theta'}\ddeltaF(\mu,\theta,\theta') = \nabla_{\theta'}\nabla_\theta\ddeltaF(\mu,\theta',\theta)^\top$. Then for any functions $f,g\in L^2(\Omega,\mu; \RR^m)$ it holds that
\begin{align*}
\left\langle f, \HH_\mu\! g\right\rangle_{L^2(\Omega,\mu; \RR^m)} &= \iint f(\theta)^\top \bH_\mu(\theta,\theta')g(\theta')\mu(\rd\theta)\mu(\rd\theta') \\
&= \iint g(\theta)^\top \bH_\mu(\theta,\theta')f(\theta')\mu(\rd\theta)\mu(\rd\theta')\\
&= \left\langle  \HH_\mu\! f, g\right\rangle_{L^2(\Omega,\mu; \RR^m)},
\end{align*}
thus $\HH_\mu$ is self-adjoint. Since the kernel is Hilbert-Schmidt by assumption, $\HH_\mu$ is also compact, and we can invoke the spectral theorem to conclude the statement. \qed

\begin{theorem}[\citet{Gallay93}, Theorem 1.1]\label{thm:centerstable}
Let $\mathscr{E}$ be a Banach space, $\bA$ a linear operator on $\mathscr{E}$, and $f:\mathscr{E}\to\mathscr{E}$ a $C^k$ perturbation with $f(0)=0$, $Df(0)=0$, where $k>1$. Consider the differential equation
\begin{equation}\label{eqn:banachde}
\textstyle\frac{\rd}{\rd t}\bz_t=\bA\bz_t+f(\bz_t), \quad t\geq 0.
\end{equation}
Assume that $\mathscr{E}$ is the direct sum of two closed, $\bA$-invariant subspaces $\mathscr{E}^s,\mathscr{E}^u$. The corresponding restrictions $\bA^s=\bA\vert_{\mathscr{E}^s}$, $\bA^u=\bA\vert_{\mathscr{E}^u}$ generate strongly continuous semigroups $e^{\bA^s t}$, $e^{-\bA^u t}$ for $t\geq 0$ which moreover satisfy for real numbers $0\leq\lambda^s<\lambda^u$,
\begin{equation*}
\textstyle\sup_{t\geq 0} \norm{e^{\bA^s t}} e^{-\lambda^s t} < \infty, \quad \sup_{t\geq 0} \norm{e^{-\bA^u t}} e^{\lambda^u t} < \infty.
\end{equation*}
Further assume there exists a spectral gap of $\lambda^u>k\lambda^s$ and that $\mathscr{E}^s$ has the $C^k$ extension property. Let $\BB_r, \BB_r^s,\BB_r^u$ denote the balls of radius $r$ around the origin in $\mathscr{E}, \mathscr{E}^s, \mathscr{E}^u$, respectively. Then for sufficiently small $r>0$, there exists a $C^k$ map $h: \BB_r^s\to \BB_r^u$ with $h(0)=0$, $Dh(0)=0$ whose graph $\mathscr{V}\subset\BB_r$ (the local center-stable manifold) has the following properties.
\begin{enumerate}[\normalfont(i)]
\item (Invariance) For all initial values $\bz_0\in \mathscr{V}$ there exists a $C^1$ curve $\bz_t:\RR_{\geq 0}\to\mathscr{E}$ such that as long as $\bz_t\in\BB_r$, then $\bz_t\in\mathscr{V}$ and \eqref{eqn:banachde} holds. 
\item (Uniqueness) If $\bz_t$ is any solution of \eqref{eqn:banachde} such that $\bz_t\in\BB_r$ for all $t\geq 0$, then $\bz_t\in\mathscr{V}$ for all $t\geq 0$.
\end{enumerate}
\end{theorem}

\paragraph{Proof of Theorem \ref{thm:measure}.} Let $\mu^\dagger\in\mathscr{G}^\dagger$ be a strict saddle point. We apply the local center-stable manifold theorem to the system \eqref{eqn:localevo} on $L^2(\Omega,\mu^\dagger;\RR^m)$. By the spectral theorem, the operator $\HH_{\mu^\dagger}$ has a complete set of eigenvalues $\lambda_j$ and corresponding eigenfunctions $\psi_j$ for $j\in\ZZ$, ordered such that
\begin{equation*}
\lambda_1\geq\lambda_2\geq\cdots\geq 0, \quad \lambda_0 =\cdots = \lambda_{-(p-1)} <\lambda_{-p}\leq \cdots \leq 0.
\end{equation*}
Since the spectrum may possess a limit point at 0, we cannot separate $\HH_{\mu^\dagger}$ into absolutely convergent and divergent components. Instead, we set the cutoff at the largest negative eigenvalue $\lambda_0=\lambda_{\text{min}}(\HH_{\mu^\dagger})$, taking all possibly multiple eigenvalues, and defining the subspace $\mathscr{E}^u$ as the span of the corresponding $\psi_0,\cdots,\psi_{-(p-1)}$. Then we are guaranteed a jump $\lambda_{-(p-1)}<\lambda_{-p}$ since the spectrum is discrete, and we choose $\lambda_s = -\lambda_{-p}, \lambda^u = -\lambda_0$ and $k\in (1,|\lambda_0/\lambda_{-p}|)$ so that the spectral gap condition is satisfied -- we only need continuity (i.e. $k\geq 0$) for our argument. Moreover, the $C^k$ extension property for $\mathscr{E}^s$ holds automatically as $L^2(\Omega,\mu^\dagger;\RR^m)$ is a Hilbert space. Therefore, any convergent local flow $(\bv_t)$ defined in an open neighborhood $\BB_{\mu^\dagger}$ must be contained in a graph $\mathscr{V}_{\mu^\dagger}\subset\BB_{\mu^\dagger}$ containing $\mu^\dagger$.


The rest of the proof is similar to \citet{Lee19}. Since the collection $\{\BB_{\mu^\dagger}:\mu^\dagger\in\mathscr{G}^\dagger\}$ forms an open cover of $\mathscr{G}^\dagger$ and $\PP_2(\Omega)$ is separable with respect to 2-Wasserstein distance \citep[Proposition 7.1.5]{Ambrosio05}, we can extract a countable subcover $\{\BB_j:j\in\NN\}$ containing $\mathscr{G}^\dagger$. If the WGF $(\mu_t)_{t\geq 0}$ converges to a strict saddle point, there exists an index $j$ and an integer threshold $\ell$ such that $\mu_t\in\BB_j$ for $t\geq \ell$. In particular, $\mu_t$ must be contained in the corresponding center-stable manifold $\mathscr{V}_j$ for $t\geq\ell$.

Let $\omega_t^-(\nu)$ denote the result of running the reversed gradient flow $\partial\nu_{-t}= -\nabla\cdot(\nu_{-t}\nabla\deltaF(\nu_{-t}))$, $\nu_0=\nu$ for time $t$ whenever it exists; time inversion $t\mapsto -t$ shows that $\omega_t^-(\mu_t)=\mu_0$ for the forward flow $(\mu_t)_{t\geq 0}$. Since $\mu_\ell\in\mathscr{V}_j$ for some integer time $\ell$ and $\mathscr{V}_j$, it holds that
\begin{align*}
\mathscr{G}_0^\dagger\subseteq\bigcup_{j\in\NN}\bigcup_{\ell\in\NN}\omega_\ell^-(\mathscr{V}_j),
\end{align*}
hence $\mathscr{G}_0^\dagger$ must be contained in the countable union of images of graphs. \qed

\section{Proofs for Section \ref{sec:iclrate}}\label{app:iclrate}

\subsection{First-order Improvement}

\begin{proposition}\label{thm:generalfirst}
Let $F$ be a functional depending on $\mu$ only through the MLP layer $h_\mu$. Suppose MFD \eqref{eqn:mfd} at time $t$ admits a distribution $\bar{\mu}\in\PP_2(\Theta)$ with $\chi^2(\bar{\mu},\mu_t)\leq\bar{\chi}^2$ such that along the linear homotopy $\bar{\mu}_s=(1-s)\mu_t+s\bar{\mu}$ we have $\frac{\rd}{\rd s}\big|_{s=0}F(\bar{\mu}_s) \leq -\delta\leq 0$. Then $\frac{\rd}{\rd t}F(\mu_t) \leq -\bar{\chi}^{-2}\delta^2$.
\end{proposition}

\begin{proof}
We may express $F$ as $F(\mu)=J(h_\mu)$ for an auxiliary functional $h\mapsto J(h)$ defined on $C(\XX,\RR^d)$, which implies that
\begin{equation*}
\deltaF(\mu,\theta) = \int\frac{\delta J}{\delta h}(h_\mu,\bx)^\top h_\theta(\bx) \rd\bx.
\end{equation*}
In particular, since the dependency on the second layer $\ba$ is linear, it holds that $\ba^\top \nabla_{\ba} \deltaF = \deltaF$. We can then directly lower bound the decrease rate of the objective under \eqref{eqn:mfd} by isolating the gradient provided by $\ba$:
\begin{align*}
\frac{\rd}{\rd t}F(\mu_t) &= \int\deltaF(\mu_t,\theta) \partial_t\mu_t(\rd\theta)\\
&= -\int\left\lVert \nabla_\theta \deltaF(\mu_t,\theta)\right\rVert^2 \mu_t(\rd\theta)\\
&\leq -\int\left\lVert \nabla_{\ba} \deltaF(\mu_t,\theta)\right\rVert^2 \mu_t(\rd\theta)\\
&\leq -\int\left(\ba^\top \nabla_{\ba} \deltaF(\mu_t,\theta)\right)^2 \mu_t(\rd\theta)\\
&= -\int\left( \deltaF(\mu_t,\theta)\right)^2 \mu_t(\rd\theta).
\end{align*}
Starting from the first-order condition, by the Cauchy-Schwarz inequality we can also bound
\begin{align*}
\left(\frac{\rd}{\rd s}\bigg\vert_{s=0}F(\bar{\mu}_s)\right)^2 =\left(\int \deltaF(\mu_t,\theta) (\bar{\mu}-\mu_t)(\rd\theta)\right)^2 \leq \chi^2(\bar{\mu},\mu_t) \int\left( \deltaF(\mu_t,\theta)\right)^2 \mu_t(\rd\theta).
\end{align*}
Joining the two inequalities gives the desired bound.
\end{proof}

\paragraph{Proof of Proposition \ref{thm:mfdfirst}.} The functional derivative is computed as
\begin{equation}\label{eqn:deltaL}
\deltaL(\mu,\theta) = -\EEbig{\bx}{\zeta_{\mu^\circ,\mu}(\bx)^\top \bSig_{\mu^\circ,\mu}\bSig_{\mu,\mu}^{-1} h_\theta(\bx)},
\end{equation}
where we have normalized the additive constant such that the integral with respect to the current measure $\mu$ is zero, i.e. $\int\deltaL(\mu)\rd\mu=0$. This can be seen from
\begin{equation*}
\int \deltaL(\mu,\theta)\mu(d\theta) = -\int \EEbig{\bx}{\zeta_{\mu^\circ,\mu}(\bx)^\top \bSig_{\mu^\circ,\mu}\bSig_{\mu,\mu}^{-1} h_\theta(\bx)} \mu(d\theta) = -\EEbig{\bx}{\zeta_{\mu^\circ,\mu}(\bx)^\top \bSig_{\mu^\circ,\mu}\bSig_{\mu,\mu}^{-1} h_\mu(\bx)} = 0,
\end{equation*}
as in the proof of Theorem \ref{thm:landscape}. We remark that this differs from the usual normalization such that $\int \deltaL{\partial\mu}(\mu,\theta) \rd\theta = 0$. Due to the spherical symmetry of $\pi$ in the first component, it is also immediate that
\begin{equation*}
\int \deltaL(\mu)\rd\pi = -\EEbig{\bx}{\zeta_{\mu^\circ,\mu}(\bx)^\top \bSig_{\mu^\circ,\mu}\bSig_{\mu,\mu}^{-1} h_\pi(\bx)} = 0.
\end{equation*}
The chi-square divergence between $\bar{\mu} = \bR\sharp\mu^\circ$ and $\mu_t$ can be bounded as
\begin{equation*}
\int\left(\frac{\rd\bar{\mu}}{\rd\mu_t}-1 \right)^2 \rd\mu_t \leq \Norm{\frac{\rd\bar{\mu}}{\rd\mu_t}}_\infty -1 \leq\gamma^{-1} \Norm{\frac{\rd\bar{\mu}}{\rd\pi}}_\infty
\end{equation*}
where the birth-death mechanism prevents the density ratio $\frac{\rd\mu_t}{\rd\pi}$ from falling below the threshold $\gamma$ at any point. Writing the convex decomposition of $\bR$ in the sense of Lemma \ref{thm:hull} as $\sum_{j=1}^m\alpha_j\bR_j$ with $\bR_j\in \mathcal{O}(k)$, the density of $\bar{\mu}$ relative to $\pi$ is further bounded as
\begin{equation*}
    \Norm{\frac{\rd\bar{\mu}}{\rd\pi}}_\infty = \Norm{\frac{\rd\bR\sharp \mu^\circ}{\rd\pi}}_\infty \leq\sum_{j=1}^m \alpha_j \Norm{\frac{\rd\mu^\circ}{\rd\pi}}_\infty \leq R_4
\end{equation*}
by the spherical symmetry of $\pi$. Hence we may apply Proposition \ref{thm:generalfirst} with $\bar{\chi}^2 = \gamma^{-1}R_4$, showing that the objective decreases along MFD by a rate of at least $\frac{\rd}{\rd t}\LL(\mu_t) \leq -R_4^{-1}\gamma\delta^2$.

Moreover, whenever the discrete linear update is performed, along the homotopy $\hat{\mu}_s := (1-s\gamma)\mu_t+s\gamma\pi$ we have
\begin{equation*}
\frac{\rd}{\rd s}\LL(\hat{\mu}_s) = \gamma\int \deltaL(\hat{\mu}_s, \theta) (\pi-\mu_t)(\rd\theta) = 0.
\end{equation*}
Hence $t\mapsto\LL(\mu_t)$ is unaffected by the discrete updates, justifying the inequality for all time $t\geq 0$. \qed

We note that if the forcing term is applied continuously as in Remark \ref{rmk:forcing}, almost the exact same proof for Proposition \ref{thm:generalfirst} applies by bounding
\begin{equation*}
\frac{\rd}{\rd t}\LL(\mu_t) = -\int\left\lVert \nabla_\theta \deltaL(\mu_t,\theta)\right\rVert^2 \mu_t(\rd\theta) + \gamma \int\deltaL(\mu_t,\theta)(\pi-\mu_t)(\rd\theta) \leq -\int\left( \deltaL(\mu_t,\theta)\right)^2 \mu_t(\rd\theta).
\end{equation*}
As we mentioned briefly, the proof can also be easily modified to handle unbounded second layer $\ba$ by invoking the Cauchy-Schwarz inequality to lower bound the gradient
\begin{equation*}
    \int\left\lVert \nabla_{\ba} \deltaL(\mu_t,\theta)\right\rVert^2 \mu_t(\rd\theta)\geq \left(\int\norm{\ba}^2\mu_t(\rd\theta)\right)^{-1}\int\left(\ba^\top \nabla_{\ba} \deltaL(\mu_t,\theta)\right)^2 \mu_t(\rd\theta)
\end{equation*}
and bounding the second moment uniformly in time with the following result,
\begin{lemma}\label{thm:secondmoment}
Denote the second moment of $\mu\in\PP_2(\Theta)$ along the $\ba$ component as $\mm_{\ba}(\mu) =\int\norm{\ba}^2\mu(\rd\theta)$. Then the mean-field dynamics $\mu_t$ for all time $t\geq 0$ satisfies $\mm_{\ba}(\mu_t) \leq\mm_{\ba}(\mu_0)\vee\mm_{\ba}(\pi)$.
\end{lemma}

\begin{proof}
In fact, $\mm_{\ba}(\cdot)$ remains unchanged by gradient flow:
\begin{align*}
    \frac{\rd}{\rd t}\mm_{\ba}(\mu_t) &= \int \norm{\ba}^2 \partial_t\mu_t(\rd\theta)\\
    &= -2\int (\ba\;\;0_d)^\top\nabla_\theta\deltaL(\mu_t,\theta) \mu_t(\rd\theta)\\
    &= -2\int \ba^\top\nabla_{\ba}\deltaL(\mu_t,\theta) \mu_t(\rd\theta)\\
    &= -2\int \deltaL(\mu_t,\theta) \mu_t(\rd\theta) = 0.
\end{align*}
Also if the discrete update is performed, the output satisfies $\mm_{\ba}((1-\gamma)\mu_t + \gamma\pi) = (1-\gamma)\mm_{\ba}(\mu_t) + \gamma\mm_{\ba}(\pi)$ by linearity of the moment functional $\mu\mapsto\mm_{\ba}(\mu)$. Hence $\mm_{\ba}(\mu_t)$ always interpolates between $\mm_{\ba}(\mu_0)$ and $\mm_{\ba}(\pi)$.
\end{proof}

\paragraph{Proof of Theorem \ref{thm:accelrate}.} Shifting the time index, suppose $\LL(\mu_0)\leq 0.49\underline{r}$. Then $\LL(\mu_t)\leq 0.49\underline{r}$ for all $t\geq 0$ and by Proposition \ref{thm:accel} we are guaranteed a direction of improvement $\bar{\mu}_s = (1-s)\mu_t+s\bar{\mu}$ with $\bar{\mu}=\bR\sharp\mu$ for some $\bR\in\BB_1(k)$ such that
\begin{equation*}
\frac{\rd}{\rd s}\bigg\vert_{s=0}\LL(\bar{\mu}_s) \leq -\frac{4}{R_1^2}\LL(\mu_t) \left(\frac{\underline{r}}{2} - \LL(\mu_t)
\right).
\end{equation*}
Proposition \ref{thm:mfdfirst} then ensures the objective decreases along the Wasserstein flow as
\begin{equation*}
\frac{\rd}{\rd t}\LL(\mu_t) \leq -\frac{16\gamma}{R_1^4 R_4} \LL(\mu_t)^2 \left(\frac{\underline{r}}{2} - \LL(\mu_t)\right)^2, \quad 0\leq\LL(\mu_t)\leq\frac{\underline{r}}{2}.
\end{equation*}
We now divide the band into two halves.
\begin{enumerate}[\normalfont(i)]
    \item $\frac{\underline{r}}{4}\leq\LL\leq\frac{\underline{r}}{2}$ (acceleration band). By substituting $\LL(\mu_t)^2 \geq \frac{\underline{r}^2}{16}$ above and solving the differential inequality, we obtain
    \begin{equation*}
        \LL(\mu_t)\leq \frac{\underline{r}}{2} - \left(\frac{100}{\underline{r}} - \frac{\underline{r}^2 \gamma t}{R_1^4 R_4}\right)^{-1}
    \end{equation*}
    and hence $\LL(\mu_t)$ decreases below $\frac{\underline{r}}{4}$ after time $t_1 \leq\frac{96R_1^4 R_4}{\underline{r}^3\gamma}$.

    \item $0\leq\LL\leq\frac{\underline{r}}{4}$ (deceleration band). By substituting $(\frac{\underline{r}}{2}-\LL(\mu_t))^2 \geq \frac{\underline{r}^2}{16}$ we likewise obtain
    \begin{equation*}
        \LL(\mu_t)\leq \left(\frac{4}{\underline{r}} + \frac{\underline{r}^2 \gamma (t-t_1)}{R_1^4 R_4}\right)^{-1}
    \end{equation*}
    and hence $\LL(\mu_t)$ achieves loss $\leq\epsilon$ after time $t_1+ \frac{R_1^4 R_4}{\underline{r}^2\gamma}\cdot\frac{1}{\epsilon}$.
\end{enumerate}
Finally, note that the second term dominates the first since $\epsilon = O(\underline{r})$. \qed

\subsection{Second-order Improvement}

\paragraph{Proof of Lemma \ref{thm:evo}.} It is straightforward to show that
\begin{align*}
\partial_t\left[\nabla_\theta\deltaF(\mu_t,\theta)\right] &= \nabla_\theta \int\ddeltaF(\mu_t,\theta,\theta') (\partial_t\mu_t)(\rd\theta')\\
&= - \nabla_\theta \int \nabla_{\theta'}\ddeltaF(\mu_t,\theta,\theta')\cdot \nabla_{\theta'}\deltaF(\mu_t,\theta') \mu_t(\rd\theta')\\
&= -\int\bH_{\mu_t}(\theta,\theta') \nabla_{\theta'}\deltaF(\mu_t,\theta') \mu_t(\rd\theta').
\end{align*}
Each term is well-defined as soon as the kernel is assumed to be Hilbert-Schmidt, or due to Lemma \ref{thm:kernelvalid} for the case $F=\LL$.
\qed

\begin{lemma}\label{thm:kernelvalid}
The kernel $\bH_\mu$ for the functional $\LL$ is Hilbert-Schmidt for all $\mu\in\PP_2^+(\Theta)$. Moreover, the corresponding integral operator $\HH_\mu f(\theta) = \int \bH_\mu(\theta,\theta') f(\theta')\mu(\rd\theta')$ is compact self-adjoint, hence there exists an orthonormal basis $\{\psi_j\}_{j\in\ZZ}$\, for $L^2(\Theta,\mu; \RR^{k+d})$ consisting of eigenfunctions of $\HH_\mu$.
\end{lemma}
\begin{proof}
We extend our notation to write for example $\bSig_{\mu,\theta} = \EE{\bx\sim\DD_{\XX}}{h_\mu(\bx)h_\theta(\bx)^\top}$. From \eqref{eqn:deltaL} the second order functional derivative can be derived as
\begin{align*}
\ddeltaL(\mu,\theta,\theta') &= -\frac{\delta}{\delta\mu}\EEbig{\bx}{(h_{\mu^\circ}(\bx) - \bSig_{\mu^\circ,\mu}\bSig_{\mu,\mu}^{-1} h_\mu(\bx))^\top \bSig_{\mu^\circ,\mu}\bSig_{\mu,\mu}^{-1} h_\theta(\bx)}(\theta')\\
&= -\tr\left(\bSig_{\mu^\circ,\theta'}\bSig_{\mu,\mu}^{-1}\bSig_{\theta,\mu^\circ}\right) \\
&\qquad + \tr\left(\bSig_{\mu^\circ,\mu}\bSig_{\mu,\mu}^{-1} (\bSig_{\theta',\mu} + \bSig_{\mu,\theta'}) \bSig_{\mu,\mu}^{-1}\bSig_{\theta,\mu^\circ}\right) \\
&\qquad + \tr\left(\bSig_{\mu,\mu}^{-1}\bSig_{\mu,\mu^\circ} \bSig_{\mu^\circ,\mu}\bSig_{\mu,\mu}^{-1} \bSig_{\theta,\theta'}\right) \\
&\qquad - \tr\left(\bSig_{\mu,\mu}^{-1} (\bSig_{\theta',\mu} + \bSig_{\mu,\theta'}) \bSig_{\mu,\mu}^{-1}\bSig_{\mu,\mu^\circ} \bSig_{\mu^\circ,\mu}\bSig_{\mu,\mu}^{-1} \bSig_{\theta,\mu}\right)\\
&\qquad +\tr\left(\bSig_{\mu,\mu}^{-1} \bSig_{\theta',\mu^\circ}\bSig_{\mu^\circ,\mu}\bSig_{\mu,\mu}^{-1} \bSig_{\theta,\mu}\right)\\
&\qquad +\tr\left(\bSig_{\mu,\mu}^{-1} \bSig_{\mu,\mu^\circ}\bSig_{\mu^\circ,\theta'}\bSig_{\mu,\mu}^{-1} \bSig_{\theta,\mu}\right)\\
&\qquad -\tr\left(\bSig_{\mu,\mu}^{-1}\bSig_{\mu,\mu^\circ} \bSig_{\mu^\circ,\mu}\bSig_{\mu,\mu}^{-1} (\bSig_{\theta',\mu} + \bSig_{\mu,\theta'}) \bSig_{\mu,\mu}^{-1}\bSig_{\theta,\mu}\right).
\end{align*}
It is tedious but straightforward to check that this expression is symmetric in $\theta,\theta'$ (which would otherwise follow directly if we had a priori second order regularity estimates for $\LL$). We then have that
\begin{equation*}
[\bH_\mu(\theta,\theta')]_{i,j} = \partial_{\theta_i}\partial_{\theta_j'} \ddeltaL(\mu,\theta,\theta') = \partial_{\theta_j'}\partial_{\theta_i} \ddeltaL(\mu,\theta',\theta) =[\bH_\mu(\theta',\theta)]_{j,i}
\end{equation*}
which implies $\HH_\mu$ is self-adjoint as before. For the proof of the first claim, we refer to the uniform spectral bound for $\bH_\mu$ obtained in Lemma \ref{thm:kernelreg}; this also shows that $\HH_\mu$ is compact.
\end{proof}

In Lemma \ref{thm:funcgrad} and \ref{thm:kernelreg}, we derive various regularity bounds of the ICFL objective $\LL$. The constants $C_1,\cdots,C_5$, numbered as to be consistent with Theorem \ref{thm:generalescape}, are explicitly defined during the proofs and have at most polynomial dependency on all problem constants.

\begin{lemma}\label{thm:funcgrad}
The gradients of the functional derivative of $\LL$ at any $\mu\in\PP_2^+(\Theta)$ such that $\lambda_\textup{min}(\bSig_{\mu,\mu}) \geq\lambda$ uniformly satisfy $\norm{\nabla_{\ba} \deltaL}\leq C_{\ba}$, $\norm{\nabla_{\bw} \deltaL}\leq C_{\bw}$ and $\norm{\nabla \deltaL}\leq C_1$. Moreover, $\nabla \deltaL$ is $C_2$-Lipschitz on $\Theta$, where $C_{\ba}, C_2=O(\frac{1}{(k\lambda)^{1/2}})$ and $C_{\bw}, C_1=O(\frac{1}{k\lambda})$.
\end{lemma}
\begin{proof}
The gradient with respect to each component is given by
\begin{equation*}
\nabla_{\ba} \deltaL(\mu,\theta) 
=-\EEbig{\bx}{\zeta_{\mu^\circ,\mu}(\bx)^\top \bSig_{\mu^\circ,\mu}\bSig_{\mu,\mu}^{-1} \sigma(\bw^\top\bx)}^\top, \quad \nabla_{\bw} \deltaL(\mu,\theta) = -\EEbig{\bx}{\zeta_{\mu^\circ,\mu}(\bx)^\top \bSig_{\mu^\circ,\mu}\bSig_{\mu,\mu}^{-1} \ba\sigma'(\bw^\top\bx)\bx}.
\end{equation*}
Hence we can bound
\begin{align*}
    \Norm{\nabla_{\ba} \deltaL(\mu,\theta)}^2 &\leq \EEbig{\bx}{\norm{\zeta_{\mu^\circ,\mu}(\bx)^\top \bSig_{\mu^\circ,\mu}\bSig_{\mu,\mu}^{-1} \sigma(\bw^\top\bx)}^2}\\
    &\leq R_1^2\cdot \EEbig{\bx}{\zeta_{\mu^\circ,\mu}(\bx)^\top \bSig_{\mu^\circ,\mu}\bSig_{\mu,\mu}^{-2}\bSig_{\mu,\mu^\circ} \zeta_{\mu^\circ,\mu}(\bx)}\\
    &\leq \frac{R_1^2}{\lambda} \tr\left(\bL_\mu \bSig_{\mu^\circ,\mu}\bSig_{\mu,\mu}^{-1}\bSig_{\mu,\mu^\circ}\right)\\
    &\leq \frac{\overline{r} R_1^2}{\lambda}\LL(\mu) - \frac{2R_1^2}{\lambda}\tr\bL_\mu^2\\
    &\leq \frac{k\overline{r}^2 R_1^2}{2\lambda} =: C_{\ba}^2,
\end{align*}
and also
\begin{align*}
    \Norm{\nabla_{\bw} \deltaL(\mu,\theta)}^2 &\leq\EEbig{\bx}{\norm{\zeta_{\mu^\circ,\mu}(\bx)^\top \bSig_{\mu^\circ,\mu}\bSig_{\mu,\mu}^{-1} \ba\sigma'(\bw^\top\bx)\bx}^2}\\
    &\leq R_2^2\cdot \EEbig{\bx}{\norm{\zeta_{\mu^\circ,\mu}(\bx)^\top \bSig_{\mu^\circ,\mu}\bSig_{\mu,\mu}^{-1}}^2 \norm{\bx}^2}\\
    &\leq \frac{R_2^2}{\lambda}\cdot \EEbig{\bx}{\norm{\zeta_{\mu^\circ,\mu}(\bx)^\top \bSig_{\mu^\circ,\mu}\bSig_{\mu,\mu}^{-1/2}}^4}^{1/2} \EE{\bx}{\norm{\bx}^4}^{1/2}\\
    &\leq \frac{R_2^2 M_4^{1/2}}{\lambda} \left(\tr\EEbig{\bx}{\left(\bSig_{\mu^\circ,\mu}\bSig_{\mu,\mu}^{-1}\bSig_{\mu,\mu^\circ} \zeta_{\mu^\circ,\mu}(\bx)\zeta_{\mu^\circ,\mu}(\bx)^\top\right)^2}\right)^{1/2}\\
    &\leq \frac{\overline{r} R_2^2 M_4^{1/2}}{\lambda} \left(\tr\EEbig{\bx}{\zeta_{\mu^\circ,\mu}(\bx)\zeta_{\mu^\circ,\mu}(\bx)^\top \bSig_{\mu^\circ,\mu}\bSig_{\mu,\mu}^{-1}\bSig_{\mu,\mu^\circ} \zeta_{\mu^\circ,\mu}(\bx)\zeta_{\mu^\circ,\mu}(\bx)^\top}\right)^{1/2}\\
    &\qquad - \frac{2R_2^2 M_4^{1/2}}{\lambda} \left(\tr \bL_\mu\EEbig{\bx}{\zeta_{\mu^\circ,\mu}(\bx)\zeta_{\mu^\circ,\mu}(\bx)^\top \bSig_{\mu^\circ,\mu}\bSig_{\mu,\mu}^{-1}\bSig_{\mu,\mu^\circ} \zeta_{\mu^\circ,\mu}(\bx)\zeta_{\mu^\circ,\mu}(\bx)^\top}\right)^{1/2}\\
    &\leq \frac{\overline{r} R_2^2 M_4^{1/2}}{\lambda} \left(\overline{r} \tr\EEbig{\bx}{\left(\zeta_{\mu^\circ,\mu}(\bx)\zeta_{\mu^\circ,\mu}(\bx)^\top\right)^2} - 2\tr \bL_\mu \EEbig{\bx}{\left(\zeta_{\mu^\circ,\mu}(\bx)\zeta_{\mu^\circ,\mu}(\bx)^\top\right)^2}\right)^{1/2}\\
    &\leq \frac{\overline{r}^{3/2} R_2^2 M_4^{1/2}}{\lambda}\textstyle\sup_{\bx}\norm{\zeta_{\mu^\circ,\mu}(\bx)}(2\tr\bL_\mu)^{1/2}\\
    &\leq \frac{2k^{1/2}\overline{r}^{5/2} R_1^3 R_2^2 M_4^{1/2}}{\lambda^2} =: C_{\bw}^2,
\end{align*}
where for the last line we have used the coarser bounds $\norm{\zeta_{\mu^\circ,\mu}(\bx)} \leq R_1 + R_1^3\lambda^{-1}$ and $\lambda\leq\frac{1}{k}\tr\bSig_{\mu_0,\mu_0}\leq\frac{R_1^2}{k}$. Combining the two bounds yields
\begin{equation*}
\Norm{\nabla\deltaL(\mu,\theta)}\leq \left(\frac{\overline{r}^2 R_1^3}{2\lambda^2} (R_1 + 4k^{1/2}\overline{r}^{1/2} R_2^2 M_4^{1/2})\right)^{1/2} =: C_1.
\end{equation*}

Furthermore, for $\theta_1=(\ba_1,\bw_1)$, $\theta_2=(\ba_2,\bw_2)$ we have
\begin{align*}
\Norm{\nabla_{\ba}\deltaL(\mu,\theta_1) - \nabla_{\ba}\deltaL(\mu,\theta_2)} & = \Norm{\EEbig{\bx}{\zeta_{\mu^\circ,\mu}(\bx)^\top \bSig_{\mu^\circ,\mu}\bSig_{\mu,\mu}^{-1} (\sigma(\bw_1^\top\bx)-\sigma(\bw_2^\top\bx))}}\\
&\leq R_2\cdot\EEbig{\bx}{\norm{\zeta_{\mu^\circ,\mu}(\bx)^\top \bSig_{\mu^\circ,\mu}\bSig_{\mu,\mu}^{-1}}\cdot\norm{\bw_1^\top\bx -\bw_2^\top\bx}}\\
&\leq R_2 M_2^{1/2}\left(\frac{k\overline{r}^2}{2\lambda}\right)^{1/2} \norm{\bw_1-\bw_2},
\end{align*}
and also
\begin{align*}
\Norm{\nabla_{\bw}\deltaL(\mu,\theta_1) - \nabla_{\bw}\deltaL(\mu,\theta_2)} & = \Norm{\EEbig{\bx}{\zeta_{\mu^\circ,\mu}(\bx)^\top \bSig_{\mu^\circ,\mu}\bSig_{\mu,\mu}^{-1} (\ba_1\sigma'(\bw_1^\top\bx)-\ba_2\sigma'(\bw_2^\top\bx)) \bx}}\\
&\leq \Norm{\EEbig{\bx}{\zeta_{\mu^\circ,\mu}(\bx)^\top \bSig_{\mu^\circ,\mu}\bSig_{\mu,\mu}^{-1} (\ba_1-\ba_2)\sigma'(\bw_1^\top\bx) \bx}}\\
&\qquad + \Norm{\EEbig{\bx}{\zeta_{\mu^\circ,\mu}(\bx)^\top \bSig_{\mu^\circ,\mu}\bSig_{\mu,\mu}^{-1} \ba_2(\sigma'(\bw_1^\top\bx) -\sigma'(\bw_2^\top\bx)) \bx}}\\
&\leq R_2M_2^{1/2} \left(\frac{k\overline{r}^2}{2\lambda}\right)^{1/2}\norm{\ba_1-\ba_2}\\
&\qquad + R_3 \cdot \EEbig{\bx}{\norm{\zeta_{\mu^\circ,\mu}(\bx)^\top \bSig_{\mu^\circ,\mu}\bSig_{\mu,\mu}^{-1}}\cdot \norm{\bw_1-\bw_2}\norm{\bx}^2}\\
&\leq R_2M_2^{1/2} \left(\frac{k\overline{r}^2}{2\lambda}\right)^{1/2}\norm{\ba_1-\ba_2} + R_3M_4^{1/2} \left(\frac{k\overline{r}^2}{2\lambda}\right)^{1/2}\norm{\bw_1-\bw_2}.
\end{align*}
Combining the two yields that
\begin{equation*}
\Norm{\nabla\deltaL(\mu,\theta_1) - \nabla\deltaL(\mu,\theta_2)} \leq \left(\frac{k\overline{r}^2}{2\lambda} (2R_2^2 M_2 + R_3^2 M_4)\right)^{1/2} \norm{\theta_1-\theta_2} =: C_2\, \norm{\theta_1-\theta_2}.
\end{equation*}
\end{proof}

\begin{lemma}\label{thm:kernelreg}
For any $\mu\in\PP_2^+(\Theta)$ such that $\lambda_\textup{min}(\bSig_{\mu,\mu}) \geq\lambda$ it holds that $\norm{\bH_\mu(\theta,\theta')}\leq C_3$, $\bH_\mu(\theta,\theta')$ is uniformly $C_4$-Lipschitz w.r.t. $\theta$ and $\theta'$, and $\bH_\mu$ is $C_5$-Lipschitz w.r.t. $\mu$ in 1-Wasserstein distance, where $C_3, C_4=O(\lambda^{-2})$ and $C_5=O(d\lambda^{-3})$.
\end{lemma}

\begin{proof}
To derive regularity estimates of $\bH_\mu$, we start from the expansion in Lemma \ref{thm:kernelvalid} and perform explicit computations for only the first trace term $t(\mu,\theta,\theta') = \tr\left(\bSig_{\mu^\circ,\theta'}\bSig_{\mu,\mu}^{-1}\bSig_{\theta,\mu^\circ}\right)$. $\nabla_\theta\nabla_{\theta'} t(\mu,\theta,\theta')$ consists of block matrices
\begin{align*}
\nabla_{\ba}\nabla_{\ba'}t(\mu,\theta,\theta') &= \EEbig{\bx}{\sigma(\bw^\top\bx) h_{\mu^\circ}(\bx)^\top} \EEbig{\bx}{\sigma(\bw'^\top\bx) h_{\mu^\circ}(\bx)}\bSig_{\mu,\mu}^{-1},\\
\nabla_{\ba}\nabla_{\bw'}t(\mu,\theta,\theta') &= \bSig_{\mu,\mu}^{-1}\ba' \EEbig{\bx}{\sigma(\bw^\top\bx) h_{\mu^\circ}(\bx)^\top} \EEbig{\bx}{\sigma'(\bw'^\top\bx) h_{\mu^\circ}(\bx)\bx^\top},\\
\nabla_{\bw}\nabla_{\ba'}t(\mu,\theta,\theta') &= \EEbig{\bx}{\sigma'(\bw^\top\bx)\bx h_{\mu^\circ}(\bx)^\top} \EEbig{\bx}{\sigma(\bw'^\top\bx) h_{\mu^\circ}(\bx)} \ba^\top \bSig_{\mu,\mu}^{-1},\\
\nabla_{\bw}\nabla_{\bw'}t(\mu,\theta,\theta') &= \EEbig{\bx}{\sigma'(\bw^\top\bx)\bx h_{\mu^\circ}(\bx)^\top} \EEbig{\bx}{\sigma'(\bw'^\top\bx) h_{\mu^\circ}(\bx)\bx^\top} \ba'^\top \bSig_{\mu,\mu}^{-1}\ba.
\end{align*}
It follows from Lemma \ref{thm:blocknorm} that $\norm{\nabla_\theta\nabla_{\theta'} t(\mu,\theta,\theta')} \leq (R_1^4+2R_1^2R_2k^{1/2}\overline{r}^{1/2} M_2^{1/2}+R_2^2k\overline{r} M_2)\lambda^{-1} = O(\lambda^{-1})$. Each term of $\bH_\mu$ is likewise uniformly bounded so that $\bH_\mu$ is a valid kernel.

The Lipschitz constant of $\nabla_{\theta}\nabla_{\theta'} t(\mu,\theta,\theta')$ w.r.t. $\theta$ can also be controlled by separately bounding
\begin{align*}
&\norm{\nabla_{\ba}\nabla_{\ba'}t(\mu,\theta_1,\theta') -  \nabla_{\ba}\nabla_{\ba'}t(\mu,\theta_2,\theta')}\\
&\qquad \leq \EEbig{\bx}{|\sigma(\bw_1^\top\bx)-\sigma(\bw_2^\top\bx)|\cdot \norm{h_{\mu^\circ}(\bx)}} \EEbig{\bx}{|\sigma(\bw'^\top\bx)|\cdot\norm{h_{\mu^\circ}(\bx)}} \norm{\bSig_{\mu,\mu}^{-1}}\\
&\qquad\leq R_1R_2 M_2^{1/2} k\overline{r}\lambda^{-1} \norm{\bw_1-\bw_2},\\
&\norm{\nabla_{\ba}\nabla_{\bw'}t(\mu,\theta_1,\theta') -  \nabla_{\ba}\nabla_{\bw'}t(\mu,\theta_2,\theta')}\\
&\qquad\leq \norm{\bSig_{\mu,\mu}^{-1}\ba'}\cdot \EEbig{\bx}{|\sigma(\bw_1^\top\bx) -\sigma(\bw_2^\top\bx)|\cdot\norm{h_{\mu^\circ}(\bx)}} \EEbig{\bx}{|\sigma'(\bw'^\top\bx)|\cdot\norm{h_{\mu^\circ}(\bx) \bx^\top}} \\
&\qquad\leq R_2^2M_2k\overline{r} \lambda^{-1} \norm{\bw_1-\bw_2},\\
&\norm{\nabla_{\bw}\nabla_{\ba'}t(\mu,\theta_1,\theta') -  \nabla_{\bw}\nabla_{\ba'}t(\mu,\theta_2,\theta')}\\
&\qquad\leq \EEbig{\bx}{|\sigma'(\bw_1^\top\bx) - \sigma'(\bw_2^\top\bx)|\cdot\norm{\bx h_{\mu^\circ}(\bx)^\top}} \EEbig{\bx}{|\sigma(\bw'^\top\bx)|\cdot\norm{h_{\mu^\circ}(\bx)}} \norm{\ba_1^\top \bSig_{\mu,\mu}^{-1}}\\
&\qquad\qquad + \EEbig{\bx}{|\sigma'(\bw_2^\top\bx)|\cdot\norm{\bx h_{\mu^\circ}(\bx)^\top}} \EEbig{\bx}{|\sigma(\bw'^\top\bx)|\cdot\norm{h_{\mu^\circ}(\bx)}} \norm{(\ba_1-\ba_2)^\top \bSig_{\mu,\mu}^{-1}}\\
&\qquad\leq R_1R_3M_4^{1/2} k\overline{r} \lambda^{-1}\norm{\bw_1-\bw_2} + R_1R_2 M_2^{1/2}k\overline{r}\lambda^{-1}\norm{\ba_1-\ba_2},\\
&\norm{\nabla_{\bw}\nabla_{\bw'}t(\mu,\theta_1,\theta') -  \nabla_{\bw}\nabla_{\bw'}t(\mu,\theta_2,\theta')}\\
&\qquad\leq \EEbig{\bx}{|\sigma'(\bw_1^\top\bx) - \sigma'(\bw_2^\top\bx)|\cdot\norm{\bx h_{\mu^\circ}(\bx)^\top}} \EEbig{\bx}{|\sigma'(\bw'^\top\bx)|\cdot\norm{h_{\mu^\circ}(\bx)\bx^\top}} \norm{\ba'^\top \bSig_{\mu,\mu}^{-1}\ba_1}\\
&\qquad\qquad + \EEbig{\bx}{|\sigma'(\bw_2^\top\bx)|\cdot\norm{\bx h_{\mu^\circ}(\bx)^\top}} \EEbig{\bx}{|\sigma'(\bw'^\top\bx)|\cdot\norm{h_{\mu^\circ}(\bx)\bx^\top}} \norm{\ba'^\top \bSig_{\mu,\mu}^{-1}(\ba_1-\ba_2)}\\
&\qquad\leq R_2R_3 M_2^{1/2}M_4^{1/2} k\overline{r}\lambda^{-1}\norm{\bw_1-\bw_2} + R_2^2 M_2k\overline{r} \lambda^{-1}\norm{\ba_1-\ba_2}.
\end{align*}
Therefore, $\nabla_{\theta}\nabla_{\theta'} t(\mu,\theta,\theta')$ is uniformly $O(\lambda^{-1})$-Lipschitz w.r.t. both $\theta$ and $\theta'$ by symmetry. All the remaining terms can also be bounded with at most an $O(\lambda^{-2})$ Lipschitz constant; in particular, the terms including three factors of $\bSig_{\mu,\mu}^{-1}$ can be controlled by removing a factor of $\lambda^{-1/2}$ twice and isolating $\bSig_{\mu,\mu}^{-1/2}\bSig_{\mu,\mu^\circ}$ and $\bSig_{\mu^\circ,\mu}\bSig_{\mu,\mu}^{-1/2}$ as in the proof of Lemma \ref{thm:funcgrad}.

Finally, the third-order functional derivative $\nabla_{\tilde{\theta}}\frac{\delta}{\delta\mu}\bH_\mu(\theta,\theta')(\tilde{\theta})$ can be bounded in a similar manner with spectral norm at most $O(\lambda^{-3})$, yielding via Kantorovich-Rubinstein duality that
\begin{equation*}
\norm{\bH_{\mu_1}(\theta,\theta') - \bH_{\mu_2}(\theta,\theta')} = \Norm{\int\frac{\delta}{\delta\mu}\bH_{(1-s)\mu_1+s\mu_2}(\theta,\theta')(\tilde{\theta}) (\mu_2-\mu_1)(\rd\tilde{\theta})} \lesssim (k+d)\lambda^{-3}\cdot\WW_1(\mu_1,\mu_2).
\end{equation*}
The additional $k+d$ factor arises from bounding each entry of $\bH_{\mu_1}-\bH_{\mu_2}$ separately. We omit the details.
\end{proof}

\begin{proposition}\label{thm:operator}
Let $F$ be a functional depending on $\mu$ only through the MLP layer $h_\mu$. Suppose MFD \eqref{eqn:mfd} at time $t$ admits a distribution $\bar{\mu}\in\PP_2(\Theta)$ with $\chi^2(\bar{\mu},\mu_t)\leq\bar{\chi}^2$ such that $\frac{\rd^2}{\rd s^2}\big|_{s=0}F(\bar{\mu}_s) \leq -\Lambda$. Then the smallest eigenvalue $\lambda_0$ of $\HH_{\mu_t}$ satisfies $\lambda_0\leq -\bar{\chi}^{-2}\Lambda$.
\end{proposition}

\begin{proof}
The second derivative along the linear homotopy $\bar{\mu}_s$ can be expanded as
\begin{align*}
\frac{\rd^2}{\rd s^2}\bigg\vert_{s=0}F(\bar{\mu}_s) &= \frac{\rd}{\rd s}\bigg\vert_{s=0} \int\deltaF(\bar{\mu}_s,\theta) (\bar{\mu}-\mu_t)(\rd\theta)\\
&= \iint\ddeltaF(\mu_t,\theta,\theta')(\bar{\mu}-\mu_t)(\rd\theta) (\bar{\mu}-\mu_t)(\rd\theta').
\end{align*}
Now similarly to the proof of Proposition \ref{thm:mfdfirst}, denoting $\theta=(\ba,\bw), \theta'=(\ba',\bw')$ we can exploit the fact that $\ddeltaF(\mu,\theta,\theta')$ is bilinear in $\ba,\ba'$ to relate it to the kernel $\bH_\mu$,
\begin{align*}
\ddeltaF(\mu,\theta,\theta') &= \ba^\top\left[ \nabla_{\ba}\nabla_{\ba'}\ddeltaF(\mu,\theta,\theta') \right]\ba' = (\ba\;\; 0_d)^\top\left[\nabla_\theta\nabla_{\theta'} \ddeltaF(\mu,\theta,\theta')\right] (\ba'\;\; 0_d)\\
&= (\ba\;\; 0_d)^\top \bH_\mu(\theta,\theta') (\ba'\;\; 0_d).
\end{align*}
Writing the eigenfunction decomposition of $\bH_{\mu_t}$ as (omitting the dependency on $t$ for brevity)
\begin{equation*}
    \bH_{\mu_t}(\theta,\theta') = \sum_{j\in\ZZ} \lambda_j\psi_j(\theta) \psi_j(\theta')^\top, \quad\int\norm{\psi_j}^2\rd\mu_t=1 \quad\forall j\in\ZZ, \quad \lambda_1\geq\lambda_2\geq\cdots\geq 0, \quad \lambda_0\leq\lambda_{-1}\leq\cdots\leq 0,
\end{equation*}
we may thus bound
\begin{align*}
    -\Lambda\geq \frac{\rd^2}{\rd s^2}\bigg\vert_{s=0}F(\bar{\mu}_s) &= \iint (\ba\;\; 0_d)^\top \bH_\mu(\theta,\theta') (\ba'\;\; 0_d) (\bar{\mu}-\mu_t)(\rd\theta) (\bar{\mu}-\mu_t)(\rd\theta')\\
    &= \sum_{j\in\ZZ}\lambda_j \left(\int (\ba\;\;0_d)^\top \psi_j(\theta)(\bar{\mu} - \mu_t)(\rd\theta)\right)^2\\
    &\geq -\abs{\lambda_0} \cdot \sum_{j\in\ZZ} \left(\int (\ba\;\;0_d)^\top \psi_j(\theta)(\bar{\mu} - \mu_t)(\rd\theta)\right)^2\\
    & = -\abs{\lambda_0} \cdot \sum_{j\in\ZZ} \left(\int \left(\frac{\rd\bar{\mu}}{\rd\mu_t}-1\right)(\ba\;\;0_d)^\top \psi_j(\theta) \mu_t(\rd\theta)\right)^2\\
    &= -\abs{\lambda_0} \int \left(\frac{\rd\bar{\mu}}{\rd\mu_t}-1\right)^2 \norm{\ba}^2 \mu_t(\rd\theta)\\
    &\geq - \bar{\chi}^2|\lambda_0|,
\end{align*}
where we have made use of Parseval's identity. Hence the largest negative eigenvalue is bounded as $\lambda_0 \leq -\bar{\chi}^{-2}\Lambda$.
\end{proof}

\begin{theorem}\label{thm:generalescape}
Assume $F:\PP_2(\Omega)\to\RR$, $\Omega\subseteq\RR^m$ satisfies $\norm{\nabla \deltaF}\leq C_1$, $\nabla \deltaF$ is $C_2$-Lipschitz, $\bH_\mu$ is Hilbert-Schmidt, $\norm{\bH_\mu}\leq C_3$, $\bH_\mu(\theta,\theta')$ is $C_4$-Lipschitz w.r.t. $\theta,\theta'$ and $C_5$-Lipschitz w.r.t. $\mu$ in $\WW_1$. Further suppose that $\lambda_0:=\lambda_\textup{min}(\HH_{\mu^\dagger})<0$ and the corresponding eigenfunction $\psi_0$ satisfies $|\int\psi_0^\top\nabla\deltaL(\mu_t)\rd\mu_t| \geq \alpha$ for some $\alpha> 0$. Then WGF initialized at $\mu_0=\mu^\dagger$ decreases $F$ by at least $F(\mu_\tau)\leq F(\mu_0) - \Omega\left(\frac{|\lambda_0|\alpha}{\sqrt{m}\tau}\right)$ in time $\tau= O\left(\frac{1}{|\lambda_0|} \log\frac{|\lambda_0|}{\sqrt{m}\alpha}\right)$.
\end{theorem}

Unlike before, $F$ can be completely general and does not need to depend on $\mu$ through an MLP layer.

\begin{proof}
First note that the function $\theta'\mapsto \bH_{\mu^\dagger}(\theta,\theta') \nabla\deltaF(\mu_t,\theta')$ is uniformly Lipschitz: for any $\theta_1',\theta_2'$,
\begin{align*}
&\Norm{\bH_{\mu^\dagger}(\theta,\theta_1') \nabla\deltaF(\mu_t,\theta_1') - \bH_{\mu^\dagger}(\theta,\theta_2') \nabla\deltaF(\mu_t,\theta_2')}\\
&\leq \norm{\bH_{\mu^\dagger}(\theta,\theta_1') - \bH_{\mu^\dagger}(\theta,\theta_2')}\cdot\Norm{\nabla\deltaF(\mu_t,\theta_1')} + \norm{\bH_{\mu^\dagger}(\theta,\theta_2')} \cdot\Norm{\nabla\deltaF(\mu_t,\theta_1') - \nabla\deltaF(\mu_t,\theta_2')}\\
&\leq C_1C_4\norm{\theta_1'-\theta_2'} +C_2C_3\norm{\theta_1'-\theta_2'}.
\end{align*}
We re-expand the evolution equation \eqref{eqn:evo} for the dynamics $(\mu_t)_{t\geq 0}$ around $\mu^\dagger$ as
\begin{align*}
\partial_t\left[\nabla_\theta\deltaF(\mu_t,\theta)\right] &= -\int\bH_{\mu_t}(\theta,\theta') \nabla_{\theta'}\deltaF(\mu_t,\theta') \mu_t(\rd\theta')\\
&=: -\int\bH_{\mu^\dagger}(\theta,\theta') \nabla_{\theta'}\deltaF(\mu_t,\theta') \mu^\dagger(\rd\theta') + e(t,\theta),
\end{align*}
where the difference or error function $e(t,\theta)$ can be bounded as
\begin{align*}
\norm{e(t,\theta)} &\leq \Norm{\int (\bH_{\mu_t}(\theta,\theta')-\bH_{\mu^\dagger}(\theta,\theta')) \nabla_{\theta'}\deltaF(\mu_t,\theta') \mu_t(\rd\theta')} + \Norm{\int\bH_{\mu^\dagger}(\theta,\theta') \nabla_{\theta'}\deltaF(\mu_t,\theta') (\mu_t - \mu^\dagger)(\rd\theta')}\\
&\leq \left(C_1C_5 + (C_1C_4+C_2C_3)m^{1/2}\right) \WW_1(\mu_t,\mu^\dagger).\\
&=: C_6\WW_1(\mu_t,\mu^\dagger).
\end{align*}
For the second term, we have used the Lipschitz constant derived above to bound each entry separately. Then the $\psi_0$-component $\alpha_0(t):= \int\psi_0^\top\nabla\deltaF(\mu_t)\rd\mu^\dagger$ of the gradient evolves according to
\begin{align*}
\frac{\rd}{\rd t}\alpha_0(t) &= -\iint\psi_0(\theta)^\top \bH_{\mu^\dagger}(\theta,\theta') \nabla_{\theta'}\deltaF(\mu_t,\theta') \mu^\dagger(\rd\theta') \mu^\dagger(\rd\theta) + \int\psi_0(\theta)^\top e(t,\theta)\mu^\dagger(\rd\theta)\\
&= -\lambda_0 \int \psi_0(\theta)^\top \nabla_\theta\deltaF(\mu_t,\theta) \mu^\dagger(\rd\theta) + \int\psi_0(\theta)^\top e(t,\theta)\mu^\dagger(\rd\theta),
\end{align*}
and hence
\begin{equation*}
\abs{\frac{\rd}{\rd t}\alpha_0(t) + \lambda_0\alpha_0(t)} \leq \left(\int\norm{\psi_0}^2\rd\mu^\dagger\right)^{1/2} \sup_{\theta\in\Theta}\,\norm{e(t,\theta)} \leq C_8\WW_1(\mu_t,\mu^\dagger).
\end{equation*}
Without loss of generality, assume initially $\alpha_0(0)$ is positive so that $\alpha_0(0)\geq\alpha$. We consider a 1-Wasserstein ball centered at $\mu^\dagger$ with radius small enough so that the error term is negligible compared to the exponential growth,
\begin{equation*}
\BB_{\WW}(\Delta) = \left\{\mu\in\PP_2(\Theta): \WW_1(\mu,\mu^\dagger)\leq\Delta:= \frac{|\lambda_0|\alpha}{2C_6}\right\}.
\end{equation*}
Then for a set time interval $\tau>0$ to be determined, either of the following must happen:
\begin{enumerate}[\normalfont(i)]
\item $(\mu_t)_{t\in [0,\tau]} \subset \BB_{\WW}(\Delta)$. In this case, $\alpha_0(t)$ grows exponentially during the entire interval $t\in [0,\tau]$ as
\begin{equation*}
\frac{\rd}{\rd t}\alpha_0(t) \geq |\lambda_0| \alpha_0(t) - C_6\Delta = |\lambda_0|\left(\alpha_0(t)-\frac{\alpha}{2}\right) >0,
\end{equation*}
showing that
\begin{equation*}
\alpha_0(t) \geq e^{|\lambda_0|t}\left(\alpha_0(0)-\frac{\alpha}{2}\right) +\frac{\alpha}{2} \geq \frac{\alpha(e^{|\lambda_0|t}+1)}{2}.
\end{equation*}
Then the decrease of $F$ after time $\tau$ can be bounded below by retrieving the $\psi_0$-component as
\begin{align*}
    F(\mu_0) -F(\mu_\tau) &=\int_0^\tau\int \Norm{\nabla\deltaF(\mu_t,\theta)}^2 \mu_t(\rd\theta)\rd t\\
    &\geq \int_0^\tau \bigg(\int \Norm{\nabla\deltaF(\mu_t,\theta)}^2 \mu^\dagger(\rd\theta)\rd t -2C_1C_2\WW_1(\mu_t,\mu^\dagger)\bigg)\\
    &\geq \int_0^\tau \left(\int \psi_0(\theta)^\top \nabla\deltaF(\mu_t,\theta)\mu^\dagger(\rd\theta)\right)^2 \rd t -2C_1C_2\Delta\tau\\
    & = \int_0^\tau\alpha_0(t)^2\rd t-2C_1C_2\Delta\tau\\
    &\geq \frac{\alpha^2}{4} \left(\frac{1}{2|\lambda_0|}(e^{2|\lambda_0|\tau}-1) + \frac{2}{|\lambda_0|}(e^{|\lambda_0|\tau}-1) +\tau\right) - \frac{C_1C_2}{C_6}|\lambda_0|\alpha\tau.
\end{align*}

\item $\mu_{\tau_e}\notin \BB_{\WW}(\Delta)$ for some $\tau_e\leq\tau$. If the mean-field flow has managed to escape the ball $\BB_{\WW}(\Delta)$ in time $\tau_e$, the Benamou-Brenier formula (Proposition \ref{thm:benamou}) immediately guarantees that
\begin{equation*}
F(\mu_0) -F(\mu_\tau) \geq F(\mu_0) -F(\mu_{\tau_e}) =\int_0^{\tau_e}\int \Norm{\nabla\deltaF(\mu_t,\theta)}^2 \mu_t(\rd\theta)\rd t \geq \frac{\WW_2(\mu_{\tau_e},\mu^\dagger)^2}{\tau_e} > \frac{\Delta^2}{\tau}.
\end{equation*}
\end{enumerate}
Thus we have proved that:
\begin{equation}\label{eqn:balance}
F(\mu_0) -F(\mu_\tau) \geq \left(\frac{\alpha^2}{4} \left(\frac{1}{2|\lambda_0|}(e^{2|\lambda_0|\tau}-1) + \frac{2}{|\lambda_0|}(e^{|\lambda_0|\tau}-1) +\tau\right) - \frac{C_1C_2}{C_6}|\lambda_0|\alpha\tau\right) \wedge \frac{\lambda_0^2\alpha^2}{4C_6^2\tau}.
\end{equation}
Due to the exponential terms, we see $\tau\asymp\log\frac{1}{\alpha}$ is enough to ensure that the two terms become roughly equal so that the guarantee is close to optimal. For the remainder of the proof, we derive the exact formula. Choose
\begin{equation*}
\tau = \frac{1}{|\lambda_0|} \log\frac{C_7}{\alpha}
\end{equation*}
for some $C_7>\alpha$. The first term in the right-hand side of \eqref{eqn:balance} can be bounded as
\begin{align*}
&\frac{\alpha^2}{4} \left(\frac{1}{2|\lambda_0|}(e^{2|\lambda_0|\tau}-1) + \frac{2}{|\lambda_0|}(e^{|\lambda_0|\tau}-1) +\tau\right) - \frac{C_1C_2}{C_6}|\lambda_0|\alpha\tau \\
& =\frac{\alpha^2}{8|\lambda_0|}\left(\frac{C_7^2}{\alpha^2}-1\right) + \frac{\alpha^2}{2|\lambda_0|} \left(\frac{C_7}{\alpha}-1\right) + \frac{\alpha^2}{4|\lambda_0|}\log\frac{C_7}{\alpha} - \frac{C_1C_2}{C_6}\alpha\log\frac{C_7}{\alpha}\\
&\geq \left(\frac{C_7^2}{24|\lambda_0|} -\frac{5\alpha^2}{8|\lambda_0|}\right) + \left(\frac{C_7^2}{24|\lambda_0|} -\frac{C_1C_2}{C_6}\alpha\log\frac{C_7}{\alpha}\right) + \frac{C_7^2}{24|\lambda_0|} + \frac{C_7\alpha}{2|\lambda_0|} + \frac{\alpha^2}{4|\lambda_0|}\log\frac{C_7}{\alpha}\\
&\geq  \left(\frac{C_7^2}{24|\lambda_0|} -\frac{5\alpha^2}{8|\lambda_0|}\right) + \left(\frac{C_7^2}{24|\lambda_0|} -\frac{C_1C_2C_7}{C_6e}\right) + \frac{C_7^2}{24|\lambda_0|},
\end{align*}
where we have used the fact that the function $x\mapsto x\log\frac{c}{x}$ has maximum $\frac{c}{e}$. Then the first term of \eqref{eqn:balance} will dominate the second as long as
\begin{equation*}
\frac{C_7^2}{24|\lambda_0|} \geq \frac{5\alpha^2}{8|\lambda_0|} \vee \frac{C_1C_2C_7}{C_6e} \vee \frac{|\lambda_0|^3\alpha^2}{4C_6^2\log\frac{C_7}{\alpha}}.
\end{equation*}
Manipulating terms shows that
\begin{equation}\label{eqn:see7}
C_7 = \sqrt{15}\alpha \vee \frac{24C_1C_2|\lambda_0|}{C_6e} \vee \sqrt{\frac{12}{\log 15}} \frac{\lambda_0^2\alpha}{C_6}
\end{equation}
is sufficient. For the purposes of the general statement, we focus on asymptotic behavior w.r.t. $\alpha$ and hide all regularity constants $C_1,\cdots,C_5$, yielding $C_6=O(m^{1/2})$ and $C_7=O(|\lambda_0|m^{-1/2})$ as the second term in \eqref{eqn:see7} dominates.
\end{proof}

\paragraph{Proof of Theorem \ref{thm:naiveescape}.} Let us fix the lower bound $\lambda_\textup{min}(\bSig_{\mu_r,\mu_r})\geq\lambda=\Theta(\frac{1}{k})$. (The bound only needs to hold either locally for the $\WW_2$-ball of radius $\Delta$ in the proof of Theorem \ref{thm:generalescape}, or along the dynamics $\mu_t$ until escape.) We first need a robust version of Theorem \ref{thm:landscape}\ref{item:secondorder} since $\mu_t$ cannot be exactly on a critical point. If $\frac{\rd}{\rd s}\big|_{s=0}\LL(\bar{\mu}_s) > -\delta$ it must hold that $\norm{\bL_\mu \bSig_{\mu^\circ,\mu}\bSig_{\mu,\mu}^{-1}}_* <\frac{\delta}{2}$ by \eqref{eqn:firstderiv}. Then from \eqref{eqn:secondderiv}, again choosing $\bR\in\mathcal{O}(k)$ such that $\bSig_{\mu^\circ,\mu} \bSig_{\mu,\mu}^{-1}\bR$ is symmetric,
\begin{align*}
\frac{\rd^2}{\rd s^2}\bigg\vert_{s=0}\LL(\bar{\mu}_s) &= -4\tr\left(\bL_\mu^2\bR^\top\bSig_{\mu,\mu}^{-1}\bR\right) + 2\tr\left(\bL_\mu(2\bSig_{\mu^\circ,\mu} \bSig_{\mu,\mu}^{-1}\bR+\bR^\top \bSig_{\mu,\mu}^{-1}\bSig_{\mu,\mu^\circ}-2\bI_k)\bSig_{\mu^\circ,\mu} \bSig_{\mu,\mu}^{-1}\bR\right)\\
&= -4\tr\left(\bL_\mu^2\bR^\top\bSig_{\mu,\mu}^{-1}\bR\right) + 2\tr\left((2\bSig_{\mu^\circ,\mu} \bSig_{\mu,\mu}^{-1}\bR+\bR^\top \bSig_{\mu,\mu}^{-1}\bSig_{\mu,\mu^\circ}-2\bI_k)^\top \bL_\mu\bSig_{\mu^\circ,\mu} \bSig_{\mu,\mu}^{-1}\bR\right)\\
&\leq -\frac{4}{kR_1^2}\LL(\mu_t)^2 +2\, \norm{\bL_\mu\bSig_{\mu^\circ,\mu} \bSig_{\mu,\mu}^{-1}}_* \norm{2\bSig_{\mu^\circ,\mu} \bSig_{\mu,\mu}^{-1}\bR+\bR^\top \bSig_{\mu,\mu}^{-1}\bSig_{\mu,\mu^\circ}-2\bI_k}\\
&\leq -\frac{4}{kR_1^2}\LL(\mu_t)^2 + \left(3\,\norm{\bSig_{\mu^\circ,\mu} \bSig_{\mu,\mu}^{-1/2}}\cdot \norm{\bSig_{\mu,\mu}^{-1/2}\bR}+2\right)\delta\\
&\leq -\frac{4}{kR_1^2}\LL(\mu_t)^2 + (3\overline{r}^{1/2}\lambda^{-1/2}+2)\delta.
\end{align*}
Hence if $\delta\leq \frac{2}{kR_1^2(3\overline{r}^{1/2}\lambda^{-1/2}+2)} \LL(\mu_t)^2$ then $\frac{\rd^2}{\rd s^2}\big\vert_{s=0}\LL(\bar{\mu}_s) \leq -\frac{2}{kR_1^2}\LL(\mu_t)^2$, and by Proposition \ref{thm:operator} it holds that
\begin{align*}
\lambda_0=\lambda_\textup{min}(\HH_{\mu_t}) \leq - \frac{2\gamma}{kR_1^2 R_4}\LL(\mu_t)^2.
\end{align*}
Then Theorem \ref{thm:generalescape} applies to $F=\LL$ by virtue of Lemma \ref{thm:kernelvalid} and the regularity constants derived in Lemma \ref{thm:funcgrad} and \ref{thm:kernelreg}. One can check that
\begin{equation*}
C_6=C_1C_5 + (C_1C_4+C_2C_3)(k+d)^{1/2} = O\left(\frac{d}{k\lambda^4}\right)
\end{equation*}
and
\begin{equation*}
C_7 = O\left(\alpha\vee\frac{\lambda^{5/2}\gamma}{k^{3/2}d}\LL(\mu_t)^2 \vee \frac{\lambda^2\gamma^{3/2}\alpha^{1/2}}{kd^{1/2}} \LL(\mu_t)^3\right) = O\left(\alpha + \frac{\gamma}{k^4d}\right);
\end{equation*}
the third term is dominated by the geometric mean of the first two, and $\LL(\mu_t) = O(1)$. Hence the time interval of interest is
\begin{equation*}
\tau = O\left(\frac{k}{\gamma\LL(\mu_t)^2}\left(\log\frac{\gamma}{k^4d\alpha}\right)\vee 1\right),
\end{equation*}
and the guaranteed decrease of the objective is
\begin{equation*}
\LL(\mu_t)-\LL(\mu_{t+\tau}) \geq \frac{|\lambda_0|\alpha}{2C_6\tau} \geq \Omega\left(\frac{\gamma^2\alpha\LL(\mu_t)^4}{k^5d}\left(\log\frac{\gamma}{k^4d\alpha}\vee 1\right)^{-1}\right).
\end{equation*}

\subsection{Escaping from Saddle Points}

The usual theory of Gaussian processes can be readily extended to multivariable outputs.
\begin{definition}[vector-valued Gaussian process]\label{def:gp}
The random function $\bxi:\Omega\to\RR^m$ is said to follow a Gaussian process if any finite collection of variables $\bxi(\theta_1),\cdots,\bxi(\theta_j)$ are jointly normally distributed. The process is determined by the mean function $\bm{m}:\Omega\to\RR^m$, $\bm{m}(\theta) = \E{\bxi(\theta)}$ and matrix-valued covariance function
\begin{equation*}
\bK:\Omega\times\Omega\to\RR^{m\times m},\quad\bK(\theta,\theta') = \E{(\bxi(\theta)-\bm{m}(\theta))(\bxi(\theta')-\bm{m}(\theta'))^\top}.
\end{equation*}
We denote this process as $\bxi\sim\GP(\bm{m},\bK)$. See \citet{Alvarez12} for further details.
\end{definition}

\begin{lemma}\label{thm:gpnormal}
For any $\mu\in\PP_2(\Omega)$, square-integrable test function $\psi\in L^2(\Omega,\mu;\RR^m)$ and covariance function $\bK:\Omega\times\Omega\to\RR^{m\times m}$ satisfying $\int\norm{\bK(\theta,\theta)}\mu(\rd\theta)<\infty$ the inner product $\langle \psi,\bxi\rangle_{L^2(\Omega,\mu;\RR^m)}$ for $\bxi\sim \GP(0,\bK)$ is normally distributed.
\end{lemma}

\begin{proof}
Note that the inner product is defined almost surely since
\begin{equation*}
\EEbig{\bxi}{\norm{\bxi}_{L^2(\Omega,\mu;\RR^m)}^2} = \int\EE{\bxi}{\norm{\bxi(\theta)}^2} \mu(\rd\theta) = \int \tr\bK(\theta,\theta) \mu(\rd\theta) <\infty.
\end{equation*}
We denote by $\mathscr{E}$ the closed linear span of the set of square-integrable random variables $\{\psi(\theta)^\top\bxi(\theta):\theta\in\Omega\}$. For any $\bZ\in \mathscr{E}^\perp$ it holds that $\EE{\bxi}{\bZ \psi(\theta)^\top\bxi(\theta)} = 0$, so that by Fubini's theorem
\begin{equation*}
\EEbig{\bxi}{\bZ \langle \psi,\bxi\rangle_{L^2(\Omega,\mu;\RR^m)}} = \EEbig{\bxi}{\int\bZ\psi(\theta)^\top\bxi(\theta) \mu(\rd\theta)} = 0.
\end{equation*}
Hence $\langle \psi,\bxi\rangle_{L^2(\Omega,\mu;\RR^m)} \in (\mathscr{E}^\perp)^\perp= \mathscr{E}$, and so is normally distributed.
\end{proof}

For the proposed perturbation process, the change in the gradient field along the flow of $\bxi$ can be quantified as
\begin{align*}
\nabla\deltaF(\mu_{\Delta t},\theta) - \nabla\deltaF(\mu^\dagger,\theta) &= \int_0^{\Delta t} \partial_t\left[\nabla\deltaF(\mu_t,\theta)\right] \rd t\\
&= -\int_0^{\Delta t} \int \nabla_\theta\nabla_{\theta'}\ddeltaF(\mu_t, \theta,\theta') \bxi(\theta') \mu_t(\rd\theta') \rd t\\
&= -\int_0^{\Delta t} \HH_{\mu_t}\left[\bxi\right] \rd t.
\end{align*}
The resulting $\psi_0$-component is
\begin{align*}
\alpha(\bxi) &= \int \psi_0(\theta)^\top \nabla\deltaF(\mu_{\Delta t},\theta)\mu^\dagger(\rd\theta)\\
&= \int \psi_0(\theta)^\top \bigg(\nabla\deltaF(\mu^\dagger,\theta) -\int_0^{\Delta t} \HH_{\mu^\dagger}\left[\bxi\right] \rd t + \int_0^{\Delta t} (\HH_{\mu^\dagger} -\HH_{\mu_t}) \left[\bxi\right] \rd t\bigg) \mu^\dagger(\rd\theta)\\
& = -\lambda_0\Delta t \int \psi_0(\theta)^\top \bxi(\theta) \mu^\dagger(\rd\theta) + \alpha +O(\Delta t^2),
\end{align*}
and first term is normally distributed by Lemma \ref{thm:gpnormal}.

\end{document}